\documentclass[letterpaper, 10 pt, conference]{ieeeconf} 
\pdfoutput=1
\IEEEoverridecommandlockouts
\usepackage{microtype}
\usepackage{balance}
\usepackage{graphicx}
\usepackage{subcaption}
\usepackage{multicol}
\usepackage{multirow}
\usepackage{pifont}

\newcommand{\xmark}{\ding{55}}
\usepackage[normalem]{ulem}
\usepackage{booktabs}
\usepackage{bbm}
\usepackage{hyperref}
\usepackage{color}
\usepackage{kbordermatrix}
\usepackage{multirow}
\usepackage[table]{xcolor}  
\usepackage{wrapfig} 
\usepackage{algorithm}
\usepackage{algorithmic}
\usepackage{amsmath}
\usepackage{xurl}
\usepackage{amssymb}
\usepackage{mathtools}

\usepackage{float}

\usepackage[capitalize]{cleveref}
\crefname{table}{Tab.}{Tabs.}
\Crefname{table}{Tab.}{Tabs.}
\crefname{figure}{Fig.}{Figs.}
\crefname{lemma}{Lem.}{Lems.}
\crefname{proposition}{Prop.}{Props.}
\Crefname{proposition}{Proposition}{Propositions}
\crefname{theorem}{Thm.}{Thms.}
\crefname{appendix}{App.}{Apps.}
\crefname{section}{Sec.}{Secs.}

\usepackage{amsthm}

\usepackage{xargs}

\newcommand{\psum}{p_{\rm sum}}
\newcommand{\pmin}{p_{\rm min}}
\DeclareMathOperator*{\argmax}{arg\,max}

\usepackage{tikz}

\newcommand{\addappendix}{}

\theoremstyle{plain}
\newtheorem{theorem}{Theorem}[section]
\newtheorem{proposition}{Proposition}[section]
\newtheorem{lemma}{Lemma}[section]
\newtheorem{corollary}{Corollary}[section]

\theoremstyle{definition}
\newtheorem{definition}{Definition}[section]

\theoremstyle{remark}
\newtheorem{remark}{Remark}[section]

\newtheorem{example}{Example}[section]

\title{\LARGE \bf
Fair Best Arm Identification with Fixed Confidence
}

\author{Alessio Russo$^{1,\star}$ and Filippo Vannella$^{2,\star}$
\thanks{$^{\star}$Equal contribution.}%
\thanks{$^{1}$Ericsson AB, Stockholm, Sweden.}%
\thanks{$^{2}$Ericsson Research, Stockholm, Sweden.}%
}

\begin{document}

\maketitle
\begin{abstract}
In this work, we present a novel framework for Best Arm Identification (BAI) under fairness constraints, a setting that we refer to as  {\color{black} \textit{F-BAI} (fair BAI)}.
Unlike traditional BAI, which solely focuses on identifying the optimal arm with minimal sample complexity,  {\color{black}F-BAI} also includes a set of fairness constraints. These constraints impose a lower limit on the selection rate of each arm and can be either model-agnostic or model-dependent. For this setting, we establish an instance-specific sample complexity lower bound and analyze the \emph{price of fairness}, quantifying how fairness impacts sample complexity. Based on the sample complexity lower bound, we propose {\color{black}F-TaS}, an algorithm provably matching the sample complexity lower bound, while ensuring that the fairness constraints are satisfied. Numerical results, conducted using both a synthetic model and a practical wireless scheduling application, show the efficiency of {\color{black}F-TaS} in minimizing the sample complexity while achieving low fairness violations. 

\end{abstract}

\section{Introduction}
\label{sec:intro}
In recent years, a large body of work has focused on making machine learning systems more fair \cite{caton2020fairness}. This effort reflects a broader societal shift towards more ethical algorithms, recognizing the impact these systems have across various sectors of society.

Notably, the importance of fairness has been also recently acknowledged within the context of Multi-Armed Bandit (MABs) \cite{gajane2022survey}. MABs (or simply bandits) \cite{lai1985asymptotically} are sequential decision-making problems under uncertainty, in which a learner must strategically select arms to maximize a given objective over time. 

Bandit algorithms have seen widespread adoption in various applications: online advertisement \cite{xu2013estimation}, recommender systems \cite{Ariu2020regret,chu2011contextual}, wireless network optimization \cite{nguyen2019scheduling}, and others \cite{Bouneffouf20_survey}. Due to the importance and impact of these applications, there has been an increasing understanding of the need to include fairness aspects in the decision-making process. For example, in wireless scheduling problems with multiple Quality of Service (QoS) classes (see Sec. \ref{sec:experiments}), fairness is achieved by ensuring that users within each class receive appropriate performance according to their specific requirements, e.g., in terms of throughput \cite{Li19combinatorial}.

However, traditional bandit algorithms do not inherently address such fairness aspects. Indeed, while guaranteeing fairness has received attention in the setting of regret minimization \cite{Joseph16, Li19combinatorial, Patil21, zhang2021fairness}, this aspect remains largely unexplored within the problem of Best Arm Identification (BAI) \cite{audibert2010best,garivier2016optimal}. In BAI, the objective is to find an optimal arm with a prescribed level of confidence and with minimal sample complexity.

In this work we investigate how to include fairness constraints into the BAI problem, leading to a novel setting that we refer to as \textit{fair BAI}{\color{black}, or F-BAI}. This setting extends the classical BAI problem by imposing fairness constraints on arm selection rates. These constraints ensure that no arm is underrepresented in the sampling process, addressing potential biases that may arise from pure exploration algorithms. Nonetheless, the introduction of fairness into BAI raises some challenges, notably how to balance the inherent trade-off between fairness and sample complexity, a duality that reflects wider challenges in algorithmic fairness and decision-making. We summarize our contributions as follows \footnote{\ifdefined \addappendix  All proofs are detailed in the appendix. Code repository: \url{https://github.com/rssalessio/fair-best-arm-identification/}\else All proofs are detailed in the appendix of the full paper. For access to the full text, please refer to the hyperlink provided in the abstract.\fi}:

\begin{enumerate}
    \item \textit{The fair BAI setting.}  
    We introduce fair BAI in \textsection \ref{sec:fair_bai}, a novel and general bandit setting for BAI under fairness constraints. Our approach to fairness is broad, encompassing various classical notions such as proportional fairness and individual fairness. This versatility allows our framework to be applicable in various settings.
    
    \item \textit{Sample complexity and price of fairness.} In \textsection \ref{sec:lower_bound} we derive an instance-specific lower bound on the sample complexity of any Probably Approximately Correct (PAC) algorithm that adheres to these fairness constraints. Based on this bound, we quantify the \textit{price of fairness} in fair BAI. This price refers to the additional samples required to identify the best arm while complying with the fairness constraints, providing additional insights into the trade-off between sample complexity and fairness.

    \item \textit{Optimal and fair algorithm.} In \textsection \ref{sec:algo} we devise {\color{black}F-TaS} , an algorithm that asymptotically matches the sample complexity lower bound as the confidence grows higher. Furthermore, the algorithm guarantees that the fairness constraints are satisfied at each round of the interaction (for pre-specified values of fairness, i.e., model-agnostic constraints) or asymptotically (for fairness constraints that depend on the model parameter, which needs to be learned by the algorithm).

    \item \textit{Numerical experiments.} Lastly, in \textsection \ref{sec:experiments} we test our algorithm on both synthetic instances and a fair scheduling problem in radio wireless networks. Numerical results show that {\color{black}F-TaS} is not only sample efficient but also consistently achieves minimal fairness violation.
   
\end{enumerate}

\section{Related Work}
\label{sec:related_work}
Different notions of fairness have been considered in MAB problems and, more generally, in sequential decision-making problems. For an extended literature review, refer to \ifdefined\addappendix App.  \ref{sec:extended_related_work}\else App. D\fi. For comprehensive surveys on this topic, see \cite{gajane2022survey,zhang2021fairness}. 

In MAB problems, fairness has been investigated in different settings including plain, combinatorial, contextual, and linear \cite{Joseph16, Li19combinatorial, jabbari2017fairness, grazzi2022group} and with different notions of fairness. The majority of these notions, generally fall into the following categories: pre-specified fairness, individual fairness, counterfactual fairness and group fairness  \cite{zhang2021fairness,gajane2022survey}. Of these notions, the closest to our work are the first two, and we focus on these in the remainder of this section. 

\emph{Selection with pre-specified values of fairness  }\cite{Li19combinatorial, Claure20, Patil21, liu2022combinatorial, Chen20} simply demands that the rate, or probability, at which an algorithm selects an arm stays within a pre-specified range. Related works in this category mostly target finite-time fairness constraints, which require that a given number of arms pull must be satisfied in each round \cite{Patil21, Chen20, celis2019controlling}. These constraints are, in general, model-agnostic.

On the other hand, \emph{individual fairness} \cite{dwork2012fairness,joseph2016fair} requires a system to make comparable decisions for similar individuals, and the constraints could be based on similarity or merit \cite{Joseph16,celis2019controlling,liu2017calibrated, gillen2018online, Wang21, Wang21fair}. These constraints impose a minimal rate at which an algorithm must select arms, and they are generally model-dependent. Moreover, in this setting, algorithms often provide asymptotic guarantees as the time horizon grows large. In general, it is hard to avoid asymptotic guarantees when the constraints depend on the unknown model parameter, which needs to be estimated during the learning process. 

Other important works consider the $\alpha$-fairness criterion \cite{atkinson1970measurement,radunovic2007unified,si2022enabling} for fair resource allocation, which encompasses different fairness criteria when varying the value of the parameter $\alpha$. This criterion includes different notions of fairness, such as \emph{max-min} fairness, which allocates resources as equally as possible,  or \emph{proportional fairness}, which allocates resources in a proportional manner \cite{Wang21fair, Wang21, talebi2018learning}.

Notably, as explained in the next section, our fairness definition is very general and includes for example the case of individual, proportional, and pre-specified values of fairness (see \cref{remark:examples} for details). 

Last, but not least, the totality of the above-mentioned works focuses on the regret minimization setting, where the aim is to maximize  the expected cumulative rewards \cite{lai1985asymptotically}. In contrast, our work focuses on the setting of pure exploration (a.k.a. BAI) with fixed confidence \cite{garivier2016optimal}.  
To the best of our knowledge, the only work investigating fairness in BAI is \cite{wu2023best}, where the authors consider fairness constraints on sub-populations (see \ifdefined\addappendix App.  \ref{sec:extended_related_work} \else App. D \fi for details). However, our setting is inherently different and assumes fairness constraints on each arm rather than sub-populations. 

\section{Problem Setting}
\label{sec:problem_setting}
In this section, we outline the bandit model considered in this paper and we present our fair BAI setting.

\subsection{Multi-Armed Bandit Model}
We consider a stochastic bandit problem with a finite set of $K$ arms, that we denote by $[K] := \{1,\dots,K\}$. In each round $t \ge 1$, the learner selects an arm $a_t \in [K]$, and observes a Gaussian reward $r_t \sim {\cal N}(\theta_{a_t}, 1)$. The rewards are i.i.d. (over rounds) and $\theta = (\theta_a)_{a\in [K]}$ is the unknown parameter vector. We indicate by $a^\star = \arg\max_{a\in [K]} \theta_a$ the arm with highest average reward, that we assume to be unique, and indicate by $\Theta =\{\theta\in \mathbb{R}^K: |\arg\max_{a\in [K]} \theta_a|=1 \}$ the set of models satisfying this property. We define the sub-optimality gap for an arm $a\neq a^\star$ as $\Delta_a = \theta_{a^\star}- \theta_a$, and the maximal (resp. minimal) gap as $\Delta_{\max} = \max_{a\neq a^\star} \Delta_a$ (resp. $\Delta_{\min} = \min_{a\neq a^\star} \Delta_a$). We also indicate by $N_a(t)=\sum_{s=1}^t\mathbf{1}_{\{a_s=a\}}$ the number of times an arm $a$ has been selected up to round $t$. The empirical average of $\theta = (\theta_a)_{a\in [K]}$ at round $t$ is denoted as $\hat{\theta}(t) = (\hat{\theta}_a(t))_{a\in [K]}$, where $\hat{\theta}_{a}(t) = \frac{1}{N_{a}(t)} \sum_{s\in[t]}r_s\mathbf{1}_{\{a_s = a\}}$. We denote by $\mathcal{F}_{t}$ the $\sigma$-algebra generated by $(a_1, r_1, \dots, a_t, r_t)$, the history of observations. We indicate by $\mathrm{kl}(p,q)$ the Kullback–Leibler divergence between two Bernoulli distributions of mean $p$ and $q$. For any two vectors we write $x,y\in[0,1]^K$, $x \geq y$ to denote $x_a \in [y_a,1],\forall a\in [K]$.

\subsection{Fair Best Arm Identification}
We now briefly introduce the BAI setting, then proceed to explain how to incorporate fairness constraints.

\label{sec:fair_bai}
\paragraph{Best Arm Identification}
In BAI \cite{kaufmann2016complexity}, the objective is to identify the best arm $a^\star$ with probability at least $1-\delta$, where $\delta\in(0,1/2)$, using the least number of samples. A BAI algorithm ${\cal A}$ consists of a sampling rule $\pi_t$, a stopping rule, and a decision rule. The
sampling rule $\pi_t$ decides which arm is selected in round $t$ based on past observations: $a_t$ is $\mathcal{F}_{t-1}$-measurable. The stopping rule decides when to stop sampling, and is defined by $\tau_{\delta}$, a stopping time w.r.t. the filtration $(\mathcal{F}_{t})_{t\ge1}$. The sample complexity for an algorithm ${\cal A}$ is denoted by $\mathbb{E}_{\theta,{\cal A}}[\tau_\delta]$, and w.l.o.g. in the following we simply denote it by $\mathbb{E}_\theta[\tau_\delta]$. Lastly, the decision rule outputs a guess of the
best arm $\hat{a}_{\tau_\delta}$, based on observations collected up to round $\tau_{\delta}$. 

\paragraph{Fairness Constraints}

In fair BAI, we seek to identify the best arm as quickly as possible while satisfying a set of fairness constraints. We consider general types of constraints which can be either \textit{pre-specified} or \textit{dependent on the problem parameter $\theta$}, as detailed in the following. 
\begin{enumerate}
    \item \emph{Pre-specified constraints}: the selection rate at the random stopping time $\tau_\delta$, needs be larger than some \emph{pre-specified} value $p_a\in [0,1]$:
    \begin{equation}\label{eq:constraint_bai_constant_p}
        \frac{\mathbb{E}_\theta[N_a(\tau_\delta)]}{\mathbb{E}_\theta[\tau_\delta]} \ge p_a,  \forall a\in[K].\end{equation}
    \item $\theta$-\emph{dependent constraints}: asymptotically, as $\delta\to 0$, the  selection rate at the stopping time $\tau_\delta$ needs to be larger than some $\theta$-dependent function   $p_a(\theta):\mathbb{R}^K\to[0,1]$:
\begin{equation}\label{eq:constraint_bai_ptheta}
        \liminf_{\delta\to 0}\frac{\mathbb{E}_\theta[N_a(\tau_{\delta})]}{\mathbb{E}_\theta[\tau_{\delta}]} \ge p_a(\theta), \forall a\in[K],
    \end{equation}
    In this case, we further assume $p_a(\theta)$ to be continuous in $\theta$, for every arm $a\in [K]$.
\end{enumerate}

{\color{black}
\begin{remark}
{\it An alternative definition of fairness could consider constraints of the type $ \mathbb{E}_\theta\left[N_a(\tau_\delta)/\tau_\delta\right]\ge p_a(\theta)$. We refer to this  as "\textit{sample-path fairness}" as it evaluates fairness on each sample path. The reason for using the definitions in \eqref{eq:constraint_bai_constant_p}-\eqref{eq:constraint_bai_ptheta} lies in its analytical tractability (note that $w_a = \mathbb{E}_\theta\left[N_a(\tau_\delta)\right]/\mathbb{E}_\theta[\tau_\delta]$). In \ifdefined\addappendix App. \ref{app:sample_path_fairness} \else App. B-B \fi we show that our algorithm ({\sc F-TaS}) satisfies both the fairness definitions in \eqref{eq:constraint_bai_constant_p}-\eqref{eq:constraint_bai_ptheta} and is sample-path fair.}

\end{remark}
}

To lighten the notation, when there is no ambiguity, we omit the dependence on $\theta$ for $p(\theta)$. In the remainder of the paper, we refer to the vector $p=(p_a)_{a\in [K]}$ as the \textit{fairness rate} or simply rate. We denote by $\psum\coloneqq \sum_{a\in [K]} p_a\leq 1$ the sum of fairness rates, and by $\pmin= \min_{a\in [K]: p_a>0} p_a$ the minimal fairness rate. 
Note that, for $\theta$-dependent constraints, the guarantees are asymptotic (as $\delta \to 0$), since the model parameter $\theta$ (from which $p(\theta)$ depends) is unknown at the start of the interaction.

Our goal is to devise a $p$-fair $\delta$-PAC algorithm with minimal sample complexity $\tau_{\delta}$, according to the following definition. 

\begin{definition} An algorithm is $p$-fair $\delta$-PAC (resp. asymptotically $p(\theta)$-fair $\delta$-PAC) if for all $\theta \in \Theta$, $\delta\in(0,1/2)$, it satisfies (i) Eq. \eqref{eq:constraint_bai_constant_p} (resp. Eq. \eqref{eq:constraint_bai_ptheta}), (ii) $\mathbb{P}_{\theta}(\hat{a}_{\tau} \neq a^\star) \le \delta$, and (iii) $\mathbb{P}_{\theta}(\tau < \infty) = 1$. 
\end{definition}

\begin{remark}
    \label{remark:examples}
 The {\color{black} fairness rates} that we {\color{black}  consider} are general enough to include the classical notions of \emph{individual fairness} or \emph{proportional fair}. For example, one can set 
$
p(\theta) = p_0 \cdot\texttt{softmax}(\theta) 
$ for some $p_0 \in [0,1]$, or
$p_a(\theta) = p_0\theta_a/\sum_{b} \theta_b$ when the values of $\theta$ are positive. In the latter case, with $p_0=1$, we recover the proportional constraints used in previous works \cite{Wang21, Wang21fair,talebi2018learning} (we refer the reader to the extended related work in \ifdefined\addappendix App.  \ref{sec:extended_related_work} \else App. D \fi  for more details).
On the other hand, by setting a constant constraint $p_a$ we find the classical notion of \emph{selection with pre-specified range} \cite{gajane2022survey}. For example, in \cite{Claure20}, they use the same value of $p_a$ for all arms, while in \cite{Patil21}, the authors select a fixed $p_a\in [0,1/K)$, for all $a\in[K]$.
\end{remark}
\section{Sample Complexity Lower Bound and the Price of Fairness}
\label{sec:lower_bound}
In this section,  we  first provide an instance-specific sample complexity lower bound that is valid for any  $p$-fair $\delta$-PAC algorithm. Next, we analyse the \emph{price of fairness}, quantifying how  fairness constraints impact  sample complexity.

\subsection{Sample complexity lower bound}
The following theorem states a lower bound on the sample complexity of any $p$-fair $\delta$-PAC algorithm. Notably, the sample complexity is characterized by the following constant, that we refer to as \emph{characteristic time}
\begin{equation}
    \label{eq:lb}
T^\star_{p} = {\color{black} 2 }\inf_{w\in \Sigma_{p}}\max_{a\neq a^\star} \frac{w_a^{-1} + w_{a^\star}^{-1}}{\Delta_a^2}, 
    \end{equation}
where $\Sigma_p = \{w \geq p: \sum_{a \in [K]} w_a = 1\}$ is the clipped simplex.

\begin{theorem}
    \label{th:lb}
    Any $p$-fair $\delta$-PAC algorithm satisfies, $\forall \theta \in\Theta$, $\mathbb{E}_{\theta}[\tau_{\delta}]/\log(1/2.4\delta) \ge {\color{black}{2}} T^\star_{p}$. Any asymptotically $p(\theta)$-fair $\delta$-PAC algorithm, instead, we have $\forall \theta \in\Theta, \liminf_{\delta\to 0} \mathbb{E}_{\theta}[\tau_{\delta}]/\log(1/\delta) \ge {\color{black}{2}}T^\star_{p}$.
\end{theorem}
The proof (see \ifdefined\addappendix App.  \ref{sec:lower_bound_price_of_fairness}\else App. B\fi) leverages classical change-of-measure arguments \cite{lai1985asymptotically} and is a straightforward extension of the one in the plain bandit setting by \cite{garivier2016optimal}. The characteristic time $T^\star_{p}$ represents the difficulty of identifying the best arm for a given fairness vector $p=(p_a)_{a\in [K]}$. The \emph{allocation vector} $w = (w_a)_{a\in[K]}$, where $w_a \coloneqq \mathbb{E}_\theta[N_a(\tau_{\delta})]/\mathbb{E}_\theta[\tau_{\delta}]$, characterizes the asymptotic proportion of rounds in which an arm $a$ is selected. Furthermore, an allocation $w^\star_{p}$ solving the optimization problem \eqref{eq:lb} is \textit{optimal} and \textit{fair}: an algorithm relying on a sampling strategy realizing $w^\star_{p}$ would yield the lowest possible sample complexity while satisfying the fairness constraints.

The main difference with the lower bound in \cite{garivier2016optimal} lies in the set of "clipped" allocations $\Sigma_p$, which accounts for the additional fairness constraints $w \ge p$. Notably, for $p = (0,\dots,0)$, we recover the lower bound for the plain BAI setting without fairness constraints \cite{garivier2016optimal}. In this case, we refer to the characteristic constant in this setting as $T^\star \coloneqq T^\star_{0}$ and to the corresponding optimal allocations as $w^\star \coloneqq w^\star_{0}$. Additionally for any $p \neq (0,\dots,0)$, we have  $T^\star_{p} \ge T^\star$; hence, ensuring fairness yields increased sample complexity. \\

\paragraph{The case of unitary $p$} A notable set of cases involves fair BAI instances where $\psum = 1$, including the important example of \textit{proportional fairness} (see \cref{remark:examples}). In these cases, the optimal fair allocations can be simply expressed as $w^\star_{p,a} = p_a$, for all $a\in[K]$. This observation allows to derive a computationally efficient algorithm that avoids the optimization step over $\Sigma_p$, as detailed in Sec. \ref{sec:algo}. This step is typically regarded as the main bottleneck on the computational complexity of BAI algorithms \cite{garivier2016optimal}, especially when the number of arms $K$ grows large. 

\subsection{The Price of Fairness}
The next lemma states an upper bound on the ratio $T^\star_{p}/T^\star$. This ratio quantifies the price in sample complexity that the learner has to pay in order to guarantee fairness. 

\begin{lemma}
\label{lem:ub_ratio}
For a set of fairness constraints $p=(p_a)_{a\in [K]}$, and for all $\theta \in \Theta$, we have that 
\begin{equation}\label{eq:price_fairness}
1 \le \frac{T^\star_{p}}{T^\star} \le O\left( \min \left(\frac{1}{1-\psum}, \frac{1}{K\pmin}\right)\right).
\end{equation}
\end{lemma}
The proof is reported in \ifdefined\addappendix App.  \ref{app:proof_ub_ratio}\else App. B-C\fi. The lemma shows that the price typically scales either as $(1-\psum)^{-1}$ or  $(\pmin)^{-1}$. In the remainder of this section, we discuss two exemplary cases that shed light on the nature of this scaling. We also refer the reader to  \ifdefined\addappendix   App. \ref{subsec:pricefairness} \else App. B-D \fi  for further details and examples.

\paragraph{Case 1, larger fairness rate for suboptimal arms}
\begin{figure}[t]
    \centering
    \includegraphics[width = 0.98\linewidth]{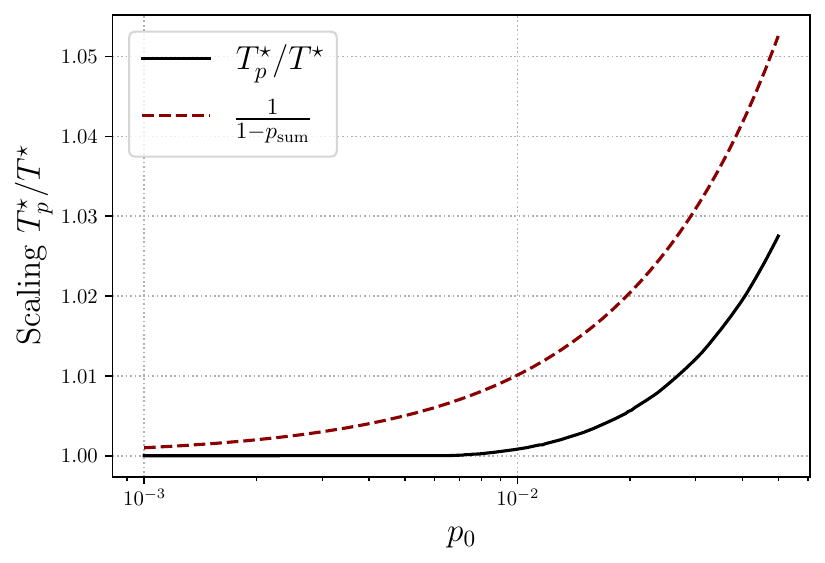}
    \caption{{\bf Price of fairness} for an instance with $K=30$ arms, and higher fairness rates for sub-optimal arms (see Ex. \ref{example_case_1}). The price of fairness $T^\star_{p}/T^\star$ scales closely with $(1-\psum)^{-1}$ }
    \label{fig:price_fairness_case1}
    \vskip -10pt
\end{figure}
We consider a scenario where the fairness rates $p$ for the sub-optimal arms are significantly larger than the optimal frequencies $w^\star$ prescribed in the unconstrained case. 

Under the assumption that $\psum<1$, if $p_{a^\star} \leq w_{a^\star}^\star$ and $p_a\geq w_{a}^\star$ for all $a\neq a^\star$  we can derive the following upper bound on the ratio $T_p^\star/T^\star$:
\begin{equation}
\label{eq:large_fairn_subopt}
    \frac{T_p^\star}{T^\star} \leq \frac{\Delta_{\rm max}^2}{K \Delta_{\rm min}^2}\left(\frac{1}{1-\psum} + \frac{1}{p_{\rm min}}\right).
\end{equation}
See \ifdefined\addappendix   App. \ref{subsec:pricefairness} \else App. B-D \fi for a detailed derivation. However, as explained in the following example, it is possible in some cases to characterize the behavior of $T_p^\star/T^\star$ by solely looking at one of the two terms $(1-\psum)^{-1}$ or $p_{\rm min}^{-1}$.
\\
\begin{example}
\label{example_case_1}
We consider an \emph{antagonistic} scenario with the following $\theta$-dependent fairness rate: $p_a(\theta)= K p_0 \frac{\Delta_a}{\sum_b \Delta_b}$, for $p_0 \in [0,1/K]$. This scenario is termed \emph{antagonistic} because the rates  $p_a(\theta)$ are roughly proportional to $\Delta_a$, and hence, larger for sub-optimal arms. To fix the ideas, consider a model where the rewards
$\theta=(\theta_a)_{a\in[K]}$ are evenly distributed in $[0,1]$. 
For small values of $p_0$ we have that  $p_{\rm min}^{-1}>(1-\psum)^{-1}$, with $(p_{\rm min})^{-1}=O(1/p_0)$.
Despite the term $(p_{\rm min})^{-1}$ being large, we find it is not necessary to characterize $T_p^\star/T^\star$. For this specific model,   as depicted in  \cref{fig:price_fairness_case1}, the ratio $T_p^\star/T^\star$ aligns more closely with $(1-\psum)^{-1}$ for small values of $p_0$. 
\end{example} 
\paragraph{Case 2, the equal gap case}
\begin{figure}[b]
    \centering
    \includegraphics[width=\linewidth]{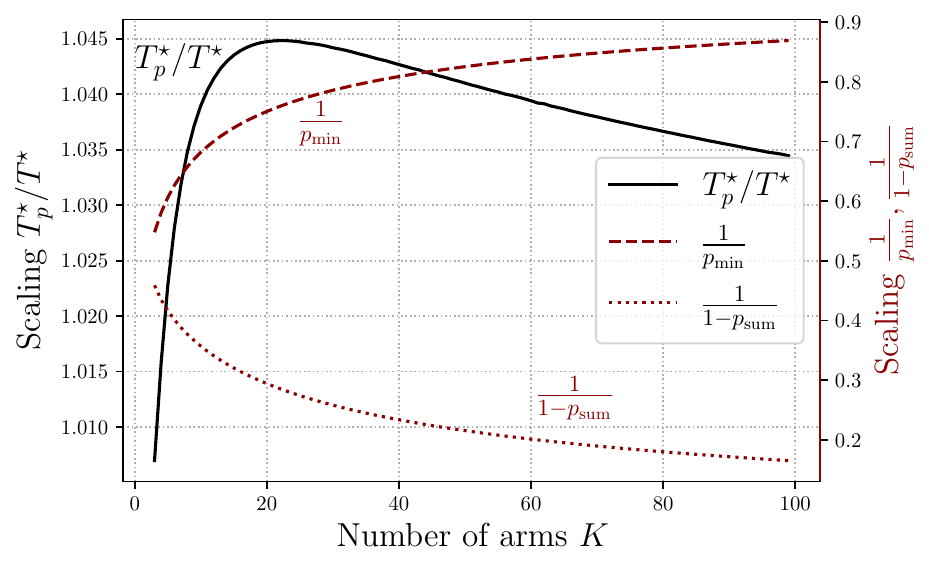}
    \caption{{\bf Price of fairness} for a MAB problem with equal gaps for different number of arms $K$. The fairness constraints are set to discourage exploration of the optimal arm by selecting $p_{a^\star}=0$ and otherwise $p_a=(w_{a}^\star+1/K)/2$ for $a\neq a^\star$.
    In black, on the left axis, it's depicted the ratio $T_p^\star/T^\star$. On the right axis, in dark-red, we plot the individual contributions due to $p_{\rm min}^{-1}$ and $(1-\psum)^{-1}$. }
    \label{fig:price_fairness_case2_equalgap}
    \vskip -10pt
\end{figure}
Another notable case involves instances with equal sub-optimality gaps, i.e., $\Delta_a=\Delta$, for all $a \neq a^\star$. In this case, we can characterize the ratio $T_p^\star/T^\star$ exactly. If $p_{a^\star}=0$ and $p_a\geq w_{a}^\star$, for all $a\neq a^\star$, we have
\begin{equation}\label{eq:tpstar_over_t_equalgap}
 \frac{T_p^\star}{T^\star} = \frac{1}{(1+\sqrt{K-1})^2}\left(\frac{1}{1-\psum}+\frac{1}{p_{\rm min}}\right).
\end{equation}
We refer the reader to \ifdefined\addappendix   App. \ref{subsec:pricefairness} \else App. B-D \fi for a detailed derivation. 

To better understand the scaling of the two terms, we note that in the unconstrained case,   the allocation $w_a^\star$ for all $a\neq a^\star$, decreases as $O(1/K)$. Hence, for a large value of $K$, if $p_a\approx w_a^\star$ for all $a\neq a^\star$, we may expect the term $p_{\rm min}^{-1}$ to be larger than $(1-\psum)^{-1}$  (see \cref{fig:price_fairness_case2_equalgap} for an example of this case). Otherwise, the term $(1-\psum)^{-1}$ is expected to be larger. 
\section{The {\sc \color{black}F-TaS} Algorithm}
\label{sec:algo}
In this section, we propose {\sc \color{black}F-TaS}, an (asymptotically) $p$-fair and $\delta$-PAC algorithm. The algorithm belongs to the family of Track-and-Stop (TaS) algorithms \cite{garivier2016optimal}, which track the sampling allocations $w^\star_p$ as suggested by the solution to the lower bound optimization problem \eqref{eq:lb}. The algorithm mainly consists of (i) a sampling rule and (ii) a stopping rule. We detail these steps in the remainder of this section, and present the pseudo-code for {\sc \color{black}F-TaS} in Alg. \ref{alg:fair_bair}.

\subsection{Sampling rule} 
\label{sec:sampling_rule}
The main idea of the algorithm is that sampling the arms according to $w_p^\star$ is automatically optimal in terms of sample complexity, and satisfies the fairness constraints. However, as the instance parameter $\theta$ is unknown, we leverage a \emph{certainty-equivalence} principle and use the current estimate $\hat{\theta}(t)=(\hat{\theta}_a(t))_{a\in [K]}$ in place of the true parameter.

In the algorithm, we denote by $w_p^\star(t)$ the solution to \cref{eq:lb} with $\hat \theta(t)$ plugged into the expression, i.e.,
$$w_p^\star(t)= \arg\inf_{w\in \Sigma_p} \max_{a\neq a_t^\star} \frac{w_a^{-1} + w_{a_t^\star}^{-1}}{\Delta_a(t)^2},$$ 
where $a_t^\star = \argmax_a \hat\theta_a(t)$ and $\Delta_a(t) = \hat\theta_{a_t^\star}(t) - \hat\theta_a(t)$.

To enforce that the parametric uncertainty asymptotically goes to $0$ (\emph{i.e.}, $\hat\theta(t)\to \theta$ a.s.), we take a convex combination of $w_p^\star(t)$ with a constant policy $\pi_c = (\pi_{c,a})_{a\in[K]}$, using a parameter $\epsilon_t$. This can be interpreted as a form of \textit{forced exploration} \cite{garivier2016optimal} and guarantees that, asymptotically, each arm is sampled infinitely often. 
\begin{algorithm}[b]
   \caption{{\sc \color{black}F-TaS}}
   \label{alg:fair_bair}
\begin{algorithmic}
   \STATE {\bfseries Input:} Fairness vector $p=(p_a)_{a\in [K]}$; confidence $\delta$; {\color{black} forced exploration schedule $(\epsilon_t)_t$}.
     \STATE Set $t\gets 1$
   \WHILE{$Z(t) <\beta(\delta,t)$}
   \STATE Compute $w_{p}^\star(t)$ and set $\pi(t)\gets (1-\epsilon_t)w_{p}^\star(t) + \epsilon_t \pi_c$
   \STATE Select $a_t\sim \pi(t)$ and observe reward $r_t$
   \STATE Update statistics $\hat \theta(t), N_a(t)$ and set $t\gets t+1$
   \ENDWHILE
   \STATE {\bf Return} $\hat{a}_{\tau_\delta} = \argmax_a \hat\theta_a(\tau_\delta)$
\end{algorithmic}
\end{algorithm}

The constant policy $\pi_c$, and the value of $\epsilon_t$ depend on the type of fairness constraint as follows:
\begin{itemize}
    \item \emph{Pre-specified constraints}: Let $K_0=|\{a\in [K]: p_a=0\}|$ be the number of arms for which $p_a = 0$. In the simple case that $K_0=0$, we set $\pi_{c,a}=p_a + (1-\psum)/K$. Otherwise we set $\epsilon_t=1/(2\sqrt{t})$, and define $\pi_c$ as
\[
\pi_{c,a} = \begin{cases}
    p_a & p_a >0\\
    \frac{1-\psum}{K_0} & \hbox{otherwise}.
\end{cases}
\]
    \item $\theta$-\emph{dependent constraints}: in this case, we select $\pi_{c,a} = 1/K$, i.e., a uniform policy for all $a\in [K]$, and we set $\epsilon_t=1/(2\sqrt{t})$.
\end{itemize}

The choice of $\pi_c$ is justified by the fact that, in the pre-specified setting, the fairness constraint naturally induces a linear exploration rate and hence we do not require any additional forced exploration. On the other hand, if the fairness constraints depend on $\theta$, we leverage a uniform policy $\pi_c$ to ensure that each arm is sufficiently explored. 

Note that our tracking procedure is probabilistic (we sample an arm from $\pi(t)$) and differs from the deterministic versions commonly employed in classical Track-and-Stop algorithms \cite{garivier2016optimal}. Therefore, our approach, inspired by best policy identification techniques \cite{al2021navigating}, requires different arguments in order to prove its optimality and fairness guarantee. See also \ifdefined\addappendix App.  \ref{sec:fbai_algo_app} \else App. C \fi for a detailed discussion.

\subsection{Stopping rule}
\label{sec:stopping_rule}
The stopping rule is defined through two components: (1) a generalized-likelihood ratio test (GLRT) $Z(t)$ and (2) a threshold function $\beta(\delta,t)$. Following  \cite{garivier2016optimal}, the GLRT can be expressed as $Z(t)\coloneqq t/ {\color{black}{2}}T^{\star}_{p}(t)$, where
\begin{equation*}
    T^\star_{p}(t) := {\color{black} 2}\max_{a\neq a_t^\star} \frac{w_a(t)^{-1} + w_{a_t^\star}(t)^{-1}}{\Delta_a(t)^2},\; w_a(t) \coloneqq \frac{N_a(t)}{t}.
\end{equation*}
Next, we consider the following threshold function from \cite{kaufmann2021mixture}
\begin{equation*}
\beta(\delta,t) = 3\sum_{a\in[K]} \log(1+\log(N_a(t))) +K {\cal C}_{exp}\left(\frac{\log(\frac{1}{\delta})}{K}\right),
\end{equation*}
where ${\cal C}_{exp}$ is a function defined in Thm. 7 in \cite{kaufmann2021mixture}. 

\subsection{Sample Complexity and Fairness Guarantees} 

\paragraph{Fairness guarantees} We obtain the following guarantees on the fairness of {\sc \color{black}F-TaS}. 
\begin{proposition}
\label{prop:p_delta_PAC}
    {\sc \color{black}F-TaS} is $p$-fair (resp. asymptotically $p(\theta)$-fair) and $\delta$-PAC.  Furthermore, for pre-specified constraints, {\sc \color{black}F-TaS} satisfies the fairness constraints for all rounds, i.e., $\frac{\mathbb{E}[N_a(t)]}{t} \ge p_a, \forall t\ge 1, \; \forall a\in[K]$.
    
\end{proposition}

\begin{corollary}
{\sc \color{black}F-TaS} is 
sample-path $p$-fair (resp. $p(\theta)$-fair), i.e., it satisfies $\mathbb{E}_\theta\left[N_a(\tau_\delta)/\tau_\delta\right]\ge p_a(\theta)$, $\forall a\in[K]$.
\end{corollary}

 \paragraph{Sample complexity guarantees} Next, we establish that our algorithm achieves optimal sample complexity asymptotically (as $\delta \to 0)$.
\begin{theorem}
\label{theorem:sample_complex_optimality}
For all $\delta\in(0,1/2)$, {\sc \color{black}F-TaS} 
has a finite expected sample complexity $\mathbb{E}_\theta[\tau_\delta]<\infty$, and it satisfies: 
\begin{enumerate}
    \item[(1)] Almost sure asymptotic optimality: 
    $$\mathbb{P}_\theta\left( \limsup_{\delta \to 0} \frac{\mathbb{E}_\theta[\tau_\delta]}{\log(1/\delta)} \leq T_{p}^\star \right) =1,$$
    \item[(2)] Asymptotic optimality in expectation: $$\limsup_{\delta \to 0} \frac{\mathbb{E}_\theta[\tau_\delta]}{\log(1/\delta)} \leq T_{p}^\star.$$ 
\end{enumerate}
\end{theorem}
See \ifdefined\addappendix App.  \ref{sec:fairness_delta_pac} \else App. C-A \fi and \ifdefined\addappendix App.  \ref{sec:sample_complex_proofs} \else App. C-B \fi  for a detailed derivation of Prop. \ref{prop:p_delta_PAC} and Thm \ref{theorem:sample_complex_optimality}.

\begin{table*}[t]
\centering
\setlength\arrayrulewidth{1pt}
\resizebox{\textwidth}{!}{ 
    \begin{tabular}{ll|cccc|cccc}\toprule
        & & \multicolumn{4}{c|}{\textbf{Pre-specified constraints}} & \multicolumn{4}{c}{\textbf{$\theta$-dependent constraints}} \\
        & Algorithm & \multicolumn{2}{c}{\textbf{Sample Complexity}} & \multicolumn{2}{c|}{\textbf{Fairness Violation}} & \multicolumn{2}{c}{\textbf{Sample Complexity}} & \multicolumn{2}{c}{\textbf{Fairness Violation}} \\
        & & Agonistic & Antagonistic & Agonistic & Antagonistic & Agonistic & Antagonistic & Agonistic & Antagonistic \\\toprule
\rowcolor[gray]{.95} $\delta=0.1$ & {\sc \color{black}F-TaS}  & $258.63 \pm 27.56$ & $935.25 \pm 129.53$ & $3.21\% \pm 1.30\%$ & $2.11\% \pm 0.69\%$ & $404.89 \pm 48.64$ & $936.52 \pm 130.58$ & $2.48\% \pm 0.46\%$ & $1.64\% \pm 0.21\%$\\ 
  & \textsc{TaS} & $258.87 \pm 38.29$ & $258.87 \pm 38.57$ & $7.50\% \pm 1.59\%$ & $18.57\% \pm 0.33\%$ & $488.01 \pm 73.29$ & $488.01 \pm 74.66$ & $4.48\% \pm 0.48\%$ & $6.31\% \pm 0.08\%$\\ 
\rowcolor[gray]{.95}  & \textsc{Uniform Fair} & $313.67 \pm 28.86$ & $1552.43 \pm 240.79$ & $3.11\% \pm 1.16\%$ & $1.03\% \pm 0.57\%$ & $673.91 \pm 72.51$ & $4905.09 \pm 697.26$ & $1.44\% \pm 0.48\%$ & $0.17\% \pm 0.23\%$\\ \midrule
\rowcolor[gray]{.95} $\delta=0.01$ & {\sc \color{black}F-TaS} & $390.72 \pm 31.42$ & $1522.11 \pm 152.75$ & $1.51\% \pm 0.38\%$ & $1.34\% \pm 0.17\%$ & $658.86 \pm 64.51$ & $1559.28 \pm 192.29$ & $1.79\% \pm 0.29\%$ & $1.20\% \pm 0.12\%$\\ 
  & \textsc{TaS} & $436.01 \pm 49.84$ & $436.01 \pm 49.46$ & $4.05\% \pm 1.24\%$ & $19.33\% \pm 0.13\%$ & $837.94 \pm 96.03$ & $837.94 \pm 95.43$ & $4.78\% \pm 0.37\%$ & $6.63\% \pm 0.08\%$\\ 
\rowcolor[gray]{.95}  & \textsc{Uniform Fair} & $475.54 \pm 33.65$ & $2504.16 \pm 289.01$ & $1.78\% \pm 0.34\%$ & $0.31\% \pm 0.11\%$ & $1052.03 \pm 112.42$ & $7666.75 \pm 948.81$ & $0.97\% \pm 0.45\%$ & $0.03\% \pm 0.08\%$\\ \midrule
\rowcolor[gray]{.95} $\delta=0.001$ & {\sc \color{black}F-TaS} & $508.11 \pm 32.49$ & $2039.71 \pm 199.70$ & $1.22\% \pm 0.32\%$ & $1.12\% \pm 0.17\%$ & $891.82 \pm 68.97$ & $2053.25 \pm 204.73$ & $1.44\% \pm 0.21\%$ & $1.03\% \pm 0.10\%$\\ 
  & \textsc{TaS} & $628.32 \pm 56.20$ & $628.32 \pm 57.14$ & $2.94\% \pm 0.91\%$ & $19.80\% \pm 0.12\%$ & $1272.28 \pm 123.52$ & $1272.28 \pm 119.12$ & $5.12\% \pm 0.34\%$ & $6.97\% \pm 0.09\%$\\ 
\rowcolor[gray]{.95}  & \textsc{Uniform Fair} & $604.36 \pm 36.91$ & $3416.73 \pm 341.82$ & $1.43\% \pm 0.28\%$ & $0.17\% \pm 0.09\%$ & $1433.40 \pm 115.00$ & $10740.08 \pm 1062.28$ & $0.51\% \pm 0.28\%$ & $0.02\% \pm 0.11\%$\\ \bottomrule

    \end{tabular}}
\caption{Synthetic experiments: sample complexity and fairness violation for {\sc \color{black}F-TaS}, {\sc TaS}, and {\sc Uniform Fair}. The fairness violation, as defined in \cref{eq:fairness_violation}, measures the average maximum extent of fairness deviation at the stopping time $\tau_\delta$. }
\vskip -10pt
\label{table:synthetic}
\end{table*}
\section{Numerical Results}
\label{sec:experiments}
In this section, we numerically evaluate the performance of {\sc \color{black}F-TaS}. We propose two sets of experiments: we apply {\sc \color{black}F-TaS} to a synthetic bandit instance (in Sec. \ref{sec:syntetic_exp}), and to an industrial use-case from the radio communication domain: wireless scheduling (in Sec. \ref{sec:resourse_alloc}). Additional results are reported in the appendix.

\paragraph{Fairness criteria}
For both sets of experiments, we focus on two settings: \emph{agonistic fairness} and \emph{antagonistic fairness}. These terms relate to how the fairness parameter $p$ impacts exploration. In the former setting, $p$ promotes exploration (e.g., by aligning with the optimal allocation in the unconstrained setting $w^\star$), while in the latter, it inhibits exploration. We clarify these concepts below in the context of {pre-specified} and {$\theta$-dependent constraints}.
\begin{itemize}
    \item[(i)] \textit{Pre-specified constraints}: we select the fairness vector as $p_{a}= p_0[\alpha w_a^\star + (1-\alpha) \bar w_a^\star$], where $p_0\in(0,1), \alpha\in(0,1)$ and $\bar w_a^\star = (1/w_a^\star) / \sum_{b\in[K]} (1/w_b^\star)$. The parameter $p_0$ regulates the "amount of fairness" in the problem. We set $\alpha=0.9$ for the \emph{agonistic} case and $\alpha=0.1$ in the \emph{antagonistic} one. Note that in the latter case, $p_a$ is almost inversely proportional to the optimal allocation $w^\star$ in the unconstrained case.

    \item[(ii)] \textit{$\theta$-dependent constraints:} in the \emph{agonistic case} we select the fairness functions as $p_a(\theta) = p_0 \frac{1/\max(\Delta_a, \Delta_{\rm min})}{\sum_{b\in[K]}1/\max(\Delta_b, \Delta_{\rm min})}$, with $p_0\in (0,1)$. In the \emph{antagonistic case} we select $p_a(\theta) = p_0 \frac{\Delta_a}{\sum_{b\in[K]}\Delta_b}$. In these two cases, we see how the fairness rates are proportional, or inversely proportional, to the sub-optimality gaps $\Delta_a=\theta_{a^\star}-\theta_a$.
\end{itemize}

\paragraph{Fairness violation}
We measure the \emph{fairness violation} 
at time $t\le \tau_{\delta}$ as $\rho(t) = (\max_a p_a(\theta)-N_a(t)/t)_+$, where and $(x)_+=\max(x,0)$. We also measure the expected fairness violation at the stopping time $\tau_\delta$ as

\begin{align}
    \textrm{Fairness Violation} = \mathbb{E}_\theta[\rho(\tau_\delta)].\label{eq:fairness_violation}
\end{align}

This metric measures the average maximum amount of fairness violation at the stopping time $\tau_\delta$. For instance, a $5\%$ violation, suggests that an arm has been sampled $5\%$ less frequently than the rate prescribed by $p_a$ (at most). 

\paragraph{Baseline algorithms} We compare our {\sc \color{black}F-TaS} algorithm to Track-and-Stop ({\sc TaS}) from \cite{garivier2016optimal}, a baseline that does not consider any fairness constraint, and {\sc Uniform Fair}, an algorithm selecting an arm $a$ in round $t$ with probability $p_a(\hat\theta(t)) + (1-\psum(\hat\theta(t)))/K$. Hence, {\sc Uniform Fair} guarantees that $\mathbb{E}_\theta[N_a(t)]/t\geq p_a , \forall t \geq 1$, or $\lim_{t\to\infty} \mathbb{E}[N_a(t)]/t \geq p_a(\theta)$, for pre-specified or $\theta$-dependent constraints, respectively. We test these algorithms by varying the values of $\delta \in\{10^{-3},10^{-2},10^{-1}\}$. The results are averaged over $N=100$ independent runs. All the confidence intervals refer to $95\%$ confidence. 

\subsection{Synthetic Experiments}
\label{sec:syntetic_exp}

\begin{figure}[b!]
    \centering
    \includegraphics[width=.92\linewidth]{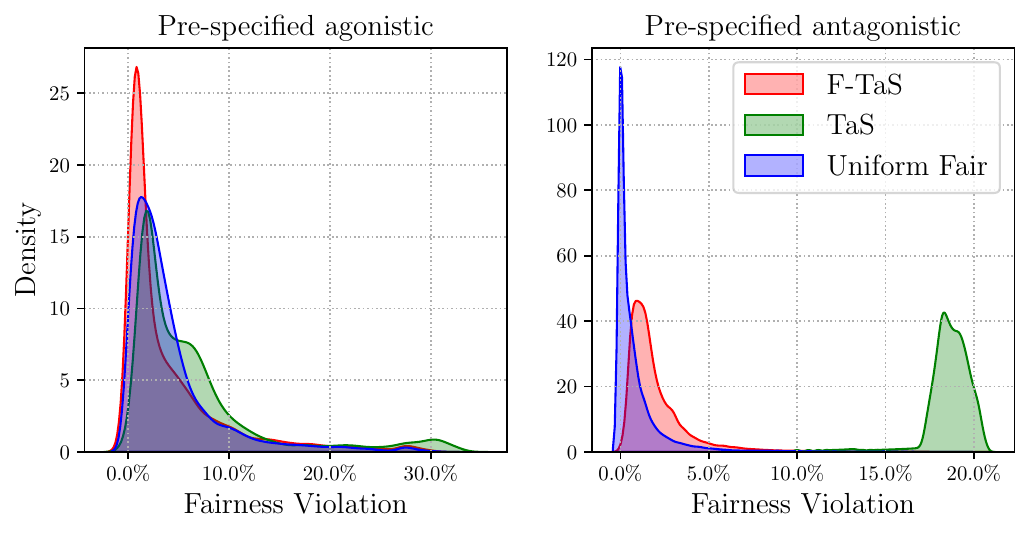}
     \includegraphics[width=.92\linewidth]{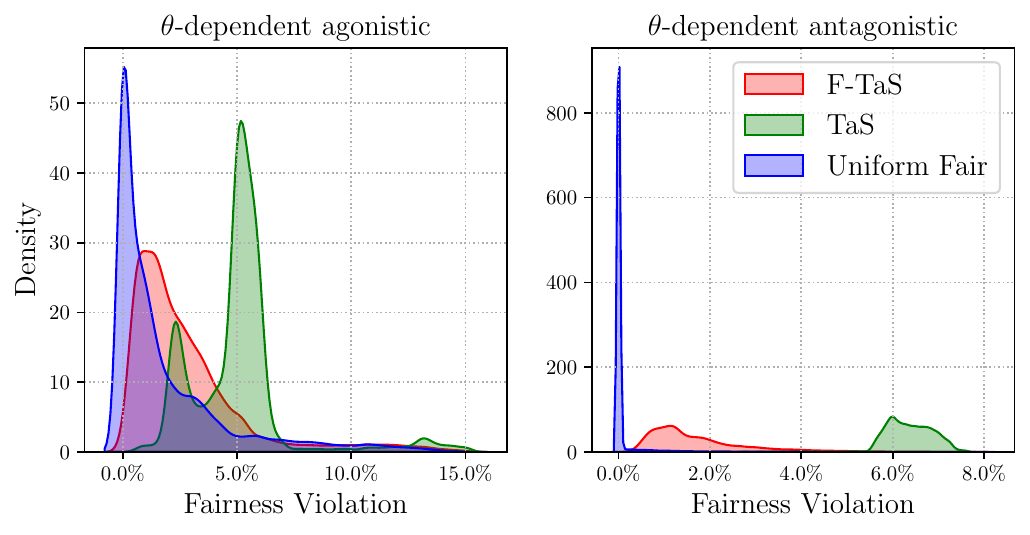}
    \caption{Violations for the synthetic experiments with $\delta=0.01$. Each subplot illustrates the distribution of maximum violation  $\rho(t) = (\max_a p_a(\theta)-N_a(t)/t)_+$, across all rounds $t \leq \tau_\delta$ and  experimental runs.}
    \label{fig:synthetic_violation_main}
\end{figure}
\paragraph{Model} For the pre-specified setting we consider a  model where the expected rewards $(\theta_a)_{a\in[K]}$ linearly range in $[0, K/2.5]$. We consider both \emph{agonistic} and \emph{antagonistic} fairness rates and select $K = 10$, $p_0=0.9$. For $\theta$-dependent constraints we consider an instance with $K=15$ arms with rewards linearly ranging in $[0,5]$, and $p_0=0.7$.\\

\paragraph{Results} In \cref{table:synthetic} we summarize the main results for this experiment for the sample complexity $\mathbb{E}_\theta[\tau_\delta]$ and fairness violation at the stopping time $\mathbb{E}_\theta[\rho(\tau_\delta)]$.

In terms of sample complexity,  {\sc \color{black}F-TaS} shows similar performances to {\sc TaS} in the agonistic setting. This is expected, since in this case the fairness constraints $p$ are closely related to $w^\star$, and thus greatly favor exploration. At the same time {\sc \color{black}F-TaS} is able to guarantee a lower fairness violation, twice as low than {\sc TaS}.

In case the fairness constraints are antagonistic, and thus do not favour exploration, we see how the sample complexity of {\sc \color{black}F-TaS} increases, while still maintaining a low fairness violation. In comparison, the sample complexity of {\sc Uniform Fair} is almost $50\%$ as high, while having similar violations. For $\theta$-dependent constraints the difference in sample complexity is even higher.

In \cref{fig:synthetic_violation_main} we show the distribution of maximum violation over all experimental runs. These results offer a comprehensive view of the algorithms' fairness throughout the duration of observation. Furthermore, the mean of these distributions effectively represents the average violation per round for each algorithm. From the results, we see that the behavior of {\sc \color{black}F-TaS} is close to that of {\sc Fair Uniform}, while {\sc TaS} has larger violations overall.

 \begin{table*}[t]
\centering
\setlength\arrayrulewidth{1pt}
\resizebox{\textwidth}{!}{ 
    \begin{tabular}{ll|cccc|cccc}\toprule
        & & \multicolumn{4}{c|}{\textbf{Pre-specified constraints}} & \multicolumn{4}{c}{\textbf{$\theta$-dependent constraints}} \\
        & Algorithm & \multicolumn{2}{c}{\textbf{Sample Complexity}} & \multicolumn{2}{c|}{\textbf{Fairness Violation}} & \multicolumn{2}{c}{\textbf{Sample Complexity}} & \multicolumn{2}{c}{\textbf{Fairness Violation}} \\
        & & Agonistic & Antagonistic & Agonistic & Antagonistic & Agonistic & Antagonistic & Agonistic & Antagonistic \\\toprule
\rowcolor[gray]{.95} $\delta=0.1$ & {\sc \color{black}F-TaS} & $199.10 \pm 15.96$ & $457.90 \pm 48.15$ & $3.03\% \pm 0.39\%$ & $2.13\% \pm 0.24\%$ & $197.80 \pm 17.05$ & $599.79 \pm 68.83$ & $4.60\% \pm 0.43\%$ & $2.97\% \pm 0.32\%$\\ 
  & \textsc{TaS} & $136.88 \pm 9.59$ & $136.88 \pm 9.78$ & $6.55\% \pm 0.68\%$ & $10.76\% \pm 0.12\%$ & $136.88 \pm 9.48$ & $136.88 \pm 9.86$ & $5.32\% \pm 0.36\%$ & $8.22\% \pm 0.08\%$\\ 
\rowcolor[gray]{.95}  & \textsc{Uniform Fair} & $236.50 \pm 16.11$ & $726.52 \pm 85.13$ & $2.45\% \pm 0.37\%$ & $1.12\% \pm 0.25\%$ & $220.07 \pm 18.00$ & $1889.56 \pm 287.37$ & $4.07\% \pm 0.35\%$ & $1.94\% \pm 0.48\%$\\ \midrule
\rowcolor[gray]{.95} $\delta=0.01$ & {\sc \color{black}F-TaS} & $285.41 \pm 15.74$ & $696.11 \pm 58.62$ & $2.35\% \pm 0.27\%$ & $1.79\% \pm 0.20\%$ & $298.68 \pm 21.88$ & $833.55 \pm 78.24$ & $3.96\% \pm 0.37\%$ & $2.38\% \pm 0.23\%$\\ 
  & \textsc{TaS} & $207.79 \pm 13.53$ & $207.79 \pm 13.64$ & $5.71\% \pm 0.67\%$ & $11.14\% \pm 0.13\%$ & $207.79 \pm 13.84$ & $207.79 \pm 13.28$ & $4.92\% \pm 0.37\%$ & $8.55\% \pm 0.11\%$\\ 
\rowcolor[gray]{.95}  & \textsc{Uniform Fair} & $323.86 \pm 19.23$ & $1071.62 \pm 91.97$ & $1.91\% \pm 0.29\%$ & $0.68\% \pm 0.18\%$ & $359.49 \pm 24.66$ & $2853.99 \pm 319.41$ & $3.00\% \pm 0.26\%$ & $1.21\% \pm 0.40\%$\\ \midrule
\rowcolor[gray]{.95} $\delta=0.001$ & {\sc \color{black}F-TaS} & $358.81 \pm 17.44$ & $899.13 \pm 74.28$ & $2.00\% \pm 0.29\%$ & $1.60\% \pm 0.18\%$ & $398.94 \pm 24.53$ & $1048.52 \pm 84.89$ & $3.43\% \pm 0.34\%$ & $2.02\% \pm 0.18\%$\\ 
  & \textsc{TaS} & $271.05 \pm 16.99$ & $271.05 \pm 16.87$ & $5.22\% \pm 0.62\%$ & $11.51\% \pm 0.10\%$ & $271.05 \pm 16.93$ & $271.05 \pm 17.11$ & $4.67\% \pm 0.33\%$ & $8.90\% \pm 0.10\%$\\ 
\rowcolor[gray]{.95}  & \textsc{Uniform Fair} & $410.72 \pm 22.63$ & $1383.06 \pm 95.08$ & $1.52\% \pm 0.21\%$ & $0.41\% \pm 0.12\%$ & $476.13 \pm 32.11$ & $3703.97 \pm 354.92$ & $2.58\% \pm 0.24\%$ & $0.86\% \pm 0.37\%$\\ \bottomrule
    \end{tabular}}
\caption{Sample complexity and fairness violations for the scheduling experiments. 
}
\vskip -10pt
\label{tab:TabCombined_scheduling}
\end{table*}

\subsection{Wireless scheduling}
\label{sec:resourse_alloc}

\paragraph{Model}
\begin{wrapfigure}{r}{0.2\textwidth} 
    \centering
        \includegraphics[width=0.2\textwidth]{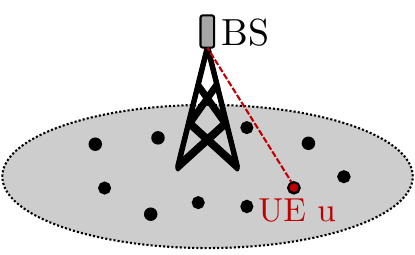} 
    \caption{Visual depiction of the scheduling environment with $K = 10$ UEs.}
\label{fig:scheduling_env}
\end{wrapfigure}

We consider a wireless radio environment with a Base Station (BS) and a set of $K$ User Equipments (UEs) connected to the BS (see \cref{fig:scheduling_env}). Communication proceeds in time slots in a down-link fashion. The BS is placed at the center of a cell (or sector) of radius $d$ [m] (measured in meters), and the $K$ UEs are randomly distributed in the cell.

At each round, $t\ge 1$, the BS selects a single UE out of the $K$ to be scheduled for transmission. Naturally, in this formulation, the BS represents the learner, and the set of UEs $[K]$ represents the various arms that can be selected by the BS at each round.

\paragraph{Objective} The objective is to maximize the sum throughput across all UEs. The throughput $T_{u,t}$ of UE $u$ at round $t$ represents the rate at which information is delivered to the UE. This quantity depends on channel conditions (or {\it fading}) between the BS antenna and the user. These conditions rapidly evolve over time around their mean. The fadings between pairs of (antenna, user) are typically stochastically independent across users and antennas \cite{Tse09}, and we assume that can be modeled as independent Gaussian r.v. The reward at round $t$ is defined as the sum throughput across UEs in the cell, i.e.,
$r_t = \sum_{u \in [K]} T_{u,t} \mathbf{1}_{\{a_t = u\}}.$\\

\begin{figure}[b!]
    \centering
\includegraphics[width = \linewidth]{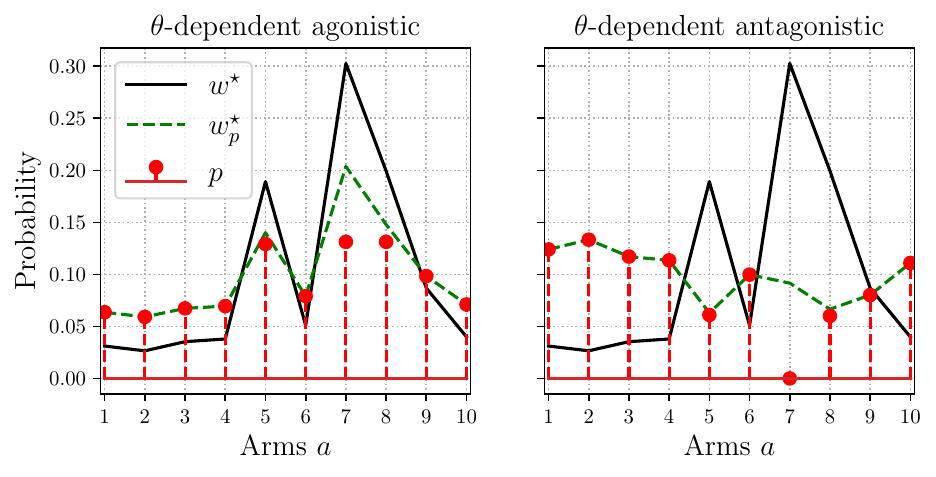}
    \caption{Allocations for the scheduling experiments with $\theta$-dependent constraints. Agonistic constraints favour exploration, since $w_p^\star\approx w^\star$, while antagonistic ones discourage exploration of good arms. }
    \label{fig:allocations_scheduling}
    
\end{figure}

\paragraph{Fairness constraints} In wireless scheduling, the fairness constraints represent the minimal fraction of rounds in which each UE is scheduled for transmission. This constraint naturally captures UE guarantees in terms of throughput: the higher the number of slots in which a UE is scheduled, the higher will be the throughput experienced. \\

\paragraph{Experimental setup}
We test {\sc \color{black}F-TaS} using mob-env, an open-source simulation environment \cite{Schneider2022mobile} based on the gymnasium interface. As for the synthetic setting, we consider two sets of experiments with \textit{pre-specified} and \textit{$\theta$-dependent} fairness. The fairness parameter $p$ and the optimal allocations $w^\star$ and $w^\star_p$ for the $\theta$-dependent setting are shown in \cref{fig:allocations_scheduling}. We set the number of UEs to $K = 10$ and  $p_0=0.9$. We refer the reader to the appendix for more details on the model and experimental setup.\\

\paragraph{Results} The sample complexity and fairness violation results are presented in  \cref{tab:TabCombined_scheduling}, while  \cref{fig:violations_scheduling} shows the distribution of the fairness violation metric. The results are generally in line with the experimental findings of the previous section: {\sc \color{black} F-TaS} achieves lower violation w.r.t. the non-fair baseline ({\sc TaS}) while outperforming the fair baseline ({\sc Uniform-Fair}) in terms of sample complexity. 
\begin{figure}[h!]
\vskip -10pt
    \centering
    \includegraphics[width=.92\linewidth]{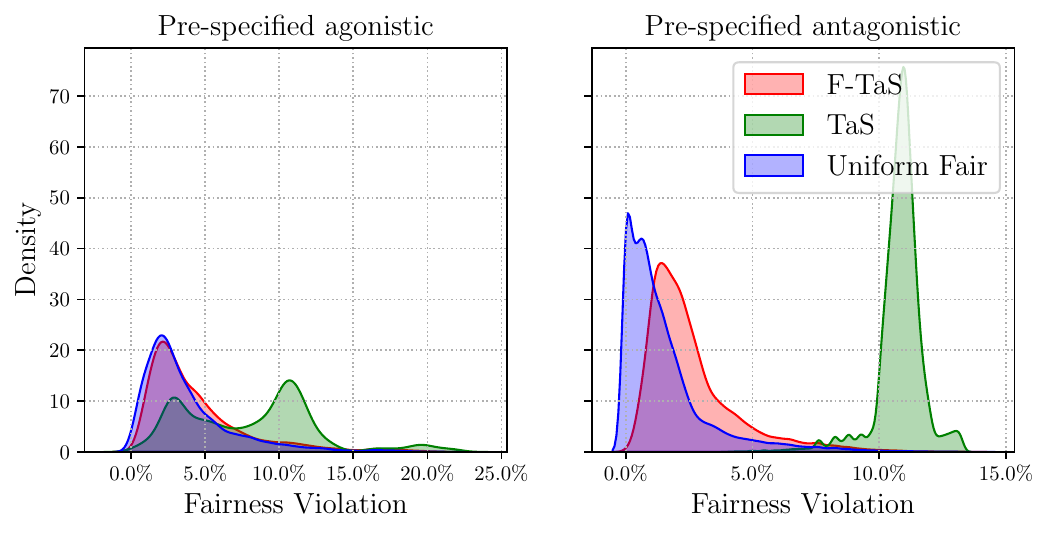}

    \includegraphics[width=.92\linewidth]{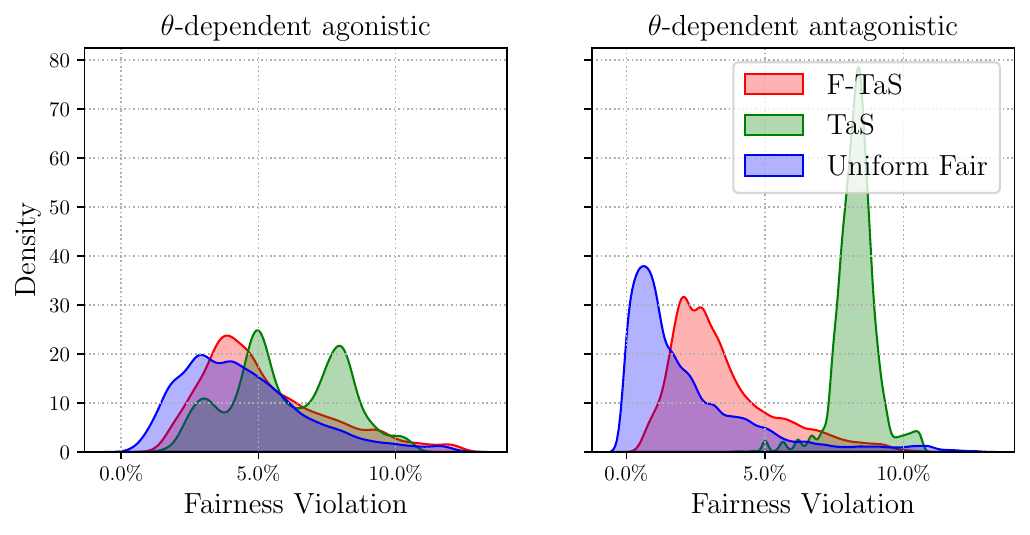}
     \caption{Distribution of the fairness violations density for the scheduling experiments.}
     \label{fig:violations_scheduling}
     \vskip -10pt
\end{figure}

\section{Conclusions}
\label{sec:conclusion}
In this paper, we introduced Fair Best Arm Identification (Fair BAI), a novel setting that integrates the classical BAI framework with fairness constraints, which are either model-agnostic or model-dependent.

For both scenarios, we derived a sample complexity lower bound and quantified the price of fairness in terms of sample complexity. Leveraging this lower bound, we devised {\sc \color{black}F-TaS}, an algorithm that provably matches this bound while complying with the fairness constraints.

Our experimental results, obtained from both synthetic and wireless scheduling scenarios, demonstrate that {\sc \color{black}F-TaS} effectively achieves low sample complexity while minimizing fairness violations.

The limitations of our work include: (i) the asymptotic nature of our fairness constraints in the $\theta$-dependent constraint; (ii) the sample complexity analysis operates in the asymptotic regime; (iii) quantifying the variance of our method is technically challenging.
Future research directions involve extending the fair BAI concept to bandits with additional structures, such as linear, Lipschitz, and unimodal bandits. Furthermore, incorporating regret minimization into our framework represents another exciting area for exploration.

\balance
\bibliography{example_paper}

\begin{thebibliography}{10}

\bibitem{caton2020fairness}
Simon Caton and Christian Haas.
\newblock Fairness in machine learning: A survey.
\newblock {\em ACM Computing Surveys}, 2020.

\bibitem{gajane2022survey}
Pratik Gajane, Akrati Saxena, Maryam Tavakol, George Fletcher, and Mykola
  Pechenizkiy.
\newblock Survey on fair reinforcement learning: Theory and practice.
\newblock {\em arXiv preprint arXiv:2205.10032}, 2022.

\bibitem{lai1985asymptotically}
Tze~Leung Lai and Herbert Robbins.
\newblock Asymptotically efficient adaptive allocation rules.
\newblock {\em Advances in applied mathematics}, 1985.

\bibitem{xu2013estimation}
Min Xu, Tao Qin, and Tie-Yan Liu.
\newblock Estimation bias in multi-armed bandit algorithms for search
  advertising.
\newblock In {\em Proc. of NeurIPS}, 2013.

\bibitem{Ariu2020regret}
Kaito Ariu, Narae Ryu, Se-Young Yun, and Alexandre Prouti{\`e}re.
\newblock Regret in online recommendation systems.
\newblock In {\em Proc. of NeurIPS}, 2020.

\bibitem{chu2011contextual}
Wei Chu, Lihong Li, Lev Reyzin, and Robert Schapire.
\newblock Contextual bandits with linear payoff functions.
\newblock In {\em Proc. of AISTATS}, 2011.

\bibitem{nguyen2019scheduling}
Thi Thuy~Nga Nguyen, Urtzi Ayesta, and Balakrishna Prabhu.
\newblock Scheduling users in drive-thru internet: a multi-armed bandit
  approach.
\newblock In {\em 2019 International Symposium on Modeling and Optimization in
  Mobile, Ad Hoc, and Wireless Networks (WiOPT)}, 2019.

\bibitem{Bouneffouf20_survey}
Djallel Bouneffouf, Irina Rish, and Charu Aggarwal.
\newblock Survey on applications of multi-armed and contextual bandits.
\newblock In {\em IEEE Congress on Evolutionary Computation (CEC)}, 2020.

\bibitem{Li19combinatorial}
Fengjiao Li, Jia Liu, and Bo~Ji.
\newblock Combinatorial sleeping bandits with fairness constraints.
\newblock {\em IEEE Transactions on Network Science and Engineering}, 2019.

\bibitem{Joseph16}
Matthew Joseph, Michael Kearns, Jamie~H Morgenstern, and Aaron Roth.
\newblock Fairness in learning: Classic and contextual bandits.
\newblock In {\em Proc. of NeurIPS}, 2016.

\bibitem{Patil21}
Vitshakha Patil, Ganesh Ghalme, Vineet Nair, and Yadati Narahari.
\newblock Achieving fairness in the stochastic multi-armed bandit problem.
\newblock In {\em JMLR}, 2021.

\bibitem{zhang2021fairness}
Xueru Zhang and Mingyan Liu.
\newblock Fairness in learning-based sequential decision algorithms: A survey.
\newblock In {\em Handbook of Reinforcement Learning and Control}. Springer,
  2021.

\bibitem{audibert2010best}
Jean-Yves Audibert, S{\'e}bastien Bubeck, and R{\'e}mi Munos.
\newblock Best arm identification in multi-armed bandits.
\newblock In {\em Proc. of COLT}, 2010.

\bibitem{garivier2016optimal}
Aur{\'e}lien Garivier and Emilie Kaufmann.
\newblock Optimal best arm identification with fixed confidence.
\newblock In {\em Proc. of COLT}. PMLR, 2016.

\bibitem{jabbari2017fairness}
Shahin Jabbari, Matthew Joseph, Michael Kearns, Jamie Morgenstern, and Aaron
  Roth.
\newblock Fairness in reinforcement learning.
\newblock In {\em ICML}, 2017.

\bibitem{grazzi2022group}
Riccardo Grazzi, Arya Akhavan, John~IF Falk, Leonardo Cella, and Massimiliano
  Pontil.
\newblock Group meritocratic fairness in linear contextual bandits.
\newblock In {\em Proc. of NeurIPS}, 2022.

\bibitem{Claure20}
Houston Claure, Yifang Chen, Jignesh Modi, Malte Jung, and Stefanos Nikolaidis.
\newblock Multi-armed bandits with fairness constraints for distributing
  resources to human teammates.
\newblock In {\em Proc. of the ACM/IEEE International Conference on Human-Robot
  Interaction}, 2020.

\bibitem{liu2022combinatorial}
Qingsong Liu, Weihang Xu, Siwei Wang, and Zhixuan Fang.
\newblock Combinatorial bandits with linear constraints: Beyond knapsacks and
  fairness.
\newblock In {\em Proc. of NeurIPS}, 2022.

\bibitem{Chen20}
Yifang Chen, Alex Cuellar, Haipeng Luo, Jignesh Modi, Heramb Nemlekar, and
  Stefanos Nikolaidis.
\newblock Fair contextual multi-armed bandits: Theory and experiments.
\newblock In {\em Conference on Uncertainty in Artificial Intelligence}, 2020.

\bibitem{celis2019controlling}
L~Elisa Celis, Sayash Kapoor, Farnood Salehi, and Nisheeth Vishnoi.
\newblock Controlling polarization in personalization: An algorithmic
  framework.
\newblock In {\em Proc. of the conference on fairness, accountability, and
  transparency}, 2019.

\bibitem{dwork2012fairness}
Cynthia Dwork, Moritz Hardt, Toniann Pitassi, Omer Reingold, and Richard Zemel.
\newblock Fairness through awareness.
\newblock In {\em Proc. of the 3rd innovations in theoretical computer science
  conference}, 2012.

\bibitem{joseph2016fair}
Matthew Joseph, Michael Kearns, Jamie~H Morgenstern, and Aaron Roth.
\newblock Fairness in learning: Classic and contextual bandits.
\newblock In D.~Lee, M.~Sugiyama, U.~Luxburg, I.~Guyon, and R.~Garnett,
  editors, {\em Advances in Neural Information Processing Systems}, volume~29.
  Curran Associates, Inc., 2016.

\bibitem{liu2017calibrated}
Yang Liu, Goran Radanovic, Christos Dimitrakakis, Debmalya Mandal, and David~C
  Parkes.
\newblock Calibrated fairness in bandits.
\newblock {\em arXiv preprint arXiv:1707.01875}, 2017.

\bibitem{gillen2018online}
Stephen Gillen, Christopher Jung, Michael Kearns, and Aaron Roth.
\newblock Online learning with an unknown fairness metric.
\newblock In {\em Proc. of NeurIPS}, 2018.

\bibitem{Wang21}
Tianyu Wang and Cynthia Rudin.
\newblock Bandit learning for proportionally fair allocations.
\newblock \url{https://wangt1anyu.github.io/papers/prop-fair-bandit.pdf}, 2021.

\bibitem{Wang21fair}
Lequn Wang, Yiwei Bai, Wen Sun, and Thorsten Joachims.
\newblock Fairness of exposure in stochastic bandits.
\newblock In {\em ICML}, 2021.

\bibitem{atkinson1970measurement}
Anthony~B Atkinson et~al.
\newblock On the measurement of inequality.
\newblock {\em Journal of economic theory}, 1970.

\bibitem{radunovic2007unified}
Bozidar Radunovic and Jean-Yves Le~Boudec.
\newblock A unified framework for max-min and min-max fairness with
  applications.
\newblock {\em IEEE/ACM Transactions on networking}, 2007.

\bibitem{si2022enabling}
Tareq Si~Salem, Georgios Iosifidis, and Giovanni Neglia.
\newblock Enabling long-term fairness in dynamic resource allocation.
\newblock {\em Proceedings of the ACM on Measurement and Analysis of Computing
  Systems}, 2022.

\bibitem{talebi2018learning}
Mohammad~Sadegh Talebi and Alexandre Proutiere.
\newblock Learning proportionally fair allocations with low regret.
\newblock {\em Proc. of the ACM on Measurement and Analysis of Computing
  Systems}, 2018.

\bibitem{wu2023best}
Yuhang Wu, Zeyu Zheng, and Tingyu Zhu.
\newblock Best arm identification with fairness constraints on subpopulations.
\newblock {\em arXiv preprint arXiv:2304.04091}, 2023.

\bibitem{kaufmann2016complexity}
Emilie Kaufmann, Olivier Capp{\'e}, and Aur{\'e}lien Garivier.
\newblock On the complexity of best arm identification in multi-armed bandit
  models.
\newblock In {\em JMLR}, 2016.

\bibitem{al2021navigating}
Aymen Al~Marjani, Aur{\'e}lien Garivier, and Alexandre Proutiere.
\newblock Navigating to the best policy in markov decision processes.
\newblock In {\em Proc. of NeurIPS}, 2021.

\bibitem{kaufmann2021mixture}
Emilie Kaufmann and Wouter~M Koolen.
\newblock Mixture martingales revisited with applications to sequential tests
  and confidence intervals.
\newblock In {\em JMLR}, 2021.

\bibitem{Tse09}
David N.~C. Tse and Pramod Viswanath.
\newblock Fundamentals of wireless communication.
\newblock {\em IEEE Trans. Inf. Theory}, 2009.

\bibitem{Schneider2022mobile}
Stefan Schneider, Stefan Werner, Ramin Khalili, Artur Hecker, and Holger Karl.
\newblock mobile-env: An open platform for reinforcement learning in wireless
  mobile networks.
\newblock In {\em NOMS IEEE/IFIP Network Operations and Management Symposium},
  2022.

\bibitem{Luenberger97}
David~G. Luenberger.
\newblock {\em Optimization by Vector Space Methods}.
\newblock John Wiley \& Sons, Inc., USA, 1st edition, 1997.

\bibitem{hall2014martingale}
Peter Hall and Christopher~C Heyde.
\newblock {\em Martingale limit theory and its application}.
\newblock Academic press, 2014.

\bibitem{russo2023sample}
Alessio Russo and Alexandre Proutiere.
\newblock On the sample complexity of representation learning in multi-task
  bandits with global and local structure.
\newblock In {\em Proceedings of the AAAI Conference on Artificial
  Intelligence}, 2023.

\bibitem{lattimore2020bandit}
Tor Lattimore and Csaba Szepesv{\'a}ri.
\newblock {\em Bandit algorithms}.
\newblock Cambridge University Press, 2020.

\bibitem{schumann2019group}
Candice Schumann, Zhi Lang, Nicholas Mattei, and John~P Dickerson.
\newblock Group fairness in bandit arm selection.
\newblock {\em arXiv preprint arXiv:1912.03802}, 2019.

\bibitem{huang2022achieving}
Wen Huang, Kevin Labille, Xintao Wu, Dongwon Lee, and Neil Heffernan.
\newblock Achieving user-side fairness in contextual bandits.
\newblock {\em Human-Centric Intelligent Systems}, 2022.

\bibitem{kusner2017counterfactual}
Matt~J Kusner, Joshua Loftus, Chris Russell, and Ricardo Silva.
\newblock Counterfactual fairness.
\newblock In {\em Proc. of NeurIPS}, 2017.

\bibitem{baek2021fair}
Jackie Baek and Vivek Farias.
\newblock Fair exploration via axiomatic bargaining.
\newblock In {\em Proc. of NeurIPS}, 2021.

\bibitem{Hossain21fair}
Safwan Hossain, Evi Micha, and Nisarg Shah.
\newblock Fair algorithms for multi-agent multi-armed bandits.
\newblock In {\em Proc. of NeurIPS}, 2021.

\bibitem{metevier2019offline}
Blossom Metevier, Stephen Giguere, Sarah Brockman, Ari Kobren, Yuriy Brun, Emma
  Brunskill, and Philip~S Thomas.
\newblock Offline contextual bandits with high probability fairness guarantees.
\newblock In {\em Proc. of NeurIPS}, 2019.

\end{thebibliography}

\bibliographystyle{unsrt}

\ifdefined \addappendix
\onecolumn
\appendices

\tableofcontents

\section{Additional Numerical Results and Detailed Experimental Setting}

\label{sec:additional_experiments}
\setcounter{subsection}{0}
In this section we present additional numerical results. First, we briefly summarize some technical information regarding the experiments. Then, in App. \ref{subsec:synth_exp} we present additional results on the synthetic model, with pre-specified constraints and $\theta$-dependent constraints. Later, in App. \ref{sec:additional_results_scheduling}, we present additional details and experiments on the scheduling problem. 

As previously mentioned, we tested all the algorithms for different values of $\delta \in\{10^{-3},10^{-2},10^{-1}\}$. The results are averaged over $N=100$ independent runs. All the confidence intervals refer to $95\%$ confidence.  Moreover, in our implementation, we tested both the exploration threshold in \cref{sec:stopping_rule} $\beta(\delta,t) = 3\sum_{a\in[K]} \log(1+\log(N_a(t))) +K {\cal C}_{exp}\left(\frac{\log(\frac{1}{\delta})}{K}\right)$ \cite{kaufmann2021mixture}, and   $\beta(t,\delta) = \log((\log(t) + 1)/\delta)$ introduced in \cite{garivier2016optimal}.  We report the results using the latter threshold for simplicity. Lastly, the instructions to run the code can be found in the {\tt README.md} file in the supplementary material.

\paragraph{Fairness criteria}
For all experiments, we focus on two settings: \emph{agonistic fairness} and \emph{antagonistic fairness}. These terms relate to how the fairness parameter $p$ impacts exploration. In the former setting, $p$ promotes exploration (e.g., by aligning with the optimal allocation in the unconstrained setting $w^\star$), while in the latter, it inhibits exploration. We clarify these concepts below in the context of {pre-specified} and {$\theta$-dependent rates}.

\begin{itemize}
    \item[(i)] \textit{Pre-specified constraints}: we select the fairness vector as $p_{a}= p_0[\alpha w_a^\star + (1-\alpha) \bar w_a^\star$], where $p_0\in(0,1), \alpha\in(0,1)$ and $\bar w_a^\star = (1/w_a^\star) / \sum_{b\in[K]} (1/w_b^\star)$. The parameter $p_0$ regulates the "amount of fairness" in the problem. We set $\alpha=0.9$ for the \emph{agonistic} case and $\alpha=0.1$ in the \emph{antagonistic} one. Note that in the latter case, $p_a$ is almost inversely proportional to the optimal allocation in the unconstrained case.

    \item[(ii)] \textit{$\theta$-dependent constraints:} in the \emph{agonistic case} we select the fairness functions as $p_a(\theta) = p_0 \frac{1/\max(\Delta_a, \Delta_{\rm min})}{\sum_{b\in[K]}1/\max(\Delta_b, \Delta_{\rm min})}$, with $p_0\in (0,1)$. In the \emph{antagonistic case} we select $p_a(\theta) = p_0 \frac{\Delta_a}{\sum_{b\in[K]}\Delta_b}$. In these two cases, we see how the fairness rates are proportional, or inversely proportional, to the sub-optimality gaps $\Delta_a=\theta_{a^\star}-\theta_a$.
\end{itemize}

\paragraph{Fairness violation} 
We measure the \emph{fairness violation} 
at time $t\le \tau_{\delta}$ as $\rho(t) = (\max_a p_a(\theta)-N_a(t)/t)_+$, where and $(x)_+=\max(x,0)$. We also measure the expected fairness violation at the stopping time $\tau_\delta$ as
\begin{align*}
    \textrm{Fairness Violation} = \mathbb{E}_\theta[\rho(\tau_\delta)].
\end{align*}
\subsection{Synthetic experiments}\label{subsec:synth_exp}
\subsubsection{Synthetic Model with Pre-specified Constraints}\label{subsec:synth_exp_prespecified}

\paragraph{Model} We considered two bandit models. First, a model where the expected rewards $(\theta_a)_{a\in[K]}$ linearly range in $[0, K/2.5]$, with $K=10$ and $p_0=0.9$. Secondly, a model where all the suboptimal gaps $\Delta_a$ have the same value $\Delta= K/5$, and we set $p_0=0.99$.

\paragraph{Allocations} In \cref{fig:synthetic_prespecified_allocations} are depicted the optimal unconstrained allocation $w^\star$, the constrained one $w_p^\star$, and the fairness constraints $(p_a)_{a\in [K]}$. In the agonistic case we see how the fairness rates are closely related to the optimal exploration, while in the antagonistic one are inversely proportional.
\begin{figure}[htb!]
    \centering
    \includegraphics[width=.8\linewidth]{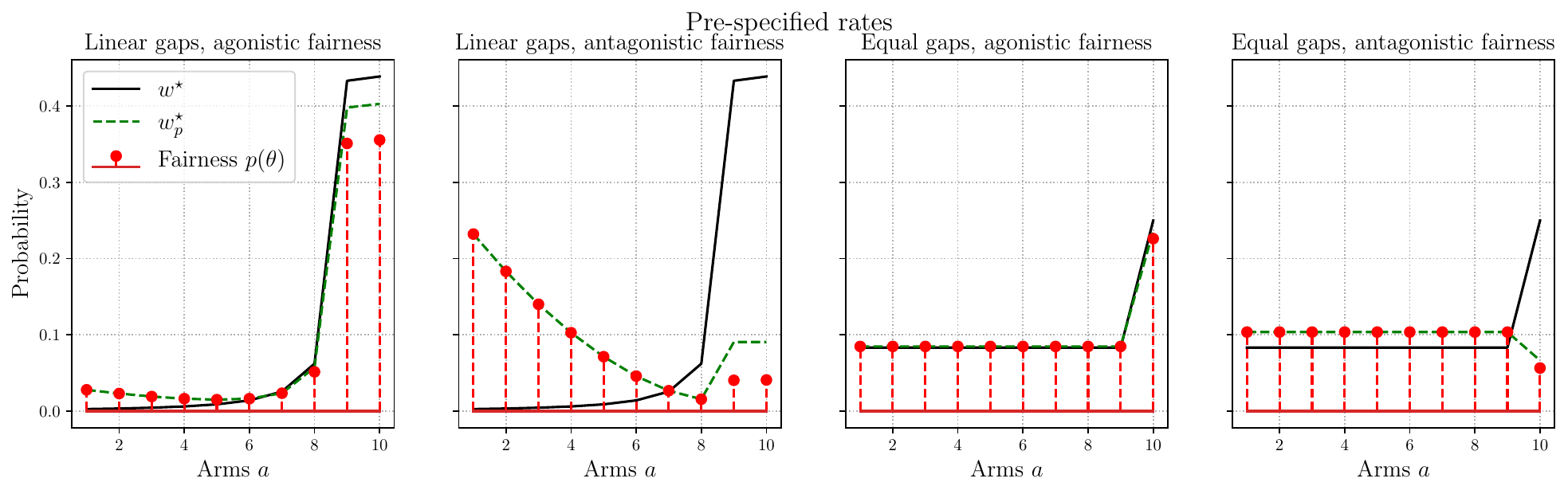}
    \caption{Synthetic experiments with pre-specified constraints. We show the the optimal unconstrained allocation $w^\star$, the constrained one $w_p^\star$, and the fairness constraints $(p_a)_{a\in [K]}$ for both the model with rewards linearly ranging in $[0,K/2.5]$ and the model with equal-gaps.}
    \label{fig:synthetic_prespecified_allocations}
\end{figure}

\paragraph{Sample complexity} In \cref{fig:synthetic_prespecified_samplecomplexity} we show the sample complexity results for each case, as well as the unconstrained sample complexity lower bound, and the constrained one. 

\begin{figure}[h!]
    \centering
    \includegraphics[width=.67\linewidth]{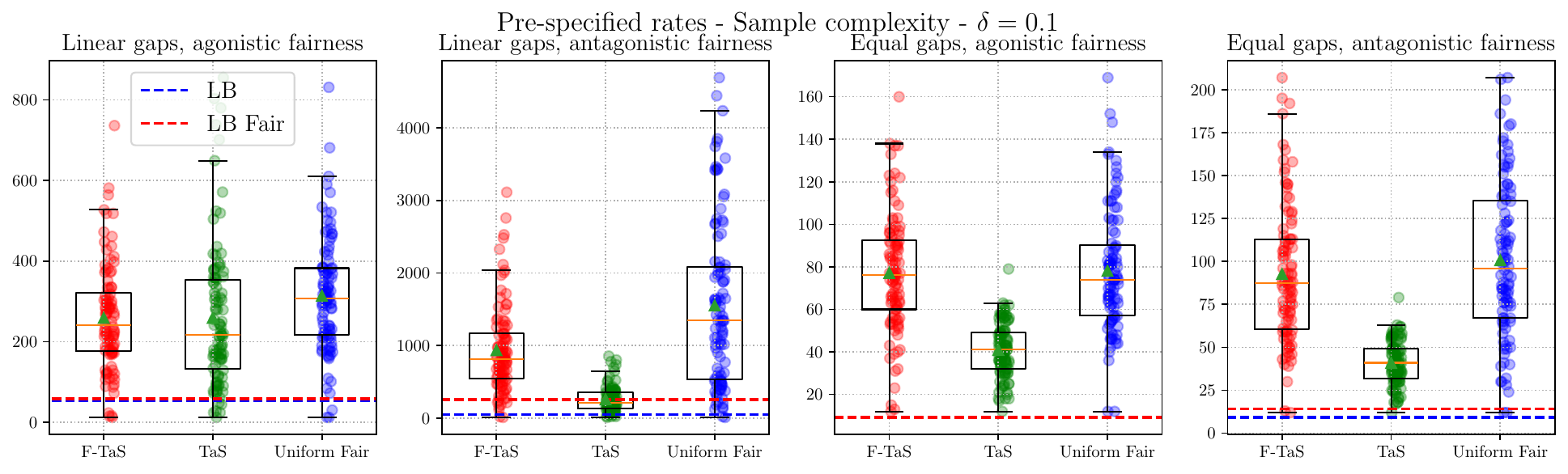}
    \includegraphics[width=.67\linewidth]{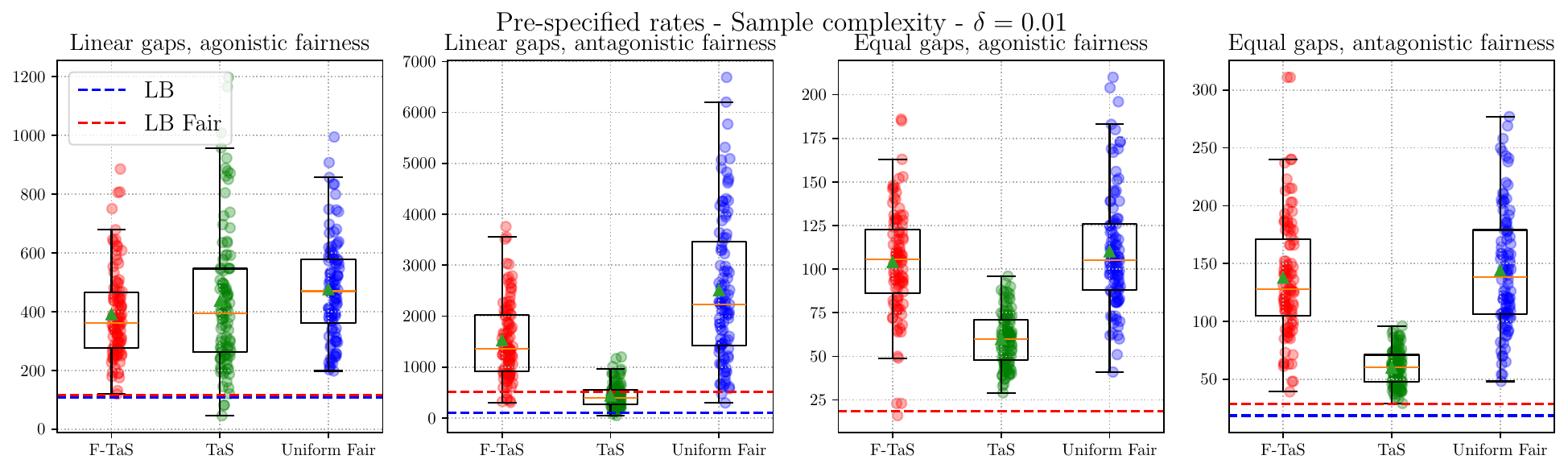}
    \includegraphics[width=.67\linewidth]{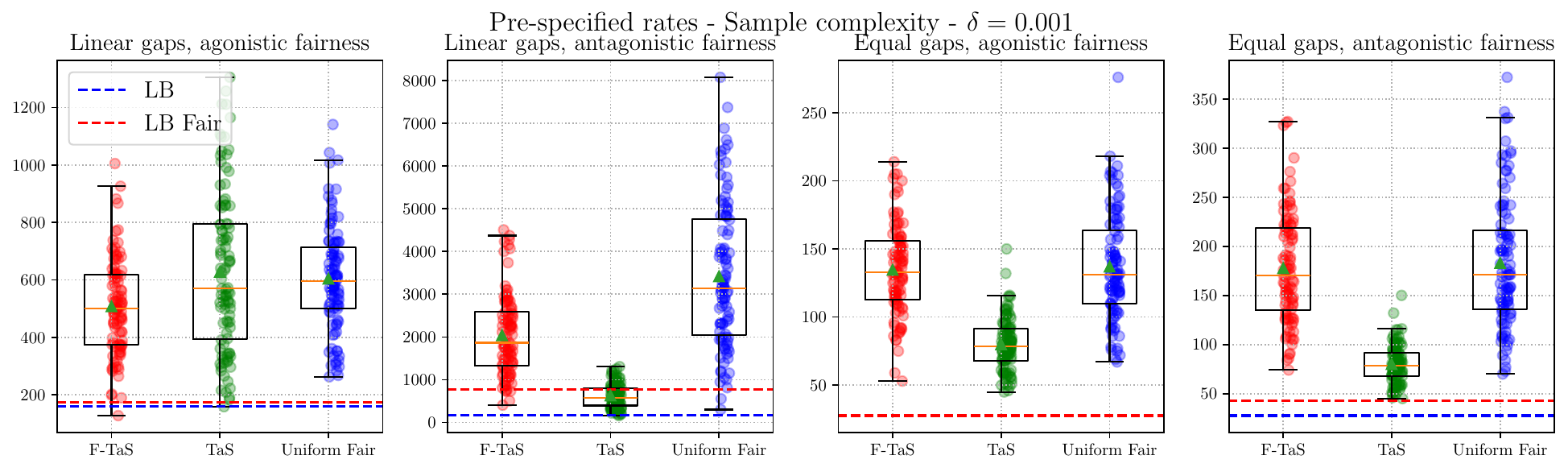}
    \caption{Synthetic model with pre-specified constraints. Sample complexity results for different values of $\delta$ are shown in each row.}
    \label{fig:synthetic_prespecified_samplecomplexity}
\end{figure}

\paragraph{Fairness violation} In \cref{fig:synthetic_prespecified_violations} we depict an aggregate distribution of the fairness violation $\rho(t)$ over all rounds. These plots offer a comprehensive understanding of the behavior of the algorithm.

\begin{figure}[H]
    \centering
    \includegraphics[width=.67\linewidth]{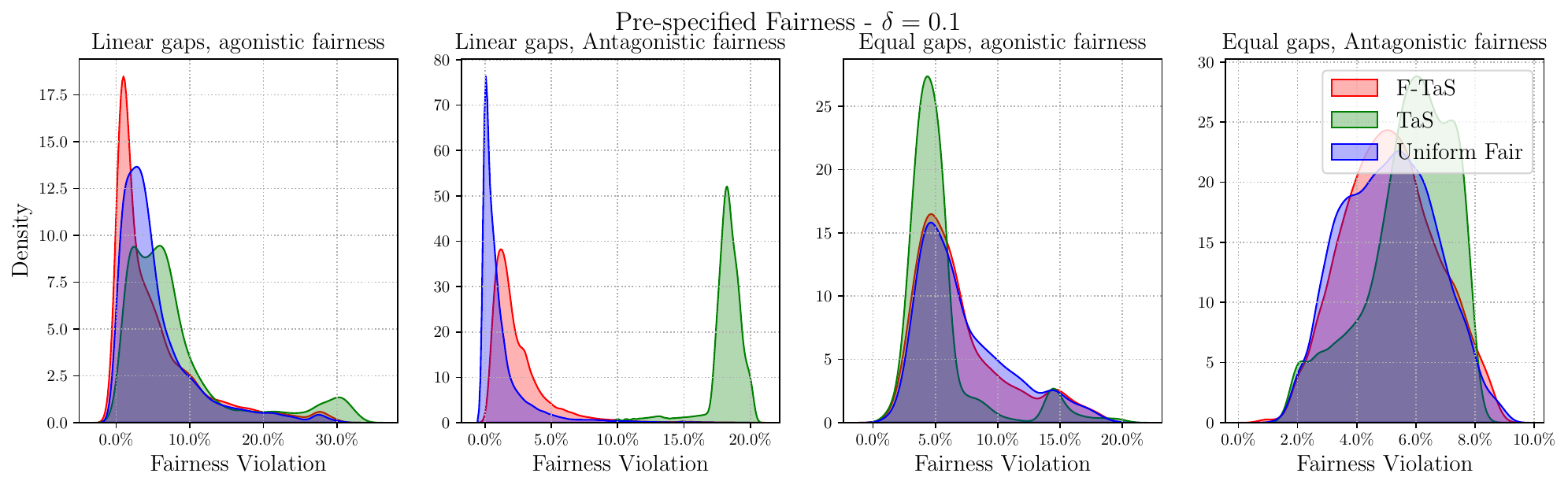}
    \includegraphics[width=.67\linewidth]{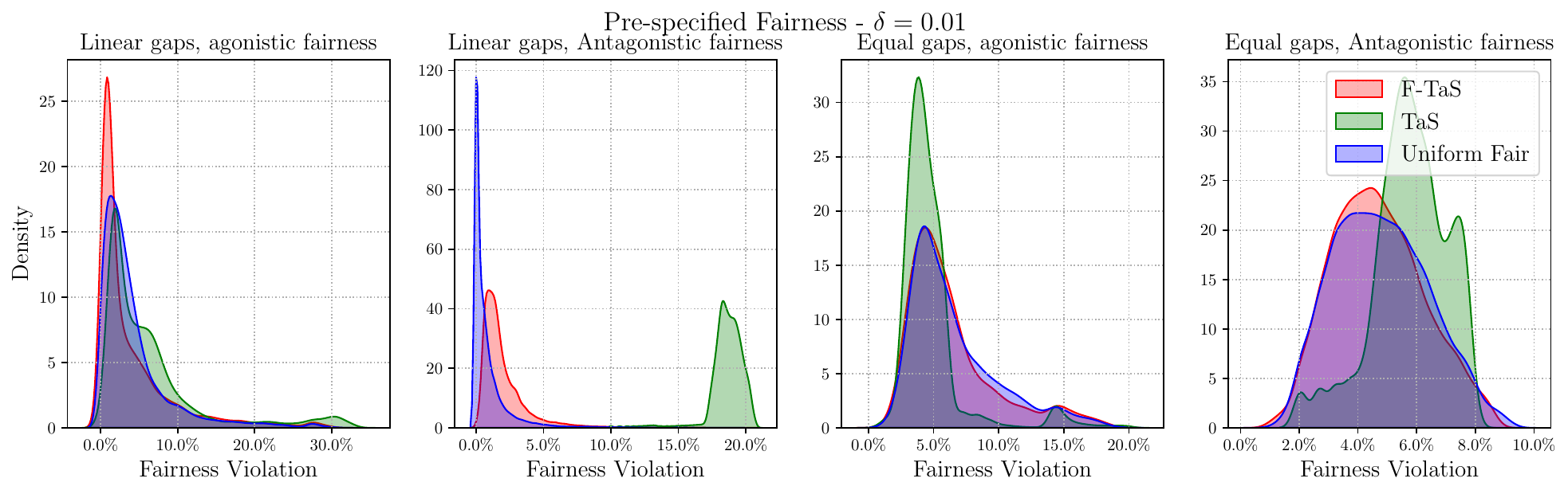}
    \includegraphics[width=.67\linewidth]{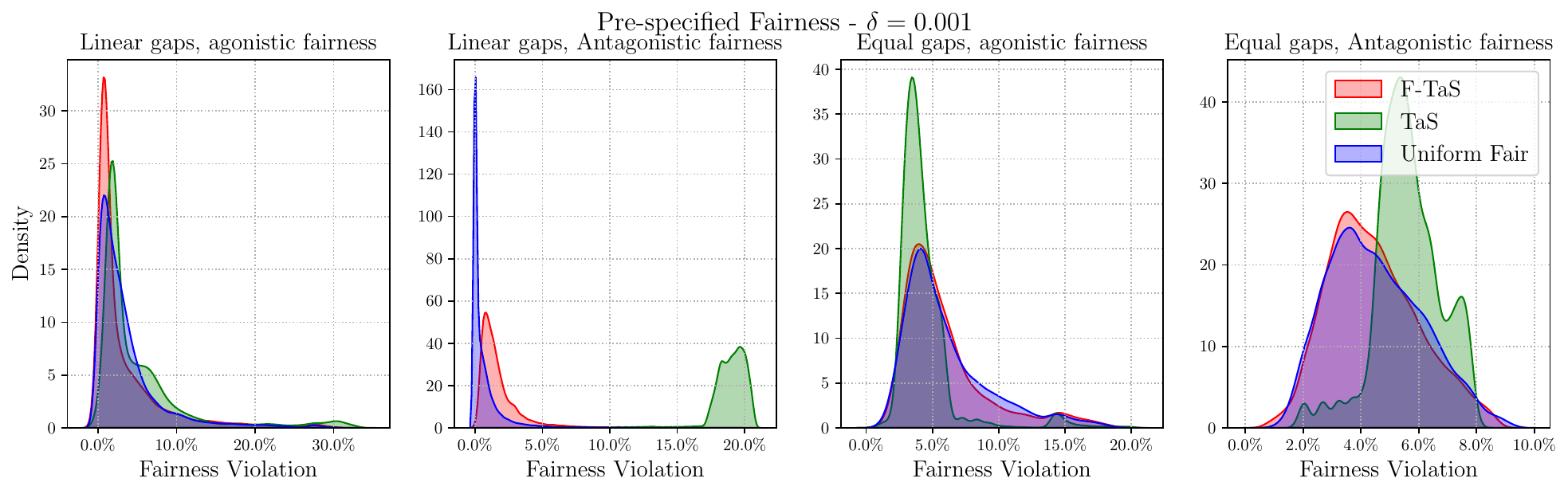}
    \caption{Violations for the synthetic experiments with pre-specified constraints for different values of $\delta$ in each row. Each subplot illustrates the distribution of maximum violation  $\rho(t) = (\max_a p_a(\theta)-N_a(t)/t)+$, across all rounds $t \leq \tau_\delta$ and  experimental runs.}
    \label{fig:synthetic_prespecified_violations}
\end{figure}

\newpage

\subsubsection{Synthetic Model with \texorpdfstring{$\theta$}{theta}-dependent Constraints}\label{subsec:synth_exp_thetadep}

\paragraph{Model} We considered a single bandit model, with $K=15$ arms and the reward linearly ranging in $[0,5]$. We used a value of $p_0=0.7$ in the fairness constraints.

\paragraph{Allocations} In \cref{fig:synthetic_thetadep_allocations} are depicted the optimal unconstrained allocation $w^\star$, the constrained one $w_p^\star$, and the fairness constraints $(p_a)_{a\in [K]}$. In the agonistic case we see how the fairness rates are closely related to the optimal exploration, while in the antagonistic one are inversely proportional.

\begin{figure}[ht!]
    \centering
    \includegraphics[width=.4\linewidth]{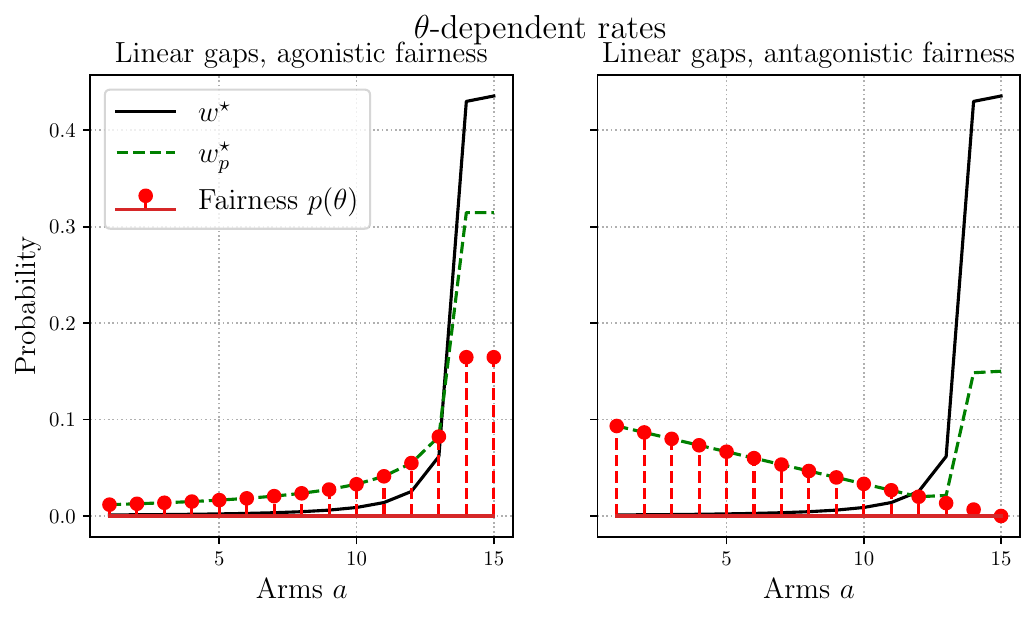}
    \caption{Synthetic experiments with $\theta$-dependent constraints. We show the the optimal unconstrained allocation $w^\star$, the constrained one $w_p^\star$, and the fairness constraints $(p_a(\theta))_{a\in [K]}$.}
    \label{fig:synthetic_thetadep_allocations}
\end{figure}

\paragraph{Sample complexity} In \cref{fig:synthetic_thetadep_samplecomplexity} we show the sample complexity results for each case, as well as the unconstrained sample complexity lower bound, and the constrained one. 

\begin{figure}[ht!]
    \centering
    \includegraphics[width=.4\linewidth]{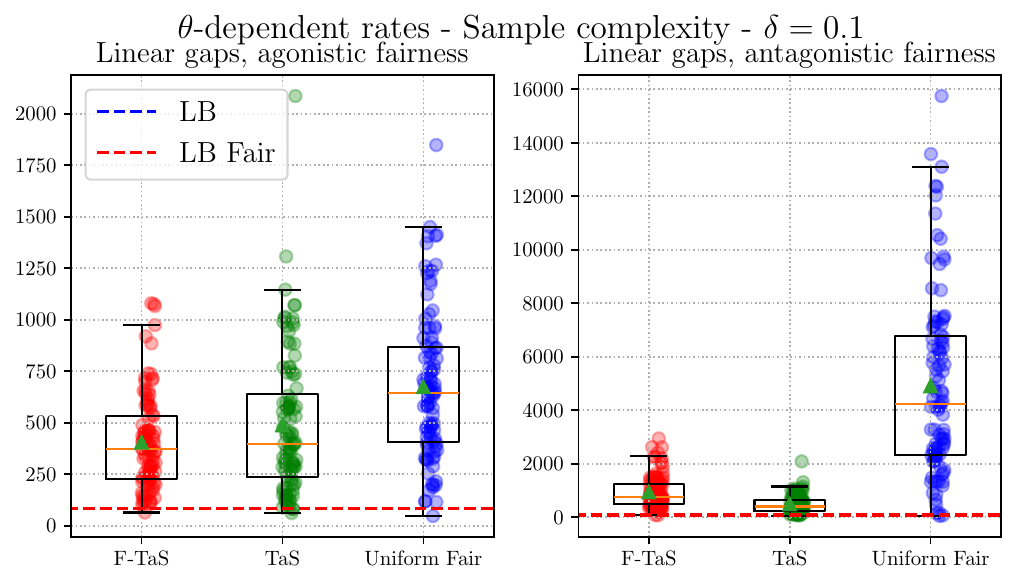}
    \includegraphics[width=.4\linewidth]{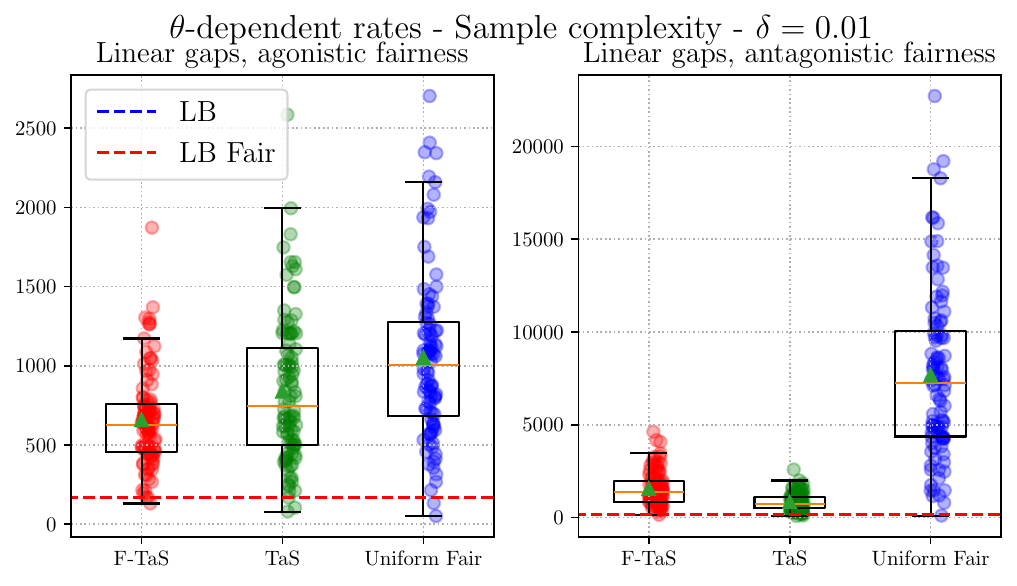}
    \includegraphics[width=.4\linewidth]{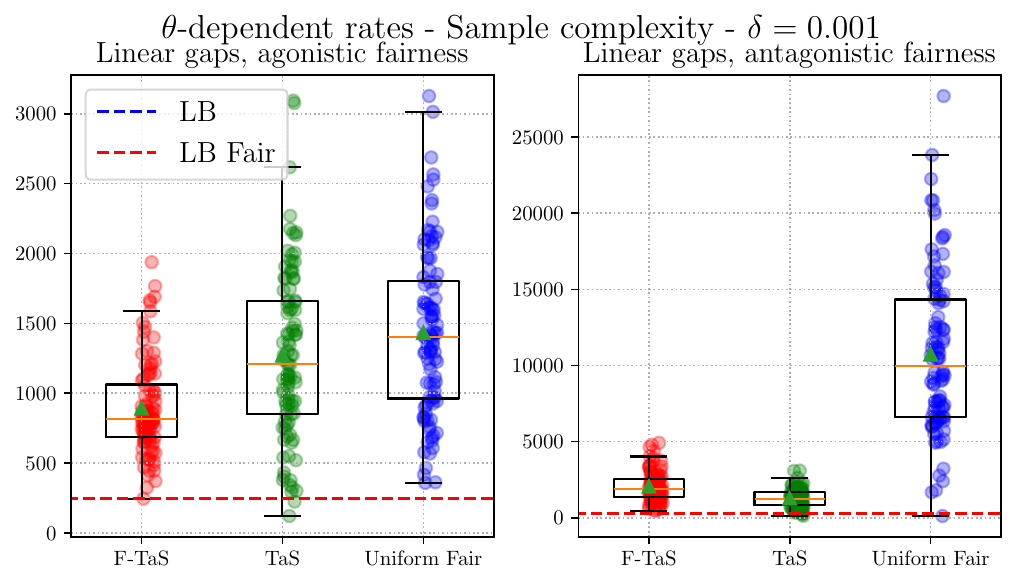}
    \caption{Synthetic model with $\theta$-dependent constraints. Sample complexity results for different values of $\delta$ are shown in each row.}
    \label{fig:synthetic_thetadep_samplecomplexity}
\end{figure}

\paragraph{Fairness violation} In \cref{fig:synthetic_thetadep_violations} we depict an aggregate distribution of the fairness violation $\rho(t)$ over all rounds. These plots offer a comprehensive understanding of the behavior of the algorithm.

\begin{figure}[h!]
    \centering
    \includegraphics[width=.4\linewidth]{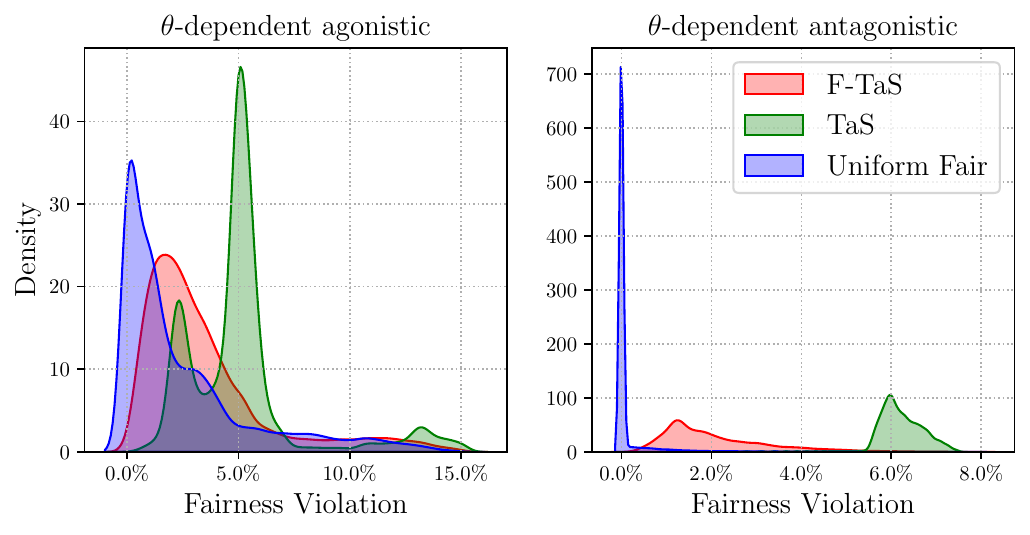}
    \includegraphics[width=.4\linewidth]{Figures/BAI/synthetic/thetadep_violation_0.01.pdf}
    \includegraphics[width=.4\linewidth]{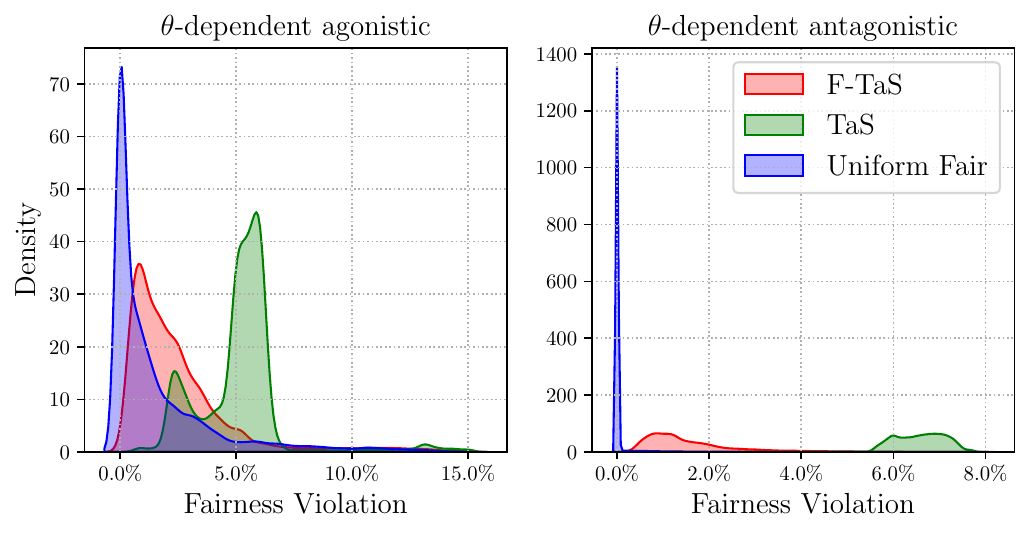}
    \caption{Violations for the synthetic experiments with $\theta$-dependent constraints for different values of $\delta$ in each row. Each subplot illustrates the distribution of maximum violation  $\rho(t) = (\max_a p_a(\theta)-N_a(t)/t)+$, across all rounds $t \leq \tau_\delta$ and  experimental runs.}
    \label{fig:synthetic_thetadep_violations}
\end{figure}

\subsection{Wireless Scheduling}
This appendix is organized as follows. In App. \ref{sec:additional_results_scheduling} we report additional experimental results on the wireless scheduling use-case and in App. \ref{sec:details_scheduling} we present more details on the experimental setup.

\subsubsection{Additional Numerical Results}
\label{sec:additional_results_scheduling}
We report extended results on (i) \textit{optimal allocations}, (ii) \textit{sample complexity}, and (iii) fairness violations in all the experimental setup described at the beginning of this appendix. 

\paragraph{Allocations} In \cref{fig:allocations_scheduling_complete} are depicted (i) the optimal unconstrained allocation $w^\star$, (ii) the optimal fair allocations $w_p^\star$, and the fairness constraints $(p_a)_{a\in [K]}$. In the agonistic case, we see how the fairness rates are closely related to the optimal exploration, while in the antagonistic one are inversely proportional.

\begin{figure}
    \centering
    \includegraphics[width = 0.8\linewidth]{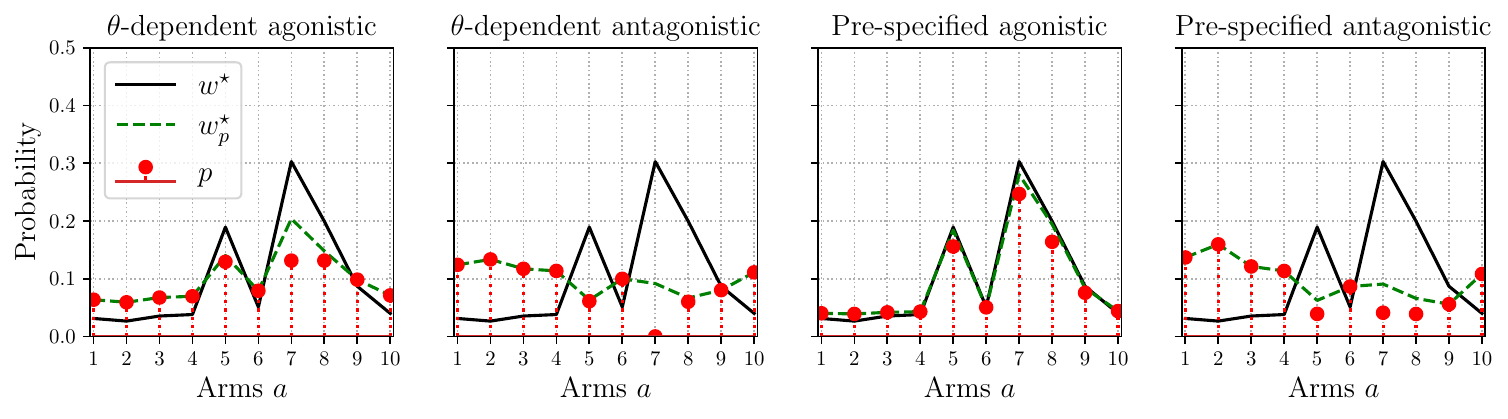}
    \caption{Wireless scheduling experiments: optimal unconstrained allocation $w^\star$, optimal fair allocations $w_p^\star$, and  fairness rates $(p_a)_{a\in [K]}$ for both pre-specified $\theta$-dependent constraints in the agonistic and antagonistic setting.}
    \label{fig:allocations_scheduling_complete}
\end{figure}

\paragraph{Sample complexity} In \cref{fig:sample_complexity_scheduling_complete} we show the boxplots for the sample complexity results. The points in each figure shows the realization of the sample complexity for each run. 

\paragraph{Fairness violation} In \cref{fig:fairness_violations_complete_scheduling} we depict an aggregate distribution of the fairness violation $\rho(t)$ over all rounds. These plots offer a comprehensive understanding of the behavior of the algorithm.

\subsubsection{Detailed Experimental Setting}
\label{sec:details_scheduling}
In this appendix, we present additional details on the scheduling experiments. 

\paragraph{Simulator} We test our F-BAI algorithm using mob-env, an
open-source simulation environment \cite{Schneider2022mobile}
based on the gymnasium interface. The simulator environment consist of a mobile network with a set of $K$ UEs and a BS. The BS is equipped with an antenna placed at a high of $h$ m. The antenna operates at a carrier frequency of $f$ Hz and the channel bandwidth is set at $W$ Hz. The size of the sector considered in the experiments is set at $L$ m$^2$. We report the configuration used in our experiments in Tab. \ref{tab:simulator_setup}.

\begin{table}[H]
\centering
\caption{Simulator parameters.}
\begin{tabular}{lcl}
\toprule
\textsc{Parameter} & \textsc{Symbol} & \textsc{Value} \\ \hline
Number of UEs & $K$ & $10$ \\
Bandwidth & $W$ & $9$ MHz \\
Carrier frequency & $f$ & $2500$ MHz \\
Antenna height & $h$ & $50$ m \\
Sector size & $L$ & $0.4 \text{ km}^2$ \\ \bottomrule
\label{tab:simulator_setup}
\end{tabular}
\end{table}
\normalsize
In our experiments we base the definition of our reward function on the sum-throughput, an important metric detailed in the following.

\begin{figure}
    \centering
    \includegraphics[width=.67\linewidth]{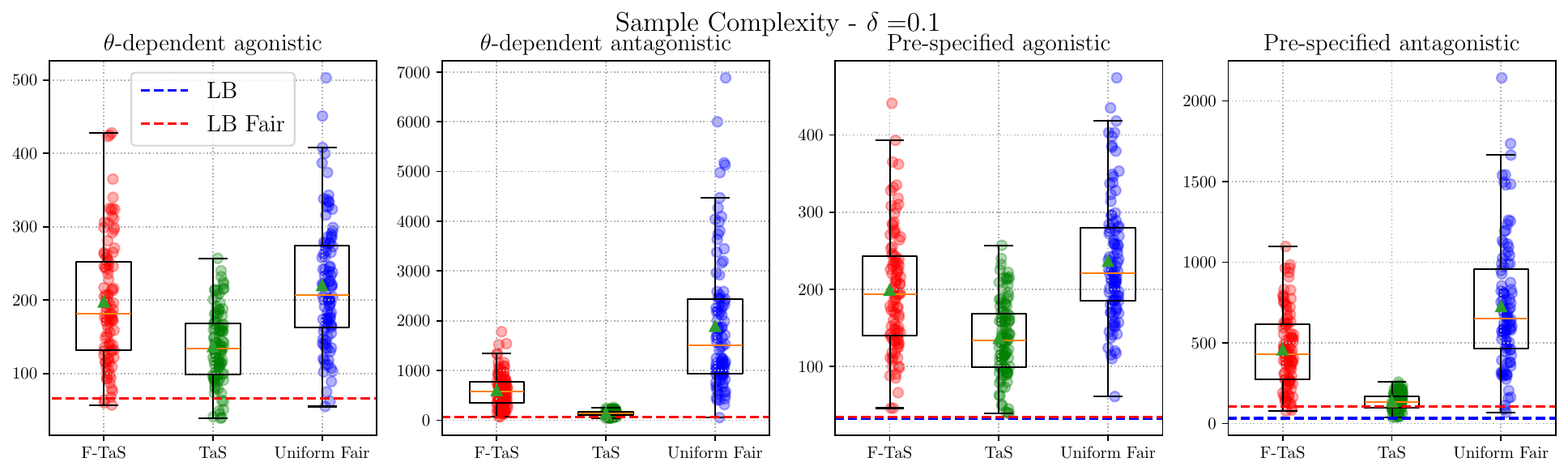}
    \includegraphics[width=.67\linewidth]{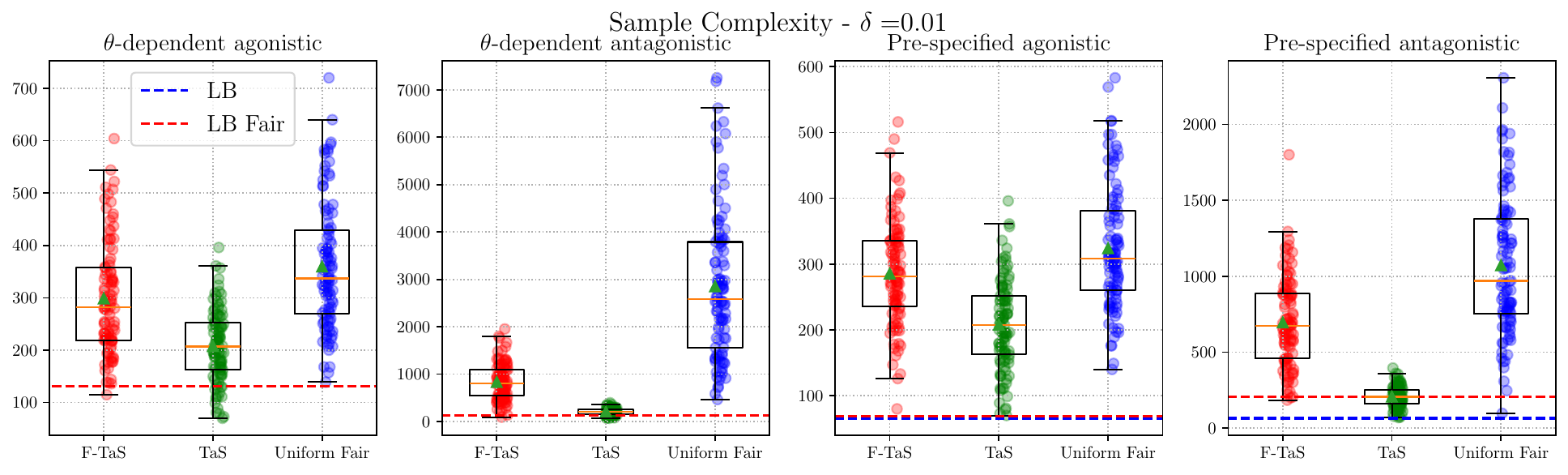}
    \includegraphics[width=.67\linewidth]{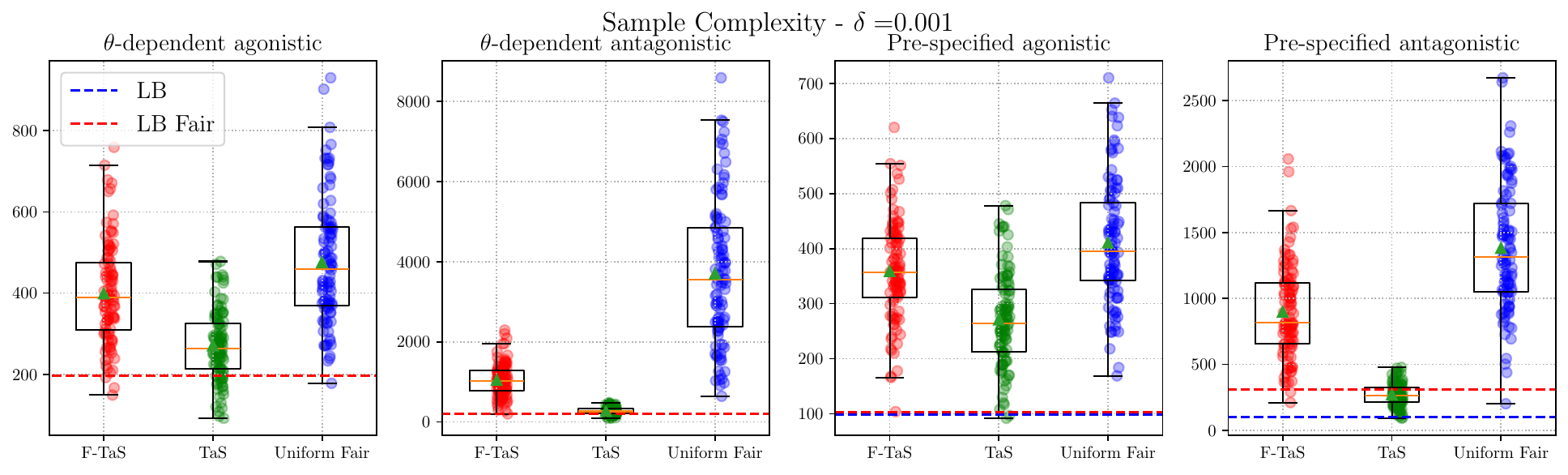}
    \caption{Wireless scheduling experiments: sample complexity results for different values of $\delta$.}
    \label{fig:sample_complexity_scheduling_complete}
\end{figure}
\begin{figure}[H]
    \centering
    \includegraphics[width=.67\linewidth]{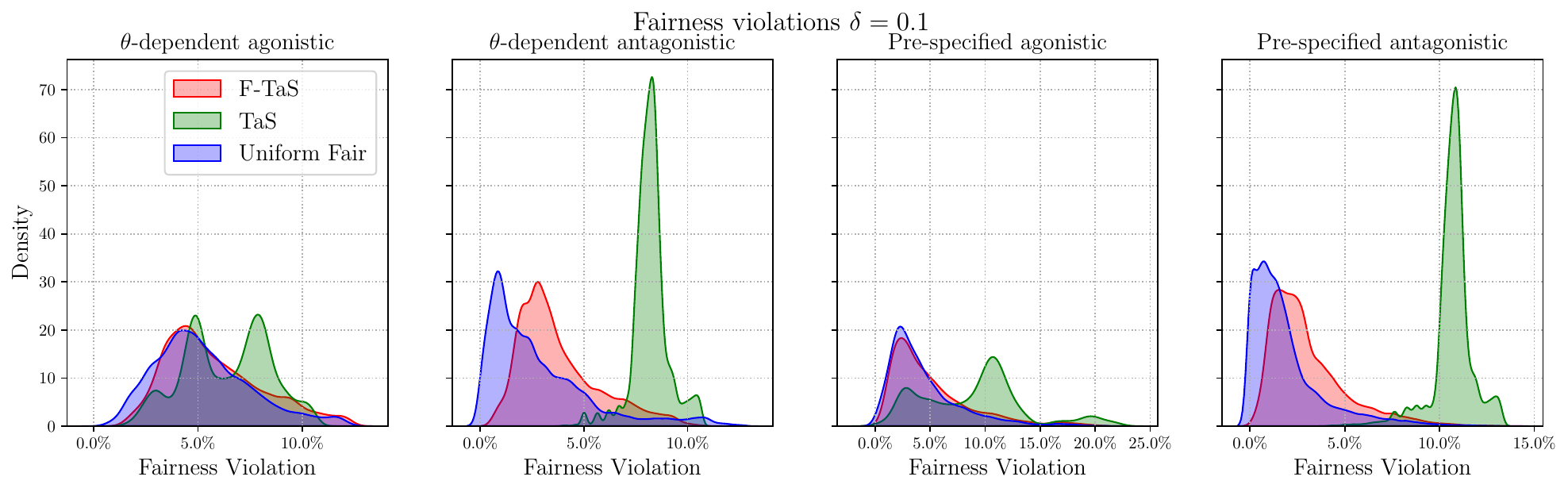}
    \includegraphics[width=.67\linewidth]{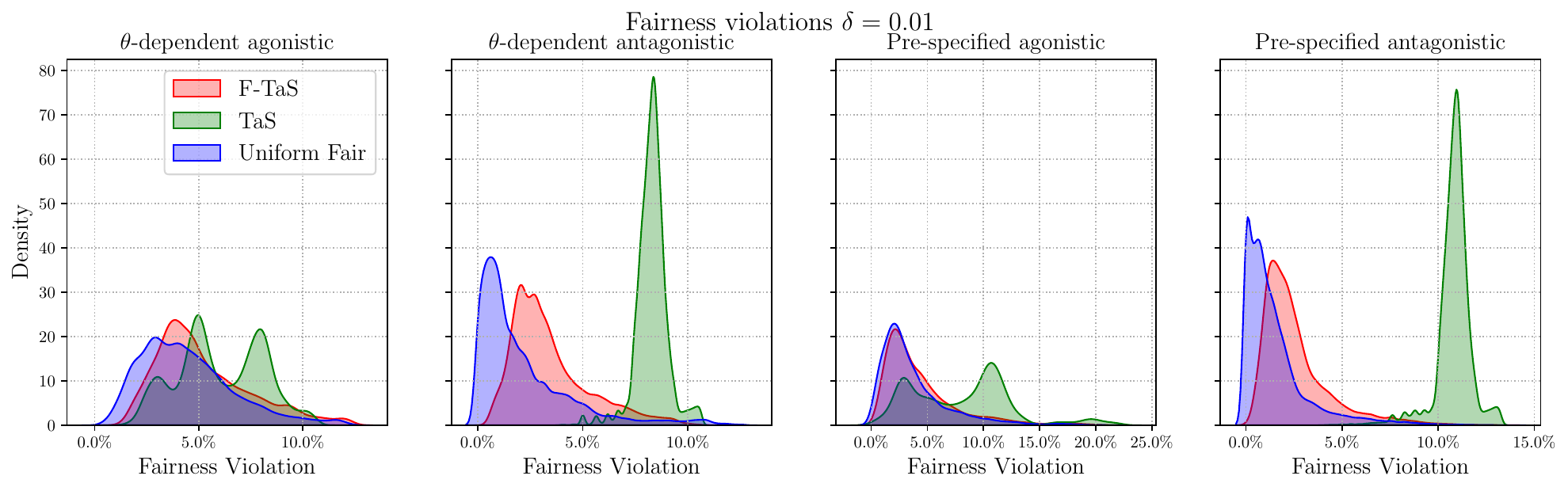}
    \includegraphics[width=.67\linewidth]{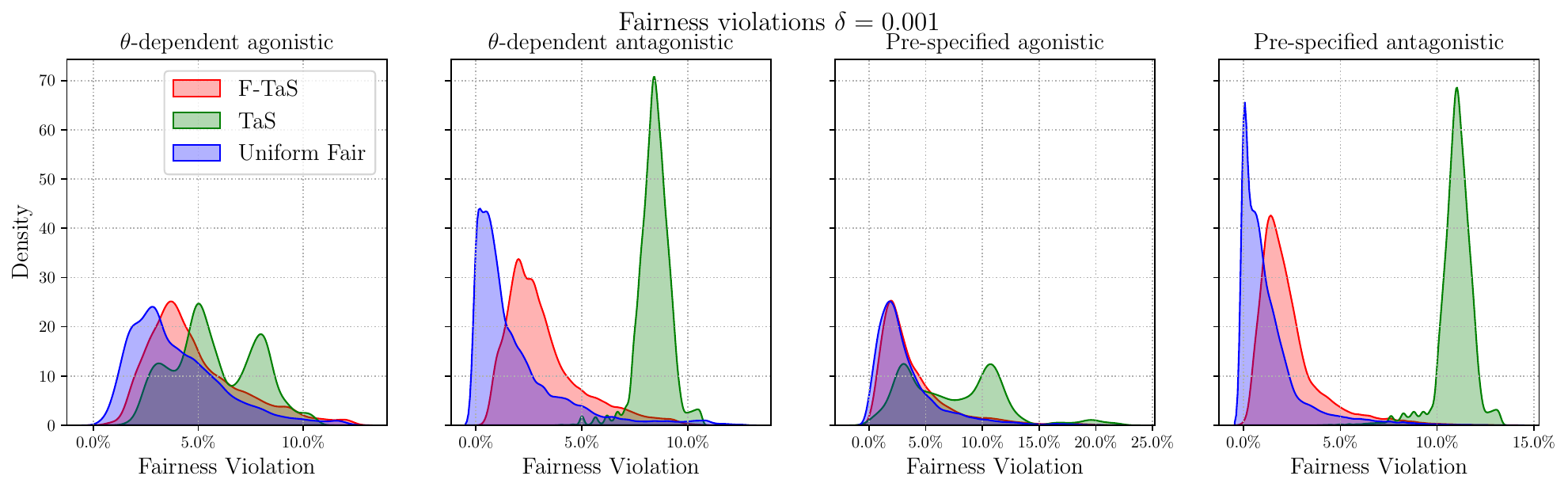}
    \caption{Violations for the scheduling experiments different values of $\delta$ in each row. Each subplot illustrates the distribution of maximum violation  $\rho(t) = (\max_a p_a(\theta)-N_a(t)/t)+$, across all rounds $t \leq \tau_\delta$ and  experimental runs.}
    \label{fig:fairness_violations_complete_scheduling}
\end{figure}

\paragraph{Throughput} The throughput $T_{u,t}$ of a UE $u\in\mathcal{U}$ at round $t\ge 1$, is formally defined in terms of the Signal-to-Noise Ratio (SNR), a metric that measures the quality of a signal in the presence of noise. Specifically, denote the SNR of a UE $u\in\mathcal{U}$ at time $t\ge 1$ is defined as $\text{SNR}_{u,t} = \frac{P_{\text{TX}}}{P_{\text{S}}}$, where $P_{\text{S}}$ and $P_{\text{N}}$ are the transmitted signal and noise power, respectively. Then the throughput (or rate) is expressed as 
$$
T_{u,t} = W\log_2(1+\text{SNR}).
$$

\paragraph{Additional details} Although different works in the literature assume that a single UE can be scheduled at each time \cite{nguyen2019scheduling}, 
we mention that more complex formulations allow for the BS to schedule a subset of the UEs at each round (see e.g., \cite{Li19combinatorial}). This extension yields an interesting combinatorial structure in the action selection. However, analyzing our fair bandit framework for such combinatorial bandit structures is out of the scope of this paper and is left as a future work. 
\clearpage 
\renewcommand{\thetheorem}{B.\arabic{theorem}}
\renewcommand{\thelemma}{B.\arabic{lemma}}

\newpage

\section{Sample Complexity Lower bound and the Price of Fairness}
\label{sec:lower_bound_price_of_fairness}

In this appendix, we prove the sample complexity lower bound (Theorem \ref{th:lb}), in App. \ref{app:proof_lb}, and the upper bound on the price of fairness $T^\star_p/T^\star$ (Lemma \ref{lem:ub_ratio}), in App. \ref{app:proof_ub_ratio}. We also discuss the price of fairness for various specific bandit instances (App. \ref{subsec:pricefairness} and App. \ref{app:price_fairness_additional}) and provide an additional result on the characterization of the optimal allocations in fair BAI (App. \ref{sec:characterize_lb}). 
\subsection{Lower Bound: proof of Theorem \ref{th:lb}}
\label{app:proof_lb}
The proof is a straightforward extension of the one in the plain bandit setting in \cite{garivier2016optimal}. The main difference is that, due to the $p$-fair $\delta$-PAC definition, the allocations must satisfy the conditions $w_a \ge p_a(\theta)$ for all $a\in[K]$. We sketch the main steps in the following.

\begin{proof}
Consider a $p$-fair $\delta$-PAC algorithm $(a_t,\tau, \hat{a}_{\tau_{\delta}})$, and define the set of confusing parameters $B(\theta) = \{ \lambda\in \Theta: a^\star_{\theta} \neq a^\star_{\lambda} \}$, where $a_\theta^\star =\arg\max_a \theta_a$. Let $\mathcal{E}_{\theta} = \{\hat{a}_{\tau_{\delta}} = a^\star_{\theta}\}$, and note that for all $\theta$,
$$
\mathbb{P}_{\theta}(\mathcal{E}_{\theta}) \ge 1-\delta,
$$
while for all $\lambda \in B(\theta)$ we have
$$
\mathbb{P}_{\lambda}(\mathcal{E}_{\theta}) \le \delta. 
$$
By Lemma 1 \cite{garivier2016optimal}, for any a.s. finite stopping time  $\tau_{\delta}$, we have that 
$$
\sum_{a\in[K]} \mathbb{E}_\theta[N_a(\tau_{\delta})] \frac{(\theta(a)-\lambda(a))^2}{2}\ge \mathrm{kl}(1-\delta,\delta).
$$
By letting $w_a = \frac{\mathbb{E}_\theta[N_a(\tau_\delta)]}{E_\theta[\tau_\delta]}$, we can rewrite the previous equation as 
$$
\mathbb{E}_\theta[\tau_{\delta}] \sum_{a\in[K]} w_a \frac{(\theta(a)-\lambda(a))^2}{2}\ge \mathrm{kl}(1-\delta,\delta). 
$$
By optimizing over the set of confusing parameters we get 
$$
\mathbb{E}_\theta[\tau_{\delta}] \min_{a\neq a^\star_{\theta}}\frac{\Delta(a)^2}{2(w_{a^\star_{\theta}}^{-1} + w_a^{-1})}\ge \mathrm{kl}(1-\delta,\delta). 
$$

\paragraph{Pre-specified fairness} As the lower bound holds for any $p$-fair $\delta$-PAC algorithm, we must have that \eqref{eq:constraint_bai_constant_p} holds, and hence $\mathbb{E}[N_a(\tau_\delta)]/\mathbb{E}[\tau_\delta] \ge p_a$. The result is finally obtained by optimizing the allocations over the set of clipped allocations $\Sigma_p = \{w \ge p: \sum_{a\in\mathcal{A}} w_a = 1\}$.

\paragraph{$\theta$-dependent fairness} As the lower bound holds for any asymptotically $p(\theta)$-fair $\delta$-PAC algorithm we must have that $\liminf_{\delta \to 0} \mathbb{E}[N_a(\tau_\delta)]/\mathbb{E}[\tau_\delta] \ge p_a(\theta)$. Hence, the result is finally obtained by optimizing the allocations over the set $\Sigma_p = \{w\ge p(\theta): \sum_{a} w_a = 1\}$ and letting $\delta \to 0$. 

\end{proof}

\subsection{Sample-path Fairness}
\label{app:sample_path_fairness}
As mentioned in \cref{sec:fair_bai}, the fairness constraints introduced in \eqref{eq:constraint_bai_constant_p}-\eqref{eq:constraint_bai_ptheta} may seem unusual, as the expectation is taken separately on the numerator and denominator, rather than considering $\mathbb{E}_\theta\left[\frac{N_a(\tau_\delta)}{\tau_\delta}\right]$. This latter expression, which we refer to as "\textit{sample-path fairness}", is arguably a more intuitive definition because it evaluates fairness on each  sample path. 

\paragraph{Relations between fairness definitions} We now show that it is possible to relate the two expressions in a simple way. First, a simple application of Jensen's inequality yields
\[
\frac{\mathbb{E}_\theta[N_a(\tau_\delta)]}{\mathbb{E}_\theta[\tau_\delta]} \leq \mathbb{E}_\theta[N_a(\tau_\delta)]\mathbb{E}_\theta\left[\frac{1}{\tau_\delta}\right].
\]
Then, combining  this last inequality with the fact that ${\rm Cov}_\theta\left(N_a(\tau_\delta), 1/\tau_\delta\right)= \mathbb{E}_\theta\left[\frac{N_a(\tau_\delta)}{\tau_\delta}\right]- \mathbb{E}_\theta[N_a(\tau_\delta)]\mathbb{E}_\theta\left[\frac{1}{\tau_\delta}\right]$, one concludes that 
\[
\frac{\mathbb{E}_\theta[N_a(\tau_\delta)]}{\mathbb{E}_\theta[\tau_\delta]} \leq \mathbb{E}_\theta\left[N_a(\tau_\delta)/\tau_\delta\right] - {\rm Cov}_\theta\left(N_a(\tau_\delta), 1/\tau_\delta\right).
\]

Henceforth, an algorithm guaranteeing a small covariance $ {\rm Cov}_\theta\left(N_a(\tau_\delta), 1/\tau_\delta\right)$ ensures that the sample-path fairness is approximately guaranteed. In the following paragraph we list some properties.

\paragraph{Further properties}
In general, for any algorithm it is possible to say that the average covariance is negative. In fact, we find that
\[
\sum_a{\rm Cov}_\theta\left(N_a(\tau_\delta), 1/\tau_\delta\right)= 1- \mathbb{E}_\theta[\tau_\delta]\mathbb{E}_\theta\left[1/\tau_\delta\right] \leq 0,
\]
since $\mathbb{E}_\theta[\tau_\delta]\mathbb{E}_\theta\left[1/\tau_\delta\right]\geq 1$.  Furthermore, we also have the following result stating that any algorithm that is $p$-fair (resp. $p(\theta)$-fair) for all $t\ge 1$ is also $p$-sample-path fair (resp. $p(\theta)$-sample-path fair). 

\begin{lemma}\label{lemma:algo_fairness_simple}
Consider an algorithm satisfying $\mathbb{E}_\theta[N_a(t)]\geq p_a(\theta) t$, for all $a\in [K], t\geq 1$. Then,  for all $a$: 
    \begin{itemize}
    \item[(I)]$\frac{\mathbb{E}_\theta[N_a(\tau_\delta)]}{\mathbb{E}_\theta[\tau_\delta]} \ge p_a(\theta) $,
    \item[(II)]$\mathbb{E}_\theta\left[\frac{N_a(\tau_\delta)}{\tau_\delta}\right] \ge p_a(\theta)$.
    \end{itemize}

\end{lemma}

\begin{proof}
W.l.o.g. fix an arm $a \in [K]$. Then, (II) follows  from the fact that $$\mathbb{E}_\theta[N_a(\tau_\delta)] = \sum_{t}\mathbb{E}_\theta[N_a(t)|\tau_\delta=t]\mathbb{P}_\theta(\tau_\delta=t) \geq p_a(\theta)\sum_t  t \mathbb{P}_\theta(\tau_\delta=t) =p_a(\theta)\mathbb{E}_\theta[\tau_\delta]. $$ 

Similarly, (III) follows as 

$$\mathbb{E}_\theta\left[\frac{N_a(\tau_\delta)}{\tau_\delta}\right] = \sum_{t}\mathbb{E}_\theta\left[\frac{N_a(t)}{t}\Bigg|\tau_\delta=t\right] \mathbb{P}_\theta(\tau_\delta=t) \geq p_a(\theta)\sum_t \mathbb{P}_\theta(\tau_\delta=t) =p_a(\theta). $$ 

\end{proof}

A similar result also holds asymptotically as $\delta \to 0$ as long as $N_a(t)/t$ concentrates around $w_a(\geq p_a)$ sufficiently fast. We have the following (informal) result.

\begin{lemma}\label{lemma:asymptotic_algo_fairness_new}
    Consider an event ${\cal E}$ such that for some $T\geq T_\delta$ (with $\lim_{\delta \to 0}T_\delta \to \infty$) we have $(\tau_\delta \leq T) \supset {\cal E}$, and that for all $\varepsilon>0$ we have $\|N(t)/t - w_p^\star \|_\infty \leq \varepsilon$ under ${\cal E}$, where $N(t)=(N_a(t))_{a\in [K]}$ and $w_p^\star = (w_{p,a}^\star)_{a\in [K]}$. Assume furthermore that $\liminf_{\delta \to 0} \mathbb{P}_\theta({\cal E})\geq1$.  Then we have that: (I) $\liminf_{\delta \to 0}{\rm Cov}_\theta\left(N_a(\tau_\delta), 1/\tau_\delta\right)\leq0$; (II) $p_a(\theta) \leq \liminf_{\delta \to 0}\frac{\mathbb{E}_\theta[N_a(\tau_\delta)]}{\mathbb{E}_\theta[\tau_\delta]} $ and (III) $p_a(\theta)\leq \liminf_{\delta \to 0}\mathbb{E}_\theta\left[\frac{N_a(\tau_\delta)}{\tau_\delta}\right]$.
\end{lemma}
\begin{proof}
    Observe that 
    \begin{align*}
        \mathbb{E}_\theta[N_a(\tau_\delta)] &= \mathbb{E}_\theta[N_a(\tau_\delta)|{\cal E}]\mathbb{P}_\theta({\cal E}) + \mathbb{E}_\theta[N_a(\tau_\delta)|\overline {\cal E}]\mathbb{P}_\theta(\overline{\cal E}),\\
        &\geq (w_{p,a}^\star -  \varepsilon)\mathbb{E}_\theta[\tau_\delta|{\cal E}]\mathbb{P}_\theta({\cal E}).
    \end{align*}
    Since $\lim_{\delta \to 0}\mathbb{P}_\theta({\cal E})=1$ then as $\delta \to 0$ we have $\mathbb{E}_\theta[\tau_\delta] \to \liminf_{\delta \to 0} \mathbb{E}_\theta[\tau_\delta|{\cal E}]$ as $\delta$ goes to $0$. Hence
    \[
    \liminf_{\delta \to 0}\frac{\mathbb{E}_\theta[N_a(\tau_\delta)]}{\mathbb{E}_\theta[\tau_\delta]}\geq (w_{p,a}^\star -  \varepsilon) \geq p_a(\theta) - \varepsilon.
    \]
    Letting $\varepsilon \to 0$ yields (II). Similarly, one can also find $\liminf_{\delta \to 0}\mathbb{E}_\theta\left[\frac{N_a(\tau_\delta)}{\tau_\delta}\right]\geq p_a(\theta)$ from which follows (III) and (I).
\end{proof}

\subsection{Price of Fairness: proof of Lemma \ref{lem:ub_ratio}}
\label{app:proof_ub_ratio}
\begin{proof}
Consider the following feasible allocation $\tilde{w}\in \Sigma_p$:
$\tilde{w}_a = p_a + \frac{1-\psum}{K}$, where we denote $p_a = p_a(\theta)$ and $\psum=\psum(\theta)$ for the sake of simplicity.

We can write
$$
T^\star_{p} = {\color{black} 2}\min_{w\in\Sigma_p} \max_{a \neq a^\star} \frac{w_a^{-1} + w_{a^\star}^{-1}}{\Delta_a^2} \le {\color{black} 2}\min_{w\in \Sigma_p} \max_{a \neq a^\star} \frac{w_a^{-1} + w_{a^\star}^{-1}}{\Delta^2_{\min}} \le {\color{black} 2}\max_{a \neq a^\star} \frac{\tilde{w}_a^{-1} + \tilde{w}_{a^\star}^{-1}}{\Delta^2_{\min}} = \frac{{\color{black} 2}}{\Delta_{\min}^2}(\tilde{w}_{a^\star}^{-1} + \max_{a\neq a^\star}  \tilde{w}_a^{-1}) . 
$$
Let $p_{\min}= \min_{a\in[K]: p_a > 0} p_a$, then

$$\tilde{w}_{a^\star}^{-1} + \max_{a\neq a^\star}  \tilde{w}_a^{-1} 
    \leq \frac{K}{Kp_{a^\star} + (1-\psum)} + \frac{K}{Kp_{\min} + (1-\psum)} \le \frac{2K}{Kp_{\min} + (1-\psum)}
$$

On the other hand, by App. A.4 in \cite{garivier2016optimal}, we have that 
$$T^\star \ge{\color{black} 2} \sum_{a\in[K]}\frac{1}{\Delta_a^2} \ge \frac{{\color{black} 2}K}{\Delta_{\max}^2}.$$ 
Hence, we find
\[
\frac{T^\star_{p}}{T^\star} \le \frac{\Delta_{\max}^2}{\Delta_{\min}^2} \frac{2}{Kp_{\min} + 1-\psum}.
\]

Now, we consider two separate cases. If $K p_{\rm min} \ge 1-\psum$ we get 
$$
\frac{2}{Kp_{\min} + 1-\psum} \le \frac{1}{1-\psum} ,
$$

Otherwise, if $1 -\psum > K p_{\rm min}$ we have 
$$
\frac{2}{Kp_{\min} + 1-\psum} < \frac{1}{Kp_{\min}} ,
$$
and hence we can conclude 
\begin{equation}
   \label{eq:ub_ratio}
\frac{T^\star_{p}}{T^\star} \le   \frac{\Delta_{\max}^2}{\Delta_{\min}^2} \min\left(\frac{1}{1-\psum}, \frac{1}{Kp_{\min}}\right). 
\end{equation}

\end{proof}

\subsection{Price of Fairness: examples and specific instances}\label{subsec:pricefairness}
For specific bandit instances, we can obtain tighter bounds on $T^\star_p/T^\star$. We consider the following cases.

\paragraph{1. The 2-armed bandits case} In the $2$-armed bandit case it is known that the optimal allocation in the plain bandit setting without fairness constraints satisfies $w^\star = (1/2,1/2)$ \cite{garivier2016optimal}. In the fair BAI setting, we have that $w_p^\star \neq w^\star$, when either $p_1$ or $p_2$ are greater than $1/2$ (naturally the case where both $p_1$ and $p_2$ are greater than $1/2$ is unfeasible). Without loss of generality let $p_1 > 1/2$. Then the optimal allocations satisfy $w_p^\star = (p_1, 1-p_1)$, and we have that
\begin{equation}
\label{eq:ub_K_2}
\frac{T^\star_{p}}{T^\star} = \frac{1}{4p_1(1-p_1)}.	
\end{equation}
This ratio quantifies the price of fairness in MAB and naturally, it is minimized for $p_1 = 1/2$. Note that when particularizing our bound in \eqref{eq:ub_ratio} for $K = 2$, with $p_1 > 1/2$, we have
$$
\frac{T^\star_{p}}{T^\star} \le \frac{2}{(1-p_1)}. 
$$

\paragraph{2. The case of unitary $p$}
Note that if $p$ is such that $\sum_{a\in[K]} p_a = 1$, the optimal solution satisfies $w^\star_{p,a} = p_a$, for all $a\in[K]$, and hence 
$$
T^\star_{p} ={\color{black} 2} \max_{a\neq a^\star} \frac{p_a^{-1} + p_{a^\star}^{-1}}{\Delta_{a}^2}.
$$

In this case, we have $ T^\star_{p} \le \frac{2}{p_{\rm min}\Delta_{\max}^2},$
and hence 
\begin{equation}
\frac{T^\star_{p}}{T^\star} \le  \frac{1}{Kp_{\min}} \left(\frac{\Delta_{\max}}{\Delta_{\min}}\right)^2,
\end{equation}
where $p_{\rm min} = \min_{a: p_a>0} p_a$.

\paragraph{3. The case $p_{a^\star}=0$ and $p_a\geq w_{a}^\star$ for $a\neq a^\star$} This scenario is important since it displays both dependencies on $\psum$ and $p_{\min} = \min_{a: p_a>0} p_a$.

Assuming $\psum<1$, if $p_{a^\star}=0$, and $p_a\geq w_{a}^\star$ for $a\neq a^\star$, by \cref{lem:opt_all_eq} we immediately have that $w_{p,a^\star}^\star = 1-\psum$ and $w_{p,a}^\star=p_a$ for $a\neq a^\star$.

We immediately conclude that $T_p^\star \leq \frac{1}{\Delta_{\rm min}^2}(\frac{1}{1-\psum} + \frac{1}{p_{\rm min}})$, and thus 
\begin{equation}
    \frac{T_p^\star}{T^\star} \leq \frac{1}{K}\frac{\Delta_{\rm max}^2}{\Delta_{\rm min}^2} \left(\frac{1}{1-\psum} + \frac{1}{p_{\rm min}}\right).
\end{equation}

We now discuss the example in Case 1 of \textsection\ref{sec:syntetic_exp} in more details. We consider the case where the values of $\theta=(\theta_a)_{a\in[K]}$ linearly range in $[0,1].$ In this case the minimum gap scales as $1/(K-1)$ and $\sum_a \Delta_a= K/2$, henceforth $p_{\rm min} = K p_0 \frac{1/(K-1)}{K/2} = \frac{2p_0}{K-1}$, while $\psum = K p_0$. Therefore $(1-\psum)^{-1}$ is the leading term for $p_0\to 1/K$ (see also \cref{fig:price_fairness_case1}), while $p_{\rm min}^{-1}$ is decreasing in $p_0$.

\paragraph{4. The case $p_{a^\star}=0$ and $p_a\geq w_{a}^\star$ for $a\neq a^\star$ with equal gaps} 
The previous scenario can be extended to the case of equal gaps, i.e., $\Delta_a =\Delta$ for all $a\neq a^\star$. In this particular case, we are able to exactly compute $T_p^\star/T^\star$. In fact, the solution to the lower bound problem $T^\star$ satisfies
\begin{equation}
    w_{a}^\star = \begin{cases}
         \dfrac{\sqrt{K-1}}{\sqrt{K-1} + K-1} & a=a^\star,\\
        \dfrac{1}{\sqrt{K-1} + K-1} &  a \neq a^\star.
    \end{cases}
\end{equation}
Therefore 
\begin{align*}
    T^\star & = \frac{1}{\Delta^2} \max_{a\neq a^\star} \frac{1}{w_a^\star} + \frac{1}{w_{a^\star}^\star}= \frac{\sqrt{K-1}+K-1}{\Delta^2} \cdot \frac{1+\sqrt{K-1}}{\sqrt{K-1}},\\
    &= \frac{1+\sqrt{K-1}}{\Delta^2} (1+\sqrt{K-1}),\\
    &= \frac{(1+\sqrt{K-1})^2}{\Delta^2}.
\end{align*}

Hence, using the fact that $T_p^\star = \frac{1}{\Delta^2}(\frac{1}{1-\psum} + \frac{1}{p_{\rm min}})$, the ratio $T_p^\star/T^\star$ becomes
\begin{equation}
    \frac{T_p^\star}{T^\star} = \frac{1-\psum +p_{\rm min}}{p_{\rm min}(1-\psum)(1+\sqrt{K-1})^2}.
\end{equation}
Lastly, we also recover the form of the previous upper bound
 $\frac{T_p^\star}{T^\star} \leq  \frac{1}{K-1} \left(\frac{1}{p_{\rm min}} + \frac{1}{1-\psum}\right).$

 \paragraph{5. The case $p_{a^\star} \geq w_{a^\star}^\star$ and $p_a\leq (1-p_{a^\star})/(K-1)$ for $a\neq a^\star$ with equal gaps} 
For the equal gap case, we can also study the scenario where $p_a$, for $a\neq a^\star$, is significantly small, and $p_{a^\star}  \geq w_{a^\star}^\star$. 
 
Similarly as before, if $p_a\in [0, \frac{1-p_{a^\star}}{K-1}]$, then we have $w_{p,a^\star}^\star = p_{a^\star}$ and $w_{p,a}^\star = c$ with $(K-1)c + p_{a^\star} = 1 \Rightarrow w_{p,a}^\star= (1-p_{a^\star})/(K-1)$ for all $a\neq a^\star$. Therefore $T_p^\star = \frac{1}{\Delta^2}(\frac{K-1}{1-p_{a^\star}} + \frac{1}{p_{a^\star}})$ and the ratio $T_p^\star/T^\star$ becomes
 \begin{equation}
      \frac{T_p^\star}{T^\star} = \frac{p_{a^\star}(K-2) + 1 }{p_{a^\star}(1-p_{a^\star})(1+\sqrt{K-1})^2}.
 \end{equation}
Note that in this case $\psum$ and $p_{\rm min}$ are equivalent. Furthermore, in this case the price of fairness is due to the additional sampling of the optimal arm at the cost of not sampling enough all the other arms.

\subsection{Additional results on Price of Fairness}
\label{app:price_fairness_additional}
In this appendix, we provide additional numerical results with the goal of quantifying numerically the price of fairness and the tightness of our bound on $T^\star_p/T^\star$ presented in \ref{lem:ub_ratio}. 

\subsubsection{The two regimes in the price of fairness}
\label{sec:the_two_regimes}
In this section, we illustrate the existence of two regimes of our upper bound on $T^\star_{p}/T^\star$. Indeed, each of the terms appearing in the bound ($1/Kp_{\min}$ and $1/(1-\psum)$) is tighter 
in different scenarios. 

To illustrate this phenomenon, we consider a set of instances where the sub-optimality gaps are fixed, i.e., $\Delta_a = \Delta$, for all $a\neq a^\star$. The fairness parameters are selected as $p = \lambda w^\star$, where $w^\star$ are the optimal allocations in the plain bandit setting (without fairness constraints), i.e., $w^\star$ are the allocations optimizing $T^\star$, and $\lambda\in [0,1]$ is a scaling parameter. Fig. \ref{fig:price_of_fairness} shows a plot of the value of $T^\star_p/T^\star$, and the terms appearing in our bound, i.e., $1/Kp_{\min}$ and $1/(1-\psum)$ when varying the parameter $\lambda$. 
\begin{figure}[ht!]
    \centering
    \includegraphics[width=0.45\linewidth]{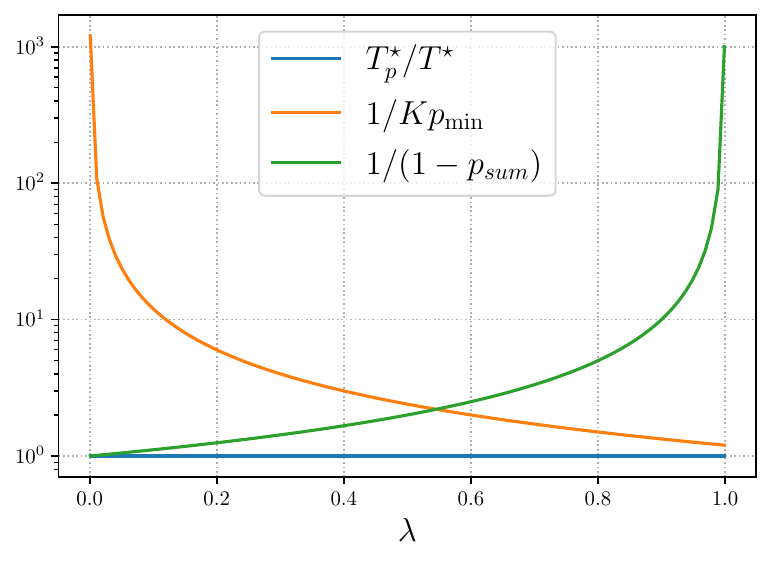}
    \caption{Illustration of the  two regimes in the upper bound on $T^\star_p/T^\star$.}
    \label{fig:price_of_fairness}
\end{figure}

Note that $T^\star_p/T^\star = 1$ as for all $\lambda$, we have $\lambda w_a^\star \le w_a^\star$, which implies  $w_a^\star = w_{p,a}^\star$. 
As it can be observed from the figure, generally we have two regimes: for low values of $\lambda$ the term $(1-\psum)^{-1}$ is tighter (as $\pmin$ will also have low value); as $\lambda$ increases, the term $(1-\psum)^{-1}$ (as $\psum$  approaches $1$) and the term $1/K\pmin$ provides a tighter bound. Our bound $T^\star_{p}/T^\star \le O(\min\{(1-\psum)^{-1},(K\pmin)^{-1}\})$ captures both these scenarios and provides a tighter bound. 

\subsubsection{Price of Fairness: scaling of \texorpdfstring{$(1-\psum)^{-1}$}{1psum}-dependent Constraints}
In this subsection we investigate the scaling of $T_p^\star/T^\star$, and how it is related to $1/(1-\psum)$. We study $4$ different cases, each for a different number of arms $K\in\{10,20,30\}$.
\begin{enumerate}
    \item In the first setting, we study a model where the rewards linearly ranging in $[0,1]$. The fairness constraints are constant, set to some value $p \in (0,1/K)$ for all arms $a\in [K]$. Results are depicted in \cref{fig:scaling_price_of_fairness_1}.
    \item In the second setting, we study a model where all the arms have equal gap $\Delta_a=0.1$. The fairness constraints are constant, set to some value $p \in (0,1/K)$ for all arms $a\in [K]$. Results are depicted in \cref{fig:scaling_price_of_fairness_1}.
    \item In the third setting we study a model with rewards linearly ranging in $[0,1]$. The fairness constraints are $\theta$-dependent. In particular, we compute it as follows
\begin{equation}\label{eq:pof_eq_fairness_theta} p_a(\theta) = pK \frac{f_a(\theta)}{\sum_b f_b(\theta)}, \hbox{ where }f_a(\theta) = 1 + \frac{1}{\Delta_a},
\end{equation}
with $f_{a^\star} = 1$. This makes the rates inversely proportional to the gap. However, since $f_{a^\star} = 1$, the optimal arm will have a very low fairness constraint, resulting in the algorithm sampling "good" arms at a very large rate. Results are depicted in \cref{fig:scaling_price_of_fairness_2}.
\item The last setting uses the same fairness constraint as the previous case, but the arms now have the same gap $\Delta_a=0.1,\; \forall a\in[K]$. Results are depicted in \cref{fig:scaling_price_of_fairness_2}.
\end{enumerate}
In all cases, we evaluate $T_p^\star/T^\star$ over different values of $p\in(0,1/K)$. From the results, we see a clear scaling in $(1-\psum)^{-1}$ when the rewards are linearly ranging, while the equal gap case is less affected by this parameter.

\begin{figure}[h!]
    \centering
    \includegraphics[width=.8\linewidth]{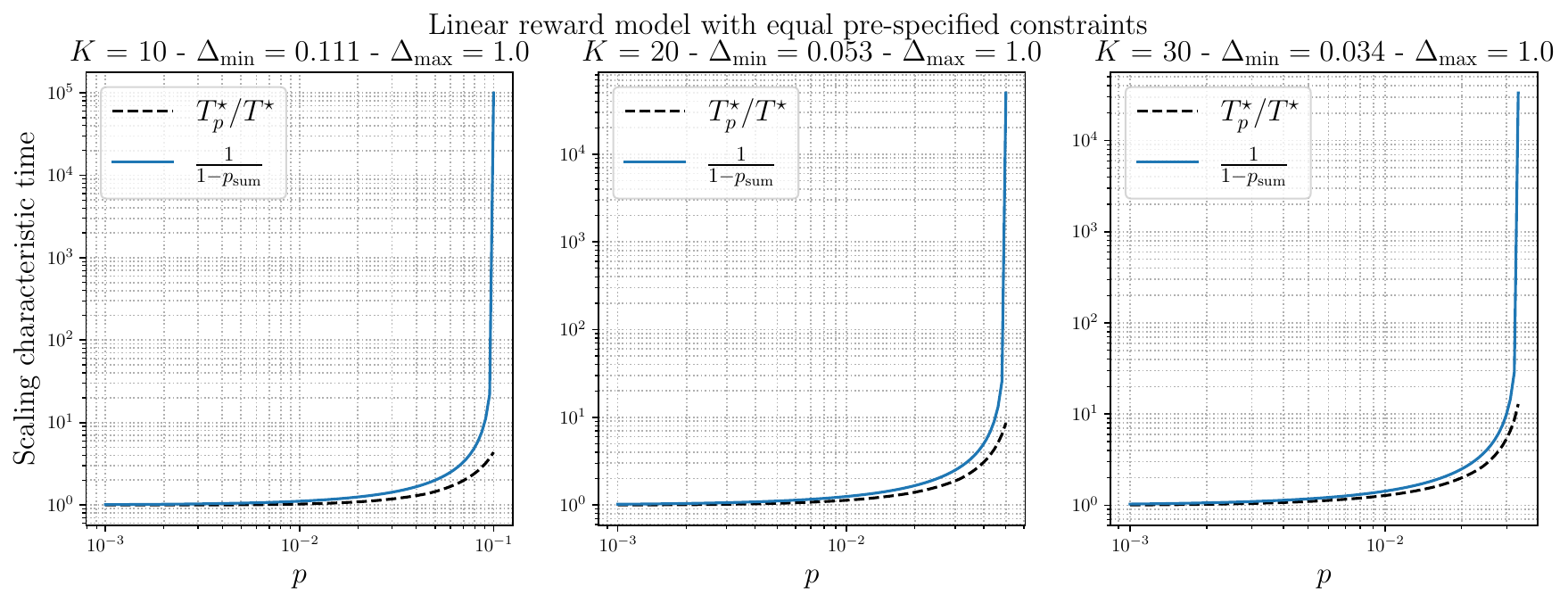}
    \includegraphics[width=.8\linewidth]{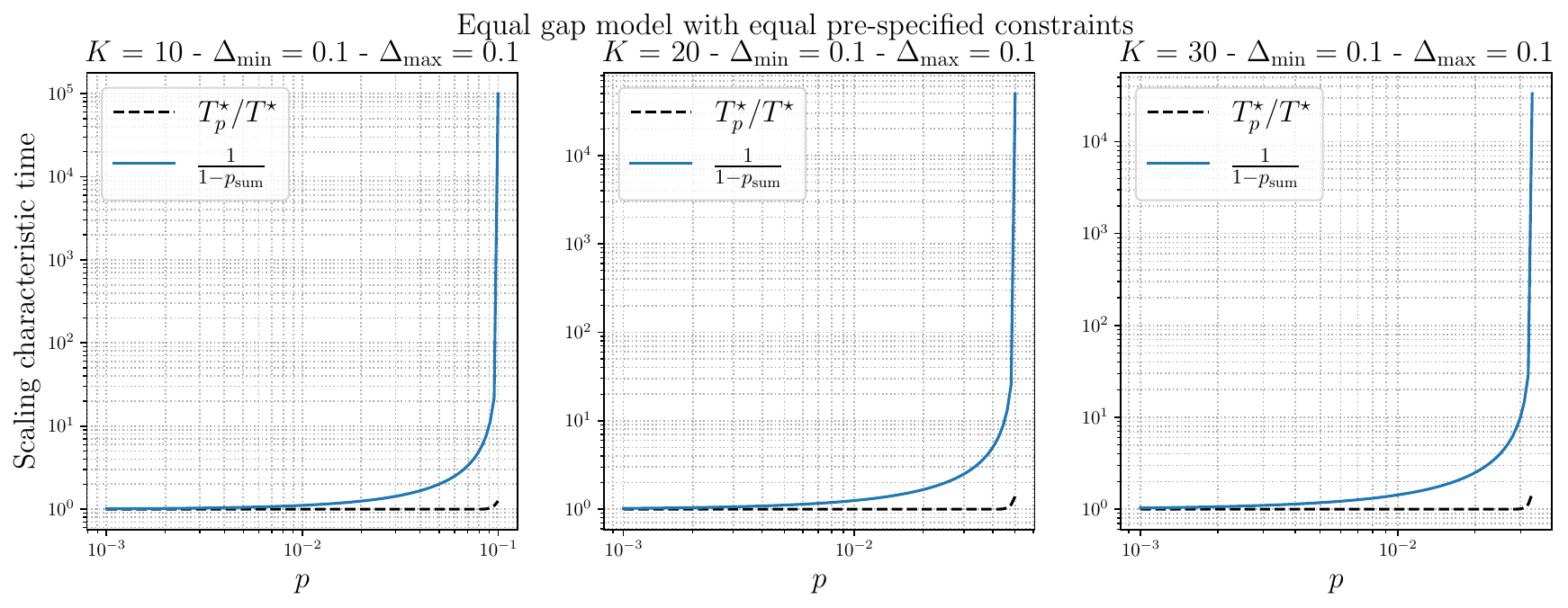}
    \caption{Scaling of $T_p^\star/T^\star$ for two different models with equal pre-specified constraints. All the arms have the same fairness constraint $p_a\geq p$. On the top we show results for the model with rewards linearly ranging in $[0,1]$. On the bottom we show results for the model with equal gaps $\Delta_a=0.1$. From left to right we depict results for different number of arms $K\in\{10,20,30\}$. }
    \label{fig:scaling_price_of_fairness_1}
\end{figure}

\begin{figure}[h!]
    \centering
    \includegraphics[width=.8\linewidth]{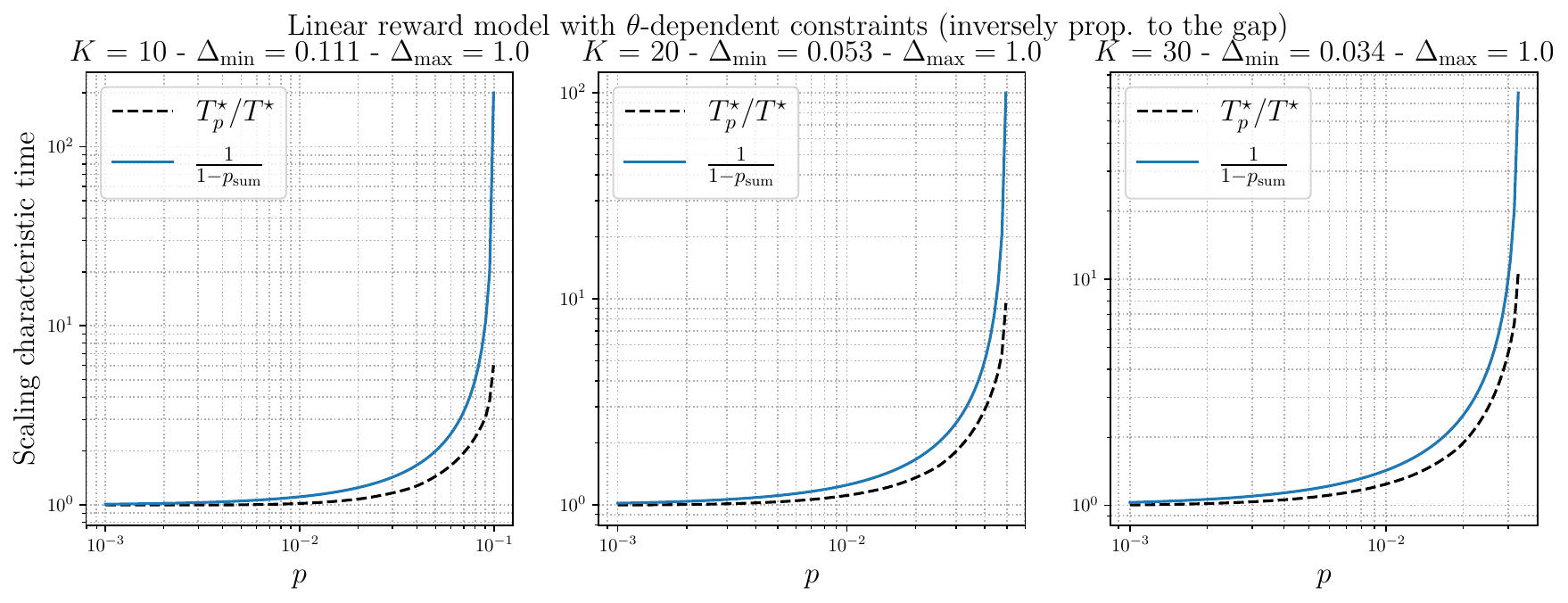}
    \includegraphics[width=.8\linewidth]{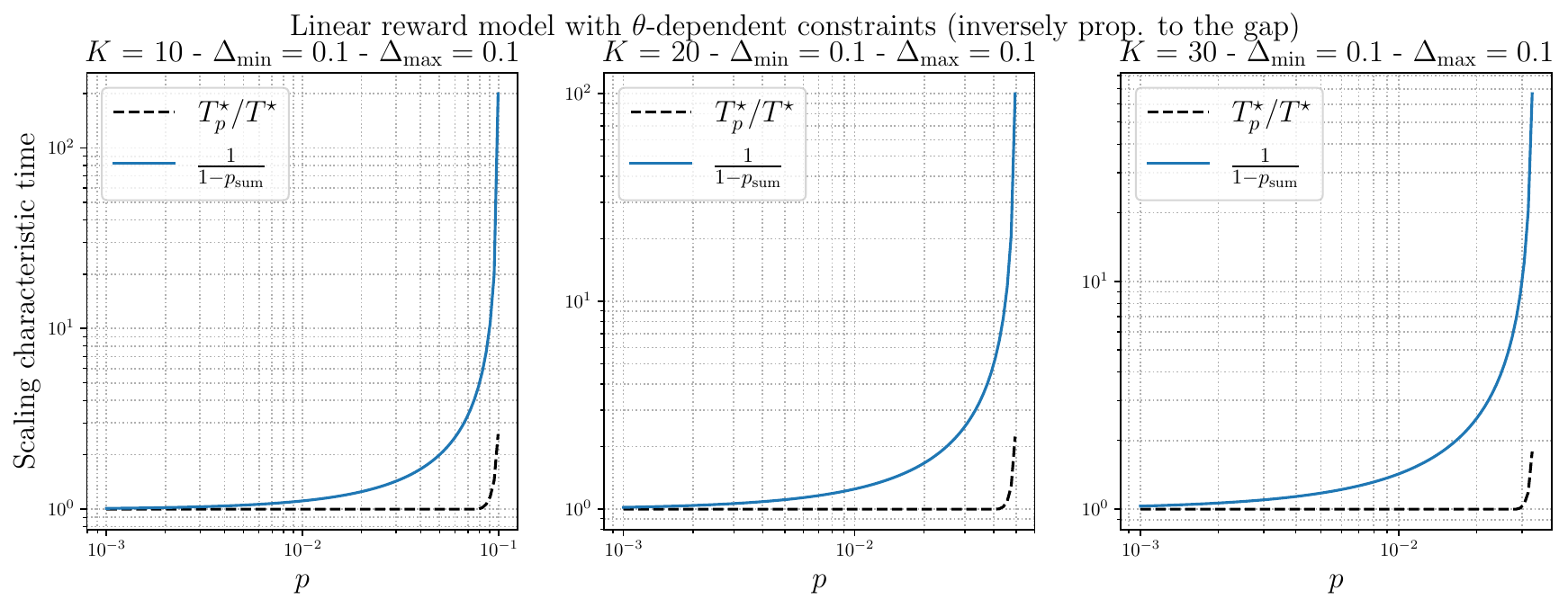}
    \caption{Scaling of $T_p^\star/T^\star$ for two different models with equal $\theta$-dependent constraints (see also the $p(\theta)$ function in \cref{eq:pof_eq_fairness_theta}). On the top we show results for the model with rewards linearly ranging in $[0,1]$. On the bottom we show results for the model with equal gaps $\Delta_a=0.1$. From left to right we depict results for different number of arms $K\in\{10,20,30\}$. }
    \label{fig:scaling_price_of_fairness_2}
\end{figure}

\newpage
\subsection{Characterizing the optimal allocations} 
\label{sec:characterize_lb}

\begin{lemma}
\label{lem:opt_all_eq}
The optimal allocations to \eqref{eq:lb} satisfies $ w_{p,a}^\star = p_a$, for all $a\in[K]$ such that $p_a \ge w^\star_{a}$.
\end{lemma}

\begin{proof}
The problem \eqref{eq:lb} can be equivalently rewritten in epigraph form as 
\begin{align*}
T^\star_{p} := \min_{w,z} \quad & z \\
\textrm{s.t.} \quad
& z\ge \frac{w_a^{-1} + w_{a^\star}^{-1}}{\Delta_a^2} , \quad \forall a\neq a^\star,\\
& w_a \ge p_a, \quad \forall a\in [K], \\
& \sum_{a\in[K]} w_a = 1  .
\end{align*}

Hence, this problem can be rewritten as 

\begin{align*}
T^\star_{p} := \min_{x} \quad & f(x) \\
\textrm{s.t.} \quad
& G(x) \le \theta(p) ,\\
& x\in \Omega  \\
\end{align*}
where $x = (z,w)$, $\Omega = (0,\infty) \times \Sigma_K$, $\Sigma_K$ is the $(K-1)$-dimensional simplex, and 

\[
G(x) \coloneqq\left[\phantom{\begin{matrix}a_0\\ \ddots\\a_0\\b_0\\ \ddots\\b_0 \end{matrix}}
\right.\hspace{-1.5em}
\begin{matrix}
\frac{w_{1}^{-1} +{w_a^\star}^{-1}}{\Delta_{1}^2} -z \\
\vdots \\
 \frac{w_{K}^{-1} +{w_a^\star}^{-1}}{\Delta_{K}^2} -z \\
-w_1 \\
 \vdots \\
 -w_K
\end{matrix}
\hspace{-1.5em}
\left.\phantom{\begin{matrix}a_0\\ \ddots\\a_0\\b_0\\ \ddots\\b_0\end{matrix}}\right]\hspace{-1em}
\begin{tabular}{l}
$\left.\lefteqn{\phantom{\begin{matrix} a_0\\ \ddots\\ a_0\ \end{matrix}}}\right\}K-1$\\
$\left.\lefteqn{\phantom{\begin{matrix} b_0\\ \ddots\\ b_0\ \end{matrix}}} \right\}K$
\end{tabular}
, \quad
\theta(p)  \coloneqq\left[\phantom{\begin{matrix}a_0\\ \ddots\\a_0\\b_0\\ \ddots\\b_0 \end{matrix}}
\right.\hspace{-1.5em}
\begin{matrix}
0 \\
\vdots \\
 0  \\
-p_1 \\
 \vdots \\
 -p_K
\end{matrix}
\hspace{-1.5em}
\left.\phantom{\begin{matrix}a_0\\ \ddots\\a_0\\b_0\\ \ddots\\b_0\end{matrix}}\right]\hspace{-1em}
\begin{tabular}{l}
$\left.\lefteqn{\phantom{\begin{matrix} a_0\\ \ddots\\ a_0\ \end{matrix}}}\right\}K-1$\\
$\left.\lefteqn{\phantom{\begin{matrix} b_0\\ \ddots\\ b_0\ \end{matrix}}} \right\}K$.
\end{tabular}
\]

Then, the Lagrangian can be then written as
$$
L(x,\lambda) = f(x) + \lambda^\top G(x),
$$
where $\lambda$ is the Lagrange multiplier, and we denote by $\lambda_p^\star$ the optimal multiplier for a given value of $p$.

Note that if $p < w^\star$, the optimal Lagrange multiplier must satisfy 
\[
\lambda_p^\star =\left[\phantom{\begin{matrix}a_0\\ \ddots\\a_0\\b_0\\ \ddots\\b_0 \end{matrix}}
\right.\hspace{-1.5em}
\begin{matrix}
\ge 0  \\
\vdots \\
\ge 0 \\
0 \\
\vdots \\
0
\end{matrix}
\hspace{-1.5em}
\left.\phantom{\begin{matrix}a_0\\ \ddots\\a_0\\b_0\\ \ddots\\b_0\end{matrix}}\right]\hspace{-1em}
\begin{tabular}{l}
$\left.\lefteqn{\phantom{\begin{matrix} a_0\\ \ddots\\ a_0\ \end{matrix}}}\right\}K-1$\\
$\left.\lefteqn{\phantom{\begin{matrix} b_0\\ \ddots\\ b_0\ \end{matrix}}} \right\}K$
\end{tabular},
\]
i.e., the Lagrange multipliers corresponding the constraints $z \ge \frac{w_{a}^{-1} +{w_a^\star}^{-1}}{\Delta_{a}^2} $ can be either zero or positive, while those corresponding to $w\ge p$ must be null since the constraints are not active. 

Denote by $x_p$ the solution for a given value of $p$. For $p = 0$, we have that the solution ($x_0$) satisfies $f(x_0) = T^\star$.  

For simplicity, suppose that $\exists a_0 \in [K]: w^\star_{a_0} < p_{a_0}$. By contradiction, we assume that the optimal $w_p^\star$ satisfies $w_{p,a_0}^\star > p_{a_0}$.
Hence, this last inequality implies that $\lambda_{p,i_{a_0}}^\star = 0$, where $i_{a_0}$ the index corresponding to the constraint $w^\star_{a_0} \geq p_{a_0}$ (this argument applies to any action that satisfies $w_{p,a}^\star > p_a$). Then, since for any $p\neq0$ we have $T^\star< T_p^\star$ (by strong convexity), by Lem. \ref{lem:luenberger} we have that

$$0> T^\star - T_p^\star \ge \theta(p)^\top \lambda_p^\star = 0,$$ which is a contradiction. Hence $w_{p,a_0}^\star =p_{a_0}$.

\end{proof}

\begin{lemma}[Theorem 1 \cite{Luenberger97}]
\label{lem:luenberger}
    Let $f$ and $G$ be convex functions, $\Omega$ a convex set, and suppose $x_0, x_1$ are solutions to 
    \begin{align*}
    \min_x \;\; & f(x)
    \\ \mathrm{s.t}\;\;\; & x \in \Omega,
    \\ &  G(x) \le u,
    \end{align*}
    with $u=u_0$ and $u=u_1$ respectively. Suppose $z_0^\star$, $z_1^\star$ are (optimal) Lagrange multipliers corresponding to these problems. Then
    $$
    (u_1-u_0)^\top z_1^\star \le f(x_0) -f(x_1) \le (u_1-u_0)^\top z_0^\star.
    $$
\end{lemma}

\newpage 

\renewcommand{\thetheorem}{C.\arabic{theorem}}
\renewcommand{\thelemma}{C.\arabic{lemma}}

\section{{\sc \color{black}F-TaS} Algorithm}
\label{sec:fbai_algo_app}
In this appendix, we provide an analysis of the {\sc \color{black}F-TaS} algorithm. First, we analyze the fairness in App. \ref{sec:fairness_delta_pac}, and later the sample complexity in App. \ref{sec:sample_complex_proofs}. The main sample complexity results are given in 2 forms: almost sure sample-complexity optimality and optimality in expectation. For conciseness, we provide unified results for both \emph{pre-specified rates} and \emph{$\theta$-dependent rates}. 

\subsection{Fairness and \texorpdfstring{$\delta$}{delta}-PAC Results}
\label{sec:fairness_delta_pac}
We provide first a generic result on the guarantees for the case of \emph{pre-specified rates}, and later provide the proof of \cref{prop:p_delta_PAC}.
\begin{proposition}
\label{prop:sat_fairness}
    For $t\geq 1$, {\sc \color{black}F-TaS} with \emph{pre-specified rates} guarantees that $\mathbb{E}_\theta[N_a(t)]\geq t p_a$ if $p_a>0$, and $ \mathbb{E}_\theta[N_a(t)] \geq (1-\psum) (\sqrt{t+1}-1)/K_0$ otherwise, where $K_0 = |\{a\in[K]: p_a=0\}|$. For {\sc \color{black}F-TaS} with $\theta$-dependent rates the arms are being selected at a rater greater than $\sqrt{t}$, i.e., $\mathbb{E}_\theta[N_a(t)]\geq O(\sqrt{t}).$
\end{proposition}
\begin{proof}
   Due to the properties of $w^\star$, for $t>1$ and any arm $a\in[K]$, we have that
\begin{align*}
    \mathbb{E}_\theta[N_a(t)] &\geq 1+\sum_{s=1}^{t} 
    \left(1-\frac{1}{2\sqrt{s}}\right)p_a + \frac{\pi_{c,a}}{2\sqrt{s}},\\
    &\geq tp_a + \frac{(\pi_{c,a}-p_a)}{2}\sum_{s=1}^{t} \frac{1}{\sqrt{s}},\\
    &\geq tp_a + \frac{(\pi_{c,a}-p_a)}{2} \int_1^{t+1}\frac{1}{\sqrt{s}} ds,\\
    &\geq tp_a + (\pi_{c,a}-p_a) (\sqrt{t+1}-1).
\end{align*}
Hence,  for an arm $a$ such that $p_a>0 $ we find $ \mathbb{E}[N_a(t)] \geq tp_a$  since $\pi_{c,a}\geq p_a$, and $ \mathbb{E}[N_a(t)] \geq (1-\psum) (\sqrt{t+1}-1)/K_0$ otherwise.

For $\theta$-dependent rates the result follows immediately by noticing that $\mathbb{E}[N_a(t)]\geq (\sqrt{t+1}-1)/K$.
\end{proof}

Then, for {\sc \color{black}F-TaS} we are able to give the following fairness guarantees.

\begin{proof}[Proof of Proposition \ref{prop:p_delta_PAC}]
    The proof for $\mathbb{P}_\theta(\tau_\delta <\infty, \hat a_{\tau_\delta}\neq a_\theta^\star)\leq \delta$ comes directly from Thm. 7 in \cite{kaufmann2021mixture}, and we omit it for simplicity.
    
    \paragraph{Fairness with pre-specified rates.} From \cref{prop:sat_fairness} we can use \cref{lemma:algo_fairness_simple} to prove the result since $\mathbb{E}_\theta[N_a(t)]\geq t p_a$.
    
    \paragraph{Asymptotic fairness with $\theta$-dependent rates.} The proof for the asymptotic result is more convoluted and relies on different tools that we present later in the appendix. For a more informal proof, we refer the reader to \cref{lemma:asymptotic_algo_fairness_new}. Here, we provide the proof only for the constraints defined in (\ref{eq:constraint_bai_constant_p})-(\ref{eq:constraint_bai_ptheta}) (the proof of the asymptotic sample-path fairness constraints follows immediately).

    For $T\geq 1, \varepsilon > 0,$ consider the concentration events 
    \[
    C_{T,0}(\varepsilon) = \cap_{t=h_0(T)}^T (\|\hat \theta(t) - \theta\|_\infty \leq \varepsilon) \hbox{ and } C_{T,1}(\varepsilon) = \cap_{t= h_1(T)}^T (\|N(t)/t-  w_{p}^\star\|_\infty \leq K(\varepsilon)),
    \]
    where $h_0(t)=t^{1/4}$, $h_1(t)=t^{1/2}$ and $K(\varepsilon)$ is a bounded function, vanishing in $0$ (see also \cref{prop:concentration_average_sampling}). Define then the value of the problem 
    \[
    T_{\theta,p}^\star(\varepsilon) = 
    \sup\limits_{\substack{\tilde\theta:\|\theta-\tilde\theta\|_\infty \leq \varepsilon,\\\tilde w: \|\tilde w-w_{p}^\star\|_\infty \leq K(\varepsilon) }} T_{\tilde \theta,p}^\star(\tilde w),\hbox{ with }T_{\tilde\theta,p}^\star(\tilde w)\coloneqq {\color{black} 2}\max_{a\neq a_{\tilde\theta}^\star} \frac{ \tilde w_a^{-1} + \tilde w_{a_{\tilde\theta}^\star}^{-1}}{\tilde \Delta_a^2},
    \]
    where $T_{\tilde\theta,p}^\star(\tilde w)$ is the value of the problem for model $\tilde\theta$ with allocation $\tilde w$ and $a_{\tilde\theta}^\star =\arg\max_a \tilde\theta_a$, $\tilde \Delta_a=\max_{b} \tilde\theta_b - \tilde\theta_a$.

    Due to \cref{prop:concentration_average_sampling}, we have that for $T_1(\varepsilon)\geq 1/\varepsilon^4$ and $T\geq T_1(\varepsilon)$, conditionally on $C_{T,0}(\varepsilon)$, the event $C_{T,1}(\varepsilon)$ occurs with high probability. Then, for every $t\in [\sqrt{T}, T]$ under $C_{T,0}(\varepsilon)\cap C_{T,1}(\varepsilon)$ we have
    $   Z(t) \geq t/ T_{\theta,p}^\star(\varepsilon)
    $. In the following, since we let $\delta\to 0$, we  choose $\delta < \varepsilon$, and let $T_1(\varepsilon)= 1/\delta^4 \geq 1/\varepsilon^4$.
    
    Next, as in \cref{thm:fairbai_sample_complexity_as_bound} let $T_2(\varepsilon) = \inf\{t:  t/T_{\theta,p}^\star(\varepsilon) \geq \ln(Bt/\delta)\}$, where $B>0$ is a constant chosen as in \cref{thm:fairbai_sample_complexity_as_bound} s.t. $\beta(\delta,t) \leq \ln(Bt/\delta)$.
    Then, for all $t\geq \max(T_1(\varepsilon), T_2(\varepsilon))$,  under $C_{T,0}(\varepsilon)\cap C_{T,1}(\varepsilon)$, we have that
    \[
        Z(t) \geq   t/T_{\theta,p}^\star(\varepsilon) \geq \ln(Bt/\delta) \geq \beta(\delta,t).
    \]

    Let then $T_\varepsilon = \max(T_1(\varepsilon),T_2(\varepsilon))$. Note that $T_\varepsilon \geq 1/\delta^4$.
    Hence $(\tau_\delta \leq T_\varepsilon)\supset C_{T_\varepsilon,0}(\varepsilon)\cap C_{T_\varepsilon,1}(\varepsilon).$

Next, note
\begin{align*}
    \mathbb{E}_\theta[N_a(\tau_\delta)] &= \mathbb{E}_\theta[N_a(\tau_\delta)|C_{T_\varepsilon,0}(\varepsilon)\cap C_{T_\varepsilon,1}(\varepsilon)] \mathbb{P}_\theta(C_{T_\varepsilon,0}(\varepsilon)\cap C_{T_\varepsilon,1}(\varepsilon)) \\& + \mathbb{E}_\theta[N_a(\tau_\delta)|\overline{C_{T_\varepsilon,0}(\varepsilon)}\cup \overline{C_{T_\varepsilon,1}(\varepsilon)}] \mathbb{P}_\theta(\overline{C_{T_\varepsilon,0}(\varepsilon)}\cup \overline{C_{T_\varepsilon,1}(\varepsilon)}),\\
    &\geq \mathbb{E}_\theta[N_a(\tau_\delta)|C_{T_\varepsilon,0}(\varepsilon)\cap C_{T_\varepsilon,1}(\varepsilon)] \mathbb{P}_\theta(C_{T_\varepsilon,0}(\varepsilon)\cap C_{T_\varepsilon,1}(\varepsilon)),\\
    &\geq (w_{p,a}^\star-K(\varepsilon))\mathbb{E}_\theta[\tau_\delta]\mathbb{P}_\theta(C_{T_\varepsilon,0}(\varepsilon)\cap C_{T_\varepsilon,1}(\varepsilon))
\end{align*}

Now we prove that $\liminf_{\delta \to 0} \mathbb{P}_\theta(C_{T_\varepsilon,0}(\varepsilon)\cap C_{T_\varepsilon,1}(\varepsilon)) =1$. Observe $\mathbb{P}_\theta(C_{T_\varepsilon,0}(\varepsilon)\cap C_{T_\varepsilon,1}(\varepsilon))=\mathbb{P}_\theta( C_{T_\varepsilon,1}(\varepsilon) | C_{T_\varepsilon,0}(\varepsilon))\mathbb{P}_\theta(C_{T_\varepsilon,0}(\varepsilon))$.

From \cref{prop:concentration_average_sampling}   we have
\[
\mathbb{P}_\theta( C_{T_\varepsilon,1}(\varepsilon) | C_{T_\varepsilon,0}(\varepsilon)) \geq 1- 2K \frac{\exp(-\sqrt{T_\varepsilon}\varepsilon^2/2)}{1-\exp(-\varepsilon^2/2)}
\]
Since $\sqrt{T_\varepsilon}\geq1/\delta^2$ we get $\liminf_{\delta \to 0}\mathbb{P}( C_{T_\varepsilon,1}(\varepsilon) | C_{T_\varepsilon,0}(\varepsilon))\geq 1.$
Then, from \cref{prop:concentration_estimate} we have
\[
\mathbb{P}_\theta(C_{T_\varepsilon,0}(\varepsilon)) \geq 1- \frac{1}{T_\varepsilon^\alpha} - 2B_\varepsilon KT_\varepsilon\exp\left(-2\left\lfloor\frac{ p_0^{1/4}T_\varepsilon^{1/16}}{(\log^2(1+K'T_\varepsilon^\alpha))^{1/4}} \right\rfloor \varepsilon^2 \right),
\]
where $\alpha >1$. Asymptotically, as $\delta\to 0$, we  have that also this term converges to $1$ due to the exponential converging to $0$ and $1/T_\varepsilon^\alpha \to 0$. Hence
\[
  \liminf_{\delta \to 0} \frac{\mathbb{E}_\theta[N_a(\tau_\delta)]}{\mathbb{E}_\theta[\tau_\delta]}
    \geq (w_{p,a}^\star-K(\varepsilon)).
\]
Letting $\varepsilon \to 0$ concludes the proof since $K(0)=0$.
\end{proof}

\subsection{Sample complexity guarantees}
\label{sec:sample_complex_proofs}
\subsubsection{Almost-sure Sample Complexity Upper Bound}
In this section, we prove an almost sure sample complexity bound of {\sc \color{black}F-TaS}. To derive this result, 
first, we prove that each arm is sampled infinitely often. Later, we show that, asymptotically, the average number of times we select an arm $a$ converges to $w_{p,a}^\star$ almost surely. 
\begin{proposition}\label{prop:fair_bai_infinite_sampling}
    Each arm is sampled infinitely often in {\sc \color{black}F-TaS}, i.e., $
    \mathbb{P}_\theta(\lim_{t\to \infty} N_a(t) =\infty)=1
    $ for all $a$.
\end{proposition}
\begin{proof}
\noindent\textbf{Case: pre-specified rates.}
        The policy in {\sc \color{black}F-TaS} for $a$ such that $p_a>0$ ensures that
	$\mathbb{P}_\theta(a_t=a) \geq p_a$ Consequently, we have that
	\[
	\sum_{t=1}^\infty \mathbb{P}_\theta(a_t=a) \geq \sum_{t=1}^\infty p_a = \infty.
	\]
	By the Borel-Cantelli lemma it follows that arm $a$ is chosen infinitely often asymptotically.

 Now consider an arm $a$ such that $p_a=0$. Then {\sc \color{black}F-TaS}
 guarantees that $\mathbb{P}_\theta(a_t=a) \geq \epsilon_t\frac{1-\psum}{K_0}$, hence
 \[
\sum_{t=1}^\infty \mathbb{P}_\theta(a_t=a) \geq \sum_{t=1}^\infty \frac{1-\psum}{2K_0\sqrt{t}}=\infty.
 \]
Hence, each arm is sampled infinitely often.

\noindent \textbf{Case: $\theta$-dependent rates.} In this case
      {\sc \color{black}F-TaS}
 guarantees that $\mathbb{P}_\theta(a_t=a) \geq \epsilon_t\frac{1}{K}$, thus
 \[
\sum_{t=1}^\infty \mathbb{P}_\theta(a_t=a) \geq \sum_{t=1}^\infty \frac{1}{2K\sqrt{t}}=\infty.
 \]
Hence, each arm is sampled infinitely often.
\end{proof}

We now show that {\sc \color{black}F-TaS} asymptotically samples arms according to $w_{p}^\star$.
\begin{proposition}\label{propo:almost_sure_forced_asymptotic_forced_exploration_fair_bai}
For every arm $a\in [K]$, {\sc \color{black}F-TaS} satisfies
\begin{equation}
    \mathbb{P}_\theta\left(\lim_{t\to\infty} \frac{N_a(t)}{t} = w_{p,a}^\star\right)=1.
\end{equation}
\end{proposition}
\begin{proof}
\cref{prop:fair_bai_infinite_sampling} guarantees that, by the law of large numbers,  $(\hat \theta_t\to\theta)$ almost surely as $t\to\infty$. By continuity, we also have that (P1) $w_{p}^\star(t)\to w_{p}^\star$ almost surely (by an application of Berge’s Theorem). 
Then, consider \[
 \frac{1}{t}N_t(a) - w_{p,a}^\star=\frac{1}{t} \sum_{k=1}^t [\mathbf{1}_{(a_k=a)} - w_{p,a}^\star]=  \underbrace{\frac{1}{t} \sum_{k=1}^t [\mathbf{1}_{(a_k=a)}  - w_{p,a}^\star(k) ]}_{(\circ)}+  \underbrace{\frac{1}{t} \sum_{k=1}^t [w_{p,q}^\star(k) - w_{p,q}^\star]}_{(\square)}.
 \]
 The second term $(\square)$ clearly tends to $0$ almost surely from property (P1). To prove that the first term $(\circ)$ converges to $0$ rewrite it as
 \[
 (\circ)=\underbrace{\frac{1}{t} \sum_{k=1}^t [\mathbf{1}_{(a_k=a)}  - \pi_{a}(k)]}_{M_t} + \underbrace{\frac{1}{t} \sum_{k=1}^t [\pi_{a}(k)-w_{p,a}^\star(k) ]}_{(*)}.
 \]
For the first term let $S_t=tM_t$, and note that $S_t$ is a martingale since $\mathbb{E}[tM_t | {\cal F}_{t-1}]=(t-1)M_{t-1} + \mathbb{E}[\mathbf{1}_{(a_t=a)}  - \pi_{a}(t)|{\cal F}_{t-1}]=(t-1)M_t$. To show that $S_t/t\to 0$ we use Lemma 2.18 \cite{hall2014martingale}. To that aim, it is sufficient to show that $\sum_k X_i^2/k^2 < \infty$, where  $X_i=S_i-S_{i-1}$. Since $|X_i|\leq 1$ then the series convergences, 
from which follows that $\lim_{t\to\infty} \frac{S_t}{t} = 0$. Hence $M_t\to 0$ almost surely.
 
For the second term $(*)$ we find $\pi_{a}(k)-w_{p,a}^\star(k) = \epsilon_k(\pi_{c,a} -w_{p,a}^\star(k))$, hence $|\pi_{a}(k)-w_{p,a}^\star(k)| \leq  \epsilon_k$. Therefore we conclude the proof by observing the convergence of $\frac{1}{t} \sum_{k=1}^t |\pi_{a}(k)-w_{p,a}^\star(k)|$ to $0$:
\[
 \frac{1}{t} \sum_{k=1}^t |\pi_{a}(k)-w_{p,a}^\star(k)| \leq \frac{1}{2t} \sum_{k=1}^t \frac{1}{\sqrt{k}} \leq \frac{2\sqrt{t}-1}{2t}\to 0 \hbox{ as } t \to \infty.
\]
 \end{proof}

We can now prove an almost sure upper bound of the sample complexity of {\sc \color{black}F-TaS}.

\begin{theorem}[ Sample complexity almost sure upper bound of {\sc \color{black}F-TaS}.]\label{thm:fairbai_sample_complexity_as_bound}
{\sc \color{black}F-TaS}, both for pre-specified rates and $\theta$-dependent rates, guarantees that
\begin{equation}
    \mathbb{P}_\theta\left( \limsup_{\delta \to 0} \frac{\tau_\delta}{\ln(1/\delta)} \leq T_{p}^\star \right) =1.
\end{equation}
\end{theorem}
\begin{proof}
    The proof follows similarly as in \cite{garivier2016optimal, russo2023sample}, and we provide it for completeness. Denote by $T_p^\star(t) = {\color{black} 2}\inf_{w\in \Sigma_p} \max_{a\neq a_t^\star} \frac{w_a^{-1}+w_{a_t^\star}^{-1}}{\Delta_a(t)^2}$ the optimal characteristic time for a model $\hat \theta(t) \in \Theta$.
    Consider the event ${\cal E}=\left(\forall a, \lim_{t\to\infty}\frac{N_a(t)}{t}= w_{p,a}^\star, \lim_{t\to\infty}\hat\theta(t)=\theta \right)$. From Prop. \ref{propo:almost_sure_forced_asymptotic_forced_exploration_fair_bai} we have $\mathbb{P}_\theta({\cal E})=1$. 
      Then, there exists $t_0$ s.t. for all $t\geq t_0$  we have  $\hat \theta(t)\in \Theta$.
     
     Furthermore, due to the continuity of $T_p^\star(t)$, for every $\eta \in (0,1)$ there exists $t_1\geq t_0$ such that for $t\geq t_1$ we have  $T_{p}^\star \geq  (1-\eta) T_{p}^\star(t) \geq (1-\eta) t/Z(t)$,  thus $Z(t) \geq (1-\eta)t /T_p^\star$.
    
    Now, recall that the stopping time is defined through $\beta(\delta,t)=3\sum_a \ln(1+\ln(N_a(t))) +K {\cal C}_{exp}\left(\frac{\ln(1/\delta)}{K}\right)
    $.
    Since at infinity ${\cal C}_{exp}(x) \sim x +O(\ln(x))$ then
    there exists  $C>0$ s.t. 
    $K {\cal C}_{exp}\left(\frac{\ln(1/\delta)}{K} \right)\leq \ln(C/\delta)$. Moreover, $3\sum_a \ln(1+\ln(N_a(t))) \leq 3 K \ln(1+t)$. Hence, there exists a constant $B>0$ such that $\beta(\delta,t)\leq \ln(Bt/\delta)$.

    Combining all the observations, we find
    \begin{align*}
    \tau_\delta &= \inf\{t\geq t_0, Z(t) \geq \beta(\delta,t)\},\\
    &\leq t_1 \vee \inf\{t\geq t_0, (1-\eta)t / T_{p}^\star \geq \beta(\delta,t)\},\\
    &\leq t_1 \vee \inf\{t\geq t_0, (1-\eta)t / T_{p}^\star \geq \ln(Bt/\delta)\}.
    \end{align*}
    Applying Lemma 8 in the appendix of \cite{russo2023sample} with $\beta= B/\delta$ and $\gamma = T_{p}^\star/(1-\eta)$ gives that
    \[
    \tau_\delta \leq \max\left(t_1,\frac{T_{p}^\star}{1-\eta}\left[\ln\left(\frac{BT_{p}^\star}{\delta(1-\eta)}\right) + \sqrt{2\left(\ln\left(\frac{BT_{p}^\star}{\delta(1-\eta)}\right)-1\right)}\right]\right).
    \]
    Therefore $\limsup_{\delta \to 0} \frac{\tau_\delta}{\ln(1/\delta)} \leq \frac{T_{p}^\star}{1-\eta} $ almost surely. We conclude by letting $\eta\to 0$.
\end{proof}

\subsubsection{Expected Sample Complexity Bound}

In order to prove the sample-complexity of {\sc \color{black}F-TaS} we must guarantee the forced exploration property with high probability. To this end, \cref{prop:forced_exploration_fair_bai} is instrumental in the derivation of the sample complexity of {\sc \color{black}F-TaS}.

Using this result we can bound the probability of the event that $\hat \theta(t)$ is not within $\varepsilon>0$ of the true value of $\theta$, and derive the asymptotic sample complexity bound.
Using \cref{prop:concentration_average_sampling} and \cref{prop:concentration_estimate} in the next theorem we provide an upper bound on the expected sample complexity of {\sc \color{black}F-TaS}.
\begin{theorem}[Upper bound in expectation of {\sc \color{black}F-TaS}]\label{thm:fairbai_expected_sample_complexity}
For all $\delta\in(0,1/2)$ {\sc \color{black}F-TaS} satisfies $\mathbb{E}_\theta[\tau_\delta]<\infty$ and $\limsup_{\delta \to 0} \frac{\mathbb{E}_\theta[\tau_\delta]}{\ln(1/\delta)} \leq T_{p}^\star.$
\end{theorem}
\begin{proof}
    For $T\geq 1, \varepsilon>0$ consider the concentration events 
    \[
    C_{T,0}(\varepsilon) = \cap_{t=h_0(T)}^T (\|\hat \theta(t) - \theta\|_\infty \leq \varepsilon) \hbox{ and } C_{T,1}(\varepsilon) = \cap_{t=h_1(T)}^T (\|N(t)/t-  w_{p}^\star\|_\infty \leq K(\varepsilon)),
    \]
    where $h_0(t)=t^{1/4}$, $h_1(t)=t^{1/2}$ and $K(\varepsilon)$ is a bounded function, vanishing in $0$ (see also \cref{prop:concentration_average_sampling}). Define then the value of the problem 
   \[
    T_{\theta,p}^\star(\varepsilon) = 
    \sup\limits_{\substack{\tilde\theta:\|\theta-\tilde\theta\|_\infty \leq \varepsilon,\\\tilde w: \|\tilde w-w_{p}^\star\|_\infty \leq K(\varepsilon) }} T_{\tilde \theta,p}^\star(\tilde w),\hbox{ with }T_{\tilde\theta,p}^\star(\tilde w)\coloneqq {\color{black}{2}} \max_{a\neq a_{\tilde\theta}^\star} \frac{ \tilde w_a^{-1} + \tilde w_{a_{\tilde\theta}^\star}^{-1}}{\tilde \Delta_a^2},
    \]
    where $T_{\tilde\theta,p}^\star(\tilde w)$ is the value of the problem for model $\tilde\theta$ with allocation $\tilde w$ and $a_{\tilde\theta}^\star =\arg\max_a \tilde\theta_a$, $\tilde \Delta_a=\max_{b} \tilde\theta_b - \tilde\theta_a$.

    Due to \cref{prop:concentration_average_sampling}, there exists $T_1(\varepsilon)$ s.t. for all $T\geq T_1(\varepsilon)$, conditionally on $C_{T,0}(\varepsilon)$, the event $C_{T,1}(\varepsilon)$ occurs with high probability. Moreover, for every $t\in [\sqrt{T}, T]$ under $C_{T,0}(\varepsilon)\cap C_{T,1}(\varepsilon)$ we have
    $    Z(t) \geq t/ T_{\theta,p}^\star(\varepsilon)
    $.
    
    Next, as in \cref{thm:fairbai_sample_complexity_as_bound} let $T_2(\varepsilon) = \inf\{t:  t/T_{\theta,p}^\star(\varepsilon) \geq \ln(Bt/\delta)\}$, where $B>0$ is a constant chosen as in \cref{thm:fairbai_sample_complexity_as_bound} s.t. $\beta(\delta,t) \leq \ln(Bt/\delta)$.
    Then, for all $t\geq \max(T_1(\varepsilon), T_2(\varepsilon))$,  under $C_{T,0}(\varepsilon)\cap C_{T,1}(\varepsilon)$, we have that
    \[
        Z(t) \geq   t/T_{\theta,p}^\star(\varepsilon) \geq \ln(Bt/\delta) \geq \beta(\delta,t).
    \]
    Hence $(\tau_\delta \leq T)\supset C_{T,0}(\varepsilon)\cap C_{T,1}(\varepsilon).$ Therefore
    \begin{align*}
        \mathbb{E}[\tau_\delta] &= \sum_{T=1}^\infty \mathbb{P}_\theta(\tau_\delta > T),\\
        &\leq \max(T_1(\varepsilon),T_2(\varepsilon)) + \sum_{T=\max(T_1(\varepsilon),T_2(\varepsilon))+1}^\infty \mathbb{P}_\theta(\tau_\delta > T),\\
        &\leq \max(T_1(\varepsilon),T_2(\varepsilon)) + \sum_{T=\max(T_1(\varepsilon),T_2(\varepsilon))+1}^\infty \mathbb{P}_\theta(\overline{C_{T,0}(\varepsilon)}\cup \overline{C_{T,1}(\varepsilon)}),\\
        &\leq \max(T_1(\varepsilon),T_2(\varepsilon)) + \sum_{T=\max(T_1(\varepsilon),T_2(\varepsilon))+1}^\infty \mathbb{P}_\theta( \overline{C_{T,1}(\varepsilon)} | C_{T,0}(\varepsilon))+ \mathbb{P}_\theta(\overline{C_{T,0}(\varepsilon)}),\\
        &\leq \max(T_1(\varepsilon),T_2(\varepsilon)) + \sum_{T=\max(T_1(\varepsilon),T_2(\varepsilon))+1}^\infty \mathbb{P}_\theta( \overline{C_{T,1}(\varepsilon)} | C_{T,0}(\varepsilon))+ \mathbb{P}_\theta(\overline{C_{T,0}(\varepsilon)}),\\
    \end{align*}
    The last sum by \cref{prop:concentration_estimate} (with $\alpha >1$) and \cref{prop:concentration_average_sampling} is clearly bounded. Hence, also the expected value of $\tau_\delta$ is bounded for all values of $\epsilon \in(0,1)$. Therefore, as in \cref{thm:fairbai_sample_complexity_as_bound}, we get $\limsup_{\delta \to 0} \frac{\mathbb{E}_\theta[\tau_\delta]}{\ln(1/\delta)} \leq \limsup_{\delta \to 0} \frac{\max(T_1(\varepsilon),T_2(\varepsilon))}{\ln(1/\delta)}\leq T_{p}^\star(\varepsilon)$. We conclude by letting $\varepsilon\to 0$.
\end{proof}

\begin{proposition}[Concentration of the estimate $\hat \theta_t$]\label{prop:concentration_estimate}
Let $\alpha,\varepsilon>0$, $h(t)= t^{1/4}$ and $C_T(\varepsilon) = \cap_{t= h(T) }^T (\|\hat \theta(t) - \theta\|_\infty \leq \varepsilon)$. Then, there exists a constant $B_\varepsilon>0$ that depends on $\varepsilon$ such that
\begin{equation}
    \forall T\geq 1, \mathbb{P}_\theta(\overline{ C_T(\varepsilon)})\leq \frac{1}{T^\alpha} + 2B_\varepsilon KT\exp\left(-2\left\lfloor\frac{p_0^{\frac{1}{4}} T^{1/16}}{\sqrt{\ln(1+K'T^\alpha)}} \right\rfloor \varepsilon^2 \right),
\end{equation}
where, for { pre-specified rates} $K'=2\max(K_0,K-K_0)$ and $p_0 = \min( (1-\psum)/2K_0, \min_{a:p_a>0} p_a)$. For { $\theta$-dependent rates} instead  $p_0=1/2K$ and $K'=K$.
\end{proposition}
\begin{proof}
    Consider the forced exploration event of \cref{prop:forced_exploration_fair_bai} with $\gamma=1/T^\alpha$ and $\alpha>0$
    \[
    {\cal E}_T=\left(\forall a\in [K],\forall t\geq 1, N_a(t) \geq \left\lfloor\left(\frac{p_0 t}{\ln^2(1+K'T^\alpha)} \right)^{1/4}\right\rfloor\right).
    \]

    Then, using the same proposition we obtain
    \begin{align*}
    \mathbb{P}_\theta(\overline{ C_T(\varepsilon)}) &= \mathbb{P}_\theta(( \overline{ C_T(\varepsilon)} \cap {\cal E}_T) \cup (\overline{ C_T(\varepsilon)} \cap \overline{ {\cal E}_T})),\\
    &\leq \mathbb{P}_\theta((\overline{C_T(\varepsilon)} \cap {\cal E}_T))+\mathbb{P}_\theta( \overline{{\cal E}_T}),\\
    &\leq \mathbb{P}_\theta(\overline{(C_T(\varepsilon)} \cap {\cal E}_T))+\frac{1}{T^\alpha}.
    \end{align*}
    Let $\lambda(T) = \frac{\ln^2(1+K'T^\alpha)}{p_0}$ and consider now the first term. We expand it using a union bound and conclude with an application of Hoeffding inequality:
    \begin{align*}
        \mathbb{P}_\theta(\overline{C_T(\varepsilon)}\cap {\cal E}_T) &\leq \sum_{t=h(T)}^T \mathbb{P}_\theta(\|\hat \theta(t)-\theta\|_\infty > \varepsilon  \cap {\cal E}_T),\\
        &\leq \sum_{t=h(T)}^T \sum_a \mathbb{P}_\theta( |\hat\theta_{a}(t)-\theta_a|>\varepsilon  \cap {\cal E}_T),\\
        &\leq \sum_{t=h(T)}^T \sum_a \sum_{k=\lfloor (t/\lambda(T))^{1/4} \rfloor}^t \mathbb{P}_\theta( |\hat\theta_{a}(t)-\theta_a|>\varepsilon, N_a(t)=k),\\
        &\leq \sum_{t=h(T)}^T \sum_a \sum_{k=\lfloor (t/\lambda(T))^{1/4} \rfloor}^t 2\exp(-2k\varepsilon^2),\\
        &\leq 2\sum_{t=h(T)}^T \sum_a \sum_{k=0}^{t-\lfloor (t/\lambda(T))^{1/4} \rfloor} \exp\left(-2(k+\lfloor (t/\lambda(T))^{1/4} \rfloor)\varepsilon^2 \right),\\
        &\leq 2\sum_{t=h(T)}^T \sum_a \exp\left(-2(\lfloor (t/\lambda(T))^{1/4} \rfloor)\varepsilon^2 \right)\sum_{k=0}^{t-\lfloor t/\lambda(T) \rfloor^{1/4}} \exp\left(-2k\varepsilon^2 \right),\\
        &\leq \frac{2K}{1-e^{-2\varepsilon^2}}\sum_{t=h(T)}^T \exp\left(-2(\lfloor (t/\lambda(T))^{1/4} \rfloor)\varepsilon^2 \right),\\
        &\leq \frac{2KT}{1-e^{-2\varepsilon^2}}\exp\left(-2(\lfloor (h(T)/\lambda(T))^{1/4} \rfloor )\varepsilon^2 \right),\\
        &\leq 2B_\varepsilon KT\exp\left(-2\left\lfloor\frac{p_0^{\frac{1}{4}} T^{1/16}}{\sqrt{\ln(1+K'T^\alpha)}} \right\rfloor \varepsilon^2 \right).
    \end{align*}
\end{proof}

\begin{proposition}[Concentration of the average sampling $N_a(t)/t$]\label{prop:concentration_average_sampling}
Let $h_0(t)= t^{1/4}, h_1(t)=t^{1/2}$ and $\varepsilon>0$.  Further let $T\geq 1/\varepsilon^4$,  $C_{T,0}(\varepsilon) = \cap_{t=h_0(T)}^T (\|\hat \theta(t) - \theta\|_\infty \leq \varepsilon)$ and $C_{T,1}(\varepsilon) = \cap_{t=h_1(T)}^T (\|N(t)/t-  w_{p}^\star\|_\infty \leq K(\varepsilon))$, where $K:[0,1]\to[0,1]$ is the modulus of continuity of $w_{p}^\star$ on $\|\theta-\theta'\|_\infty\leq \theta$, a function continuous in a neighbourhood of $0$ satisfying $\lim_{\varepsilon\to 0} K(\varepsilon)=0$. Then, we have
\begin{equation}
    \mathbb{P}_\theta(\overline{C_{T,1}(\varepsilon)}| C_{T,0}(\varepsilon))\leq 2K \frac{\exp(-\sqrt{T}\varepsilon^2/2)}{1-\exp(-\varepsilon^2/2)}.
\end{equation}
\end{proposition}
\begin{proof}
We first prove that  for $t \in  [h_1(T),T]$ and $T\geq 1/\varepsilon^4$ we have
\begin{equation}
    \mathbb{P}_\theta(\exists a: |N_a(t)/t - w_{p,a}^\star|> K(\varepsilon)  | C_{T,0}(\varepsilon))\leq 2K \exp(-t\varepsilon^2/2),
\end{equation}
where $K(\varepsilon)<\infty$ for all $\varepsilon>0$, and $\lim_{\varepsilon\to 0}K(\varepsilon)=0$. If the last inequality holds, then, the proposition' statement follows by a union bound since
\[
\mathbb{P}_\theta(\overline{C_{T,1} (\varepsilon)}| C_{T,0}(\varepsilon))\leq \sum_{t=h_1(T)}^T 2K \exp(-t\varepsilon^2/2) \leq 2K \frac{\exp(-\sqrt{T} \varepsilon^2/2)}{1-\exp(-\varepsilon^2/2)}.
\]
Then, consider 
\begin{align*}
N_a(t)/t - w_{p}^\star(a) &= \underbrace{\frac{1}{t}\sum_{k=1}^{h_0(T)} [\mathbf{1}_{(a_t=a)} - w_{p,a}^\star ]}_{(\circ)}+\underbrace{\frac{1}{t}\sum_{k=h_0(T)+1}^t [\mathbf{1}_{(a_t=a)} - w_{p,a}^\star(t)]}_{(\square)}+ \underbrace{\frac{1}{t}\sum_{k=h_0(T)+1}^t [w_{p,a}^\star(t) - w_{p,a}^\star ]}_{(*)}.
\end{align*}
\begin{itemize}
    \item For the first term $(\circ)$,  for $t\geq h_1(T)$  and $T\geq 1/\varepsilon^{4}$  we have $(\circ)\leq h_0(T)/h_1(T)= 1/T^{1/4}\leq\varepsilon$.
    \item For the middle term $(\square)$ we have
        \[
        (\square)=\underbrace{\frac{1}{t}\sum_{k=h_0(T)+1}^t [\mathbf{1}_{(a_t=a)} - \pi_{a}(t)]}_{M_t} + \underbrace{\frac{1}{t}\sum_{k=h_0(T)+1}^t [\pi_{a}(t)- w_{p,a}^\star(t)]}_{(\triangle)}
        \]
    Let $S_t=tM_t$,  and observe that $S_t$ is a martingale since $\mathbb{E}[S_t|{\cal F}_{t-1}]=(t-1)M_{t-1}$. Then, using Azuma-Hoeffding inequality we have $\mathbb{P}_\theta( |M_t|\geq  \varepsilon) =\mathbb{P}_\theta( |S_t|\geq t \varepsilon) \leq 2 \exp(- \varepsilon^2 t/2) $.
    
    Instead,  for $(\triangle)$, for $t\geq h_1(T)$ we have \[
    (\triangle) \leq \frac{1}{t}\sum_{k=h_0(T)+1}^t \frac{1}{2\sqrt{k}} \leq \frac{1}{2t} \int_{h_0(T)}^t \frac{1}{\sqrt{x}} dx \leq 1/\sqrt{t} \leq 1/T^{1/4}.
    \]
    Hence for $T \geq 1/\varepsilon^{4}$ we have $(\triangle) \leq \varepsilon$.

\item For the last term $(*)$ under $C_{T,0}(\varepsilon)$ by continuity (Berge's theorem) there exists $K'(\varepsilon)$ s.t. $\|w_{p,a}^\star(t) - w_{p,a}^\star\|_\infty \leq K'(\varepsilon) \leq 1$, hence $(*)\leq K'(\varepsilon)$.

\end{itemize}

In conclusion, by letting $K(\varepsilon)=(3\varepsilon+K'(\varepsilon))$ we obtain 
\begin{equation*}
    \mathbb{P}_\theta(\exists a: |N_a(t)/t - w_{p,a}^\star|> K(\varepsilon)  | C_{T,0}(\varepsilon))\leq 2K \exp(-t\varepsilon^2/2).
\end{equation*}
\end{proof}
\subsection{Fair-BAI: Other Results}
Here we provide the following result that is instrumental to prove the sample-complexity of {\sc Fair-BAI} .
\begin{proposition}[ Forced exploration of {\sc Fair-BAI}]\label{prop:forced_exploration_fair_bai}
    In {\sc Fair-BAI} we have for $\gamma \in(0,1)$
    \[
    \mathbb{P}_\theta\left(\forall a\in [K],\forall t\geq 1, N_a(t) \geq \left\lfloor\left(\frac{p_0 t}{\log^2(1+\frac{K'}{\gamma})} \right)^{1/4}\right\rfloor\right) \geq 1-\gamma,
    \]
    where, for { pre-specified rates} $K'=2\max(K_0,K-K_0)$ and $p_0 = \min( (1-\psum)/2K_0, \min_{a:p_a>0} p_a)$. For { $\theta$-dependent rates} instead  $p_0=1/2K$ and $K'=K$.
\end{proposition}
\begin{proof}
The proof is inspired by the forced exploration property of best-policy identification techniques for MDPs \cite{al2021navigating}. We provide the proof for \emph{pre-specified rates} and later extend it to the case with \emph{$\theta$-dependent rates}.

Define the event
${\cal E}=\{\forall a\in [K], \forall k\geq 1, \tau_a(k) \leq g(k)\}$, where $\tau_a(k)$ is the time arm $a$ is sampled the $k$-th time, and let $g(k)$ be an increasing function of $k$ with $g(k)=0$. Later we specialize to $g(k)=\lambda k^4$ for some $\lambda>0$. To prove the claim, we can instead prove
\[  \mathbb{P}_\theta\left(\forall a\in [K],\forall k\geq 1, \tau_a(k) \leq \frac{\log^2(1+\frac{K'}{\gamma})}{p_0} k^4\right) \geq 1-\gamma.\]
A strategy is to bound $\mathbb{P}_\theta(\bar {\cal E})$, where $\bar{\cal E}= \{\exists a\in [K], \exists k\geq 1, \tau_a(k) > g(k) \wedge \forall n=1,\dots,k-1, \tau_a(n)\leq g(n)\}$.

\paragraph{Decomposition of $\mathbb{P}_\theta(\bar{\cal E})$.} We begin by using a union bound and rewriting the terms that appear.
\begin{align*}
    \mathbb{P}_\theta(\bar{\cal E}) &\leq \sum_a\Big[\mathbb{P}_\theta(\tau_a(1)> g(1))  +\sum_{k\geq 2}\mathbb{P}_\theta( \tau_a(k)> g(k), \tau_a(k-1)\leq g(k-1))  \Big],\\
    &\leq \sum_a\Big[\mathbb{P}_\theta(\tau_a(1)> g(1))  +\sum_{k\geq 2}\mathbb{P}_\theta( \tau_a(k)-\tau_a(k-1)> g(k)-g(k-1), \tau_a(k-1)\leq g(k-1))  \Big],\\
    &\leq \sum_a\Big[\mathbb{P}_\theta(\tau_a(1)> g(1))  +\sum_{k\geq 2}\sum_{n=1}^{g(k-1)}\mathbb{P}_\theta( \tau_a(k)-\tau_a(k-1)> g(k)-g(k-1)| \tau_a(k-1)=n) \mathbb{P}_\theta(\tau_a(k-1)=n)\Big].
\end{align*}
\paragraph{Upper bound of the main two terms.}
Now we bound the two terms that appear in the last sentence. Regarding the first term, observe that $\mathbb{P}_\theta(\tau_a(1)> g(1))$ is the probability the first time arm $a$ is picked after $g(1)$ trials, then $\mathbb{P}_\theta(\tau_a(1)> g(1)) \leq  \prod_{i=1}^{g(1)}(1-\pi_{i,a})$. For an arm s.t. $p_a>0$ we find that
$\mathbb{P}_\theta(\tau_a(1)> g(1)) \leq (1-p_a)^{g(1)}$ and $\mathbb{P}_\theta(\tau_a(1)> g(1)) \leq \prod_{i=1}^{g(1)}(1-\epsilon_i (1-\psum)/K_0)\leq (1-\epsilon_{g(1)} (1-\psum)/K_0)^{g(1)}$ otherwise (since $\epsilon_i$ is a decreasing sequence).

For the second term we find
\[
\mathbb{P}_\theta( \tau_a(k)-\tau_a(k-1)> N| \tau_a(k-1)=n) \leq \prod_{i=n+1}^{N+n} (1- \pi_{i,a}) \leq \begin{cases}
    (1-p_a)^N & \hbox{if } p_a>0,\\
    \prod_{i=n+1}^{N+n} (1- \epsilon_i(1-\psum)/K_0) & \hbox{otherwise.} 
\end{cases}
\]
Hence
\[
\mathbb{P}_\theta( \tau_a(k)-\tau_a(k-1)> g(k)-g(k-1)| \tau_a(k-1)=n) \leq \begin{cases}
    (1-p_a)^{g(k)-g(k-1)} & \hbox{if } p_a>0,\\
    \prod_{i=n+1}^{g(k)-g(k-1)+n} (1- \epsilon_i(1-\psum)/K_0) & \hbox{otherwise.} 
\end{cases}
\]
In the last term, perform the change of variable $\prod_{i=n+1}^{g(k)-g(k-1)+n} (1- \epsilon_i(1-\psum)/K_0)=\prod_{i=0}^{g(k)-g(k-1)-1} (1- \epsilon_{i+n+1}(1-\psum)/K_0)$. Next, use the fact that $n\leq g(k-1)$ and that $\epsilon_t$ is decreasing in $t$, to obtain
\[\prod_{i=0}^{g(k)-g(k-1)-1} (1- \epsilon_{i+n+1}(1-p)/K') \leq \prod_{i=0}^{g(k)-g(k-1)-1} (1- \epsilon_{i+g(k-1)+1}(1-\psum)/K_0) \leq  (1- \epsilon_{g(k)}(1-\psum)/K_0)^{g(k)-g(k-1)}.\]
Let $b_{k,a}=(1-p_a)^{g(k)-g(k-1)}$ for $a$ s.t. $p_a>0$. Then for $g(k) = \lambda k^\alpha$ we have $g(k)-g(k-1)=\lambda(k^\alpha - (k-1)^\alpha) \geq \lambda k^{\alpha-1}$, implying $b_{k,a} \leq (1-p_a)^{\lambda k^{\alpha-1}}$. Applying the inequality $1-x \leq \exp(-x)$ we find
\[
b_{k,a} \leq \exp( - p_a \lambda k^{\alpha -1})\leq \exp( - p_a \lambda k) \Rightarrow \sum_{k\geq 1} b_{k,a} \leq \sum_{k\geq.1} \exp( - p_a \lambda k)= \frac{\exp( - p_a \lambda )}{1-\exp( - p_a \lambda )}.
\]

Now, let $b_{k,a}'= (1- \epsilon_{g(k)}(1-\psum)/K_0)^{g(k)-g(k-1)}$. As before, we find
$
b_{k,a}' \leq  (1- \epsilon_{g(k)}(1-\psum)/K_0)^{\lambda k^{\alpha-1}}
$. Now, use $\epsilon_{g(k)} = 1/2\sqrt{\lambda k^\alpha}$, thus
\[
b_{k,a}' \leq \left(1- \frac{1-\psum}{2K_0\sqrt{\lambda k^\alpha}}\right)^{\lambda k^{\alpha-1}} \leq \exp\left(- \frac{\lambda k^{\alpha -1}(1-\psum)}{2K_0\sqrt{\lambda k^\alpha}}\right) = \exp\left(- \frac{\lambda k^{\alpha -1}(1-\psum)}{2K_0\sqrt{\lambda k^\alpha}}\right) \leq \exp\left(- \frac{\sqrt{\lambda} k^{\alpha/2 -1}(1-\psum)}{2K_0}\right).
\]
letting $\alpha=4$, we have $b_{k,a}' \leq \exp\left(- \frac{\sqrt{\lambda} k(1-\psum)}{2K_0}\right) $, hence $\sum_{k\geq 1}b_{k,a}' \leq \frac{\exp\left(- \frac{\sqrt{\lambda} (1-\psum)}{2K_0}\right)}{1-\exp\left(- \frac{\sqrt{\lambda} (1-\psum)}{2K_0}\right)} $. 

\paragraph{Final step.} In conclusion, letting $p_{\rm min} = \min_{a:p_a>0} p_a$ and using the fact that $e^{-x}/(1-e^{-x})$ is a decreasing function for $x>0$:
\begin{align*}
\mathbb{P}_\theta(\bar {\cal E})&\leq\sum_{a:p_a>0}\frac{\exp( - p_a \lambda )}{1-\exp( - p_a \lambda )} + K_0\frac{\exp\left(- \frac{\sqrt{\lambda} (1-\psum)}{2K_0}\right)}{1-\exp\left(- \frac{\sqrt{\lambda} (1-\psum)}{2K_0}\right)},\\
&\leq(K-K_0)\frac{\exp( - p_{\rm min} \lambda )}{1-\exp( - p_{\rm min}\lambda )} + K_0\frac{\exp\left(- \frac{\sqrt{\lambda} (1-\psum)}{2K_0}\right)}{1-\exp\left(- \frac{\sqrt{\lambda} (1-\psum)}{2K_0}\right)},\\
&\leq K' \left[\frac{\exp( - p_{\rm min}\lambda )}{1-\exp( - p_{\rm min} \lambda )} + \frac{\exp\left(- \frac{\sqrt{\lambda} (1-\psum)}{2K_0}\right)}{1-\exp\left(- \frac{\sqrt{\lambda} (1-\psum)}{2K_0}\right)}\right],\\
&\leq K' \left[\frac{\exp( - p_{\rm min} \lambda )}{1-\exp( - p_{\rm min} \lambda )} + \frac{\exp\left(- \frac{\sqrt{\lambda} (1-\psum)}{2K_0}\right)}{1-\exp\left(- \frac{\sqrt{\lambda} (1-\psum)}{2K_0}\right)}\right],\\
&\leq K' \left[\underbrace{\frac{\exp( - p_0 \lambda )}{1-\exp( - p_0\lambda )}}_{(\circ)} + \underbrace{\frac{\exp\left(- p_0\sqrt{\lambda}\right)}{1-\exp\left(- p_0\sqrt{\lambda}\right)}}_{(\square)}\right].
\end{align*}
where $K'=2\max(K-K_0,K_0)$ and $p_0 = \min(p_{\rm min}, (1-\psum)/2K_0)$.
We want to verify for what value of $\lambda$ the last inequality is smaller than $\delta$.
In case the first term $(\circ)$ dominates, we can upper bound the last expression by two times $(\circ)$ and obtain that
\[
K' \frac{\exp( - p_0 \lambda )}{1-\exp( - p_0 \lambda )} = \delta \Rightarrow \lambda = \frac{\log(1+\frac{K'}{\delta})}{p_0}
\]
and otherwise if the second term  $(\square)$ dominates we find
\[
K'\frac{\exp\left(- p_0\sqrt{\lambda}\right)}{1-\exp\left(- p_0\sqrt{\lambda}\right)}=\delta \Rightarrow  \lambda=\frac{\log^2(1+\frac{K'}{\delta})}{p_0}
\]
In both cases, since $K'/\delta >2$, we have $\ln(1+\frac{K'}{\delta}) < \log^2(1+\frac{K'}{\delta})$. Hence $\lambda = \frac{\log^2(1+K'/\delta)}{p_0}$ guarantees $\mathbb{P}_\theta(\bar {\cal E}) \leq \delta$.

\paragraph{Adaptation with $\theta$-dependent rates.} The adaptation in this setting is straightforward by noting now that we only have the contribution due to $b_{k,a}'$. In fact, for all arms $a$  we can bound $\mathbb{P}_\theta(\tau_a(1)> g(1)) \leq \prod_{i=1}^{g(1)}(1-\epsilon_i/K)\leq (1-\epsilon_{g(1)} /K)^{g(1)}$  and

$
\mathbb{P}_\theta( \tau_a(k)-\tau_a(k-1)> g(k)-g(k-1)| \tau_a(k-1)=n)  
   \leq  (1- \epsilon_{g(k)}/K)^{g(k)-g(k-1)}.$
   Choosing $g(k)=\lambda k^4$ leads to
   \[
   b_{k,a}' \leq \left(1 - \frac{\epsilon_{g(k)}}{K}\right)^{\lambda k^\alpha-1} \leq 
\exp\left(-\frac{k\sqrt{\lambda}}{2K}\right).
   \]
   Therefore $\mathbb{P}_\theta(\bar {\cal E}) \leq K \frac{\exp(-\sqrt{\lambda}/2K)}{1-\exp(-\sqrt{\lambda}/2K)}$. Choosing $\lambda$ s.t. this latter probability is bounded by $\delta$  yields the result.
\end{proof}

\newpage 
\section{Extended related work}
\label{sec:extended_related_work}

\begin{table}[b]
\fontsize{8.6}{8.6}\selectfont
\centering
\scalebox{1}{
\begin{tabular}{|c || c | c | c | c |} 
 \hline
  & Setting & Fairness definition & Lower Bound & Upper Bound  \\
 \hline\hline
\cite{celis2019controlling} & Pre-specified range  & $ l_a \leq \pi_a(t) \leq u_a, \forall a\in [K],\forall t\in [T] $ & \xmark & $O\left(K\log T/ \Delta_{\rm min}^2  \right)$ \\ \hline
    \cite{Claure20} & Pre-specified range  & $ \mathbb{E}_\theta[N_a(T)/T]\geq p, \forall a\in [K]$ & \xmark & $O\left(\sum_a \log(T)/\Delta_a \right)$ \\ \hline
   \cite{Patil21} & Pre-specified range  & $\lfloor p_at\rfloor  - \eta \le N_{a}(t), \forall t\in[T], \forall a\in[K]$ &  \xmark & $O\left(\sum_a \Delta_a\right)$ \\
 \hline
 \cite{Li19combinatorial} & Pre-specified range & $\lim\inf_{T\to\infty}  \mathbb{E}[ N_a(T)/T] \ge p_a, \forall a\in [K]$ & \xmark & $O\left(\sqrt{T\log(T)}\right)$  \\ 
 \hline
 \cite{Chen20} & Pre-specified range & $\mathbb{E}_{x\sim p_{\cal X}}[\pi_{x,a}(t)] \geq p, \forall a\in [K], \forall t\in [T]$ & \xmark &  $O(\sqrt{TMK\ln(K)})$ \\ 
 \hline

 \cite{Joseph16} & Individual fairness & 
 $\mathbb{P}_\theta( \pi_a(t) > \pi_b(t)$  \hbox{only if }$\theta_a > \theta_b |{\cal H}_t)\geq 1-\delta$ & $\Omega(K^3\log(1/\delta))$ &  $O(\sqrt{K^3T\log(TK/\delta)})$\\ 
 \hline
 \cite{liu2017calibrated} & Individual fairness & 
$D_1(\pi_a(t),\pi_a(t)) \leq \varepsilon_1 D_2(\theta_a,\theta_b) + \varepsilon_2 $ w.p. $1-\delta, \forall t\in [T]$& \xmark  &  $O((KT)^{2/3})$\\ 
 \hline
  \cite{gillen2018online} & Individual fairness & 
Oracle feedback & \xmark  &  $O(\sqrt{T})$\\ 
 \hline

 \cite{Wang21fair} &  Individual fairness & Proportional fair: $\pi_a^\star/\pi_b^\star = p(\theta_a)/p(\theta_b),  \forall a,b\in [K].$ & \xmark &  $\tilde{O}\left(\sqrt{TK}\right)$ w.p. $1-\delta$. \\ 
 \hline
 \cite{Wang21} & Individual fairness  & Proportional fair: $\pi_a^\star/\pi_b^\star = \theta_a/\theta_b,  \forall a,b\in [K].$   & \xmark &  \xmark\\
 \hline
\end{tabular}}
\caption{Summary of bandit fairness settings.}
\label{tab:comparison_learners}
\end{table}

Fairness in machine learning has been extensively studied \cite{caton2020fairness} for various fairness criteria. Similarly, different notions of fairness have been considered in the bandit literature \cite{zhang2021fairness,gajane2022survey}, which have predominantly focused on the framework of regret minimization rather than pure exploration (\emph{a.k.a.}, Best-Arm Identification \cite{lattimore2020bandit}).

The majority of these notions deal with the problem that arms should not be neglected, emphasizing the importance of selecting each arm sufficiently. This selection could be based on an average or a specific probability, and may or may not depend on the arm's reward distribution. Such an approach essentially places a constraint on the algorithm to guarantee balanced arm selection. 

In particular, fairness concepts in the literature generally fall into the following categories: \emph{individual fairness}, \emph{selection with pre-specified range}, \emph{counterfactual fairness} and \emph{group fairness} \cite{zhang2021fairness,gajane2022survey}. 
\begin{itemize}
    \item \emph{Individual fairness}  \cite{dwork2012fairness,Joseph16} requires a system to make comparable decisions for similar individuals, and the constraints could be based on similarity or merit \cite{Wang21,liu2017calibrated}.
    \item \emph{Selection with pre-specified range} \cite{celis2019controlling,Patil21,Li19combinatorial} simply demands that the rate, or probability, at which an algorithm selects an arm stays within a pre-specified range.
    \item \emph{Group fairness} imposes constraints based on some statistical parity across subgroups  \cite{gajane2022survey}. For example, in \cite{schumann2019group} divide arms into several subgroups, and ensure that the probability of pulling an arm is constant given the group membership. In contextual bandit problems, one can ensure fairness among different contexts, as in \cite{huang2022achieving} or between groups similarly to the non-contextual setting \cite{grazzi2022group}.
    \item In \cite{kusner2017counterfactual} the authors study the concept of \emph{counterfactual fairness}. Their definition captures the idea that a decision is fair towards an individual if it is fair also in an alternative situation where the individual belong to a different group while keeping all the other important variables unchanged.
\end{itemize}
In the following, we focus on the first  two notions of fairness for single-agent systems in the online setting (we refer the reader to \cite{gajane2022survey} for more details about the other cases). For the multi-agent case, recent works include \cite{baek2021fair}, where the authors study the notion of \emph{Nash bargaining solution}, and \cite{Hossain21fair}, where they use the \emph{Nash social welfare} as notion of fairness. For the offline case, in \cite{metevier2019offline} the authors present {\sc RobinHood}, an offline contextual bandit method designed to satisfy, in probability, a generic fairness criterion defined through a constraint objective.

\paragraph{Selection with pre-specified range} This type of fairness constraint demands the rate at which an arm is pulled to stay within a pre-specified range, and thus does not depend on $\theta$.

In \cite{celis2019controlling} the authors use a notion of fairness that constrains the probability that an arm  $a$ is selected to stay within a pre-specified constant interval $[l_a,u_a]$. This type of constraint yields a polytope ${\cal C}$ on the possible set of policies, and they compare the performance of their algorithm to the best-performing policy in ${\cal C}$. They propose {\sc Constrained-$\varepsilon$-greedy}, an algorithm that achieves a regret of $O(K\ln T/ \eta \Delta_{\rm min}^2)$, where $\eta$ is some small constant.

The authors in \cite{Li19combinatorial} provide a notion of \emph{asymptotic fairness} in a combinatorial sleeping bandit setting
\[ \quad \lim\inf_{T\to\infty}  \mathbb{E}\left[\frac{N_a(T)}{T}\right] \geq p_a \quad  \forall a\in[K],
\]
where $ (p_a)_a \in [0,1]^K$ are known fixed values. This constraint effectively limits the rate at which arms are selected asymptotically and does not depend on the value of $\theta$.

In \cite{Claure20} the authors propose two algorithms: Strictly-rate-constrainted {\sc UCB} and Stochastic-rate-constrained {\sc UCB}. The former algorithm guarantees that at any time the pulling rate for any arm is at least $p-1/t$, with a regret upper bound of $O(\sum_{a\neq a^\star} \ln T/\Delta_a)$. 
The latter algorithm, Stochastic-rate-constrained {\sc UCB}, guarantees that at each time $t$ each arm has to be pulled with a probability greater than $p$. The regret is computed by comparing to a policy that pulls the best-estimated arm with probability $(1-Kp)$, and uniformly otherwise, leading to a regret upper bounded of $O(\sum_{a\neq a^\star} \ln T/\Delta_a)$.

In \cite{Patil21} the authors propose a constraint similar to Strictly-rate-constrained {\sc UCB}  that holds uniformly over time. They say an algorithm to be $\eta$-fair if 
\[
\lfloor p_at\rfloor  - N_{a}(t) \le \eta, \forall t\in[T], \forall a\in\mathcal{A}.
\]
where $\eta\geq 0$ and $p_a \in [0,1/K)$ for every arm $a$. The parameter $\eta$ quantifies the \emph{unfairness tolerance} allowed in the system. Furthermore, the parameter $p_a$ is constrained in $[0,1/K)$. 
Lastly, to account for the fact that now any fair algorithm must incur a linear regret, they define a new notion of regret that does not account for the regret accumulated due to the fairness constraint:
\[
R_F(T) = \sum_{a\in\mathcal{A}} \Delta_a\left(\mathbb{E}[N_a(T)] - \max(0, \lfloor p_a T\rfloor - \eta)\right).
\]
They provide an instance-specific upper bound to their {\sc Fair-UCB} algorithm that for large $T$ becomes $R_F(T) \leq (1+\pi^2/3) \sum_{a\neq a^\star} \Delta_a$. The fact that the regret does not scale in time is due to the fact that for large values of $T$ the algorithm will pull sub-optimal arms only to satisfy the fairness constraints. In \cite{liu2022combinatorial} the authors show a more generic {\sc UCB-LP} algorithm that is able to deal with this type of fairness constraint, and other types of combinatorial constraints.

The authors of \cite{Chen20} study an adversarial contextual bandit setting with $M$ contexts and $K$ arms. In this work fairness is defined as a minimum rate that a task is assigned to a user, and the constraint is formalized as: 
$$\mathbb{E}_{x\sim p_{\cal X}}[\pi_{x,a}(t)] \ge p, \forall a\in [K], \forall t \in[T],
$$
where $p_{\mathcal{X}}(x)$ is the probability of observing context $x$, and $\pi_{x,a}(t)$ is the conditional probability of selecting an arm $a$ for a given context $x$ at time $t$. They propose a variant of Follow-the-Regularized-Leader (FRTL) which yields $O(\sqrt{TMK\ln(K)})$ regret. 

\paragraph{Individual fairness} These types of fairness constraints demand making similar decisions for similar arms. One of the first works to consider this type of fairness in stochastic bandits and contextual bandits is \cite{Joseph16}. Their notion of fairness considers an history of observations ${\cal H}_t$ up to time $t$, and define an algorithm to be $\delta$-fair if with probability at least $1-\delta$,  over the realization history ${\cal H}_t$, for all rounds $t\in [T]$ and all pairs of arms $a,b \in [K]$ we have
$$
\pi_{a}(t) > \pi_{b}(t) \hbox{ only if } \theta_a > \theta_b,
$$
where $\pi_{a}(t)$ is the probability that at time $t$ the algorithms chooses arm $a$. A similar fairness criterion is considered also in \cite{jabbari2017fairness} for Markov decision processes (by using the $Q$-values of the optimal policy).
In other words, the condition above ensures that a better arm is always selected with a higher probability than a worse arm. They propose an adaptation of the UCB algorithm that satisfies this fairness condition and whose pseudo-regret satisfies $R(T) = O(\sqrt{K^3T\ln(TK/\delta)})$. They also state a lower bound $\Omega(K^3 \ln(1/\delta))$ that suggests that this cubic rate in $K$ may be hard to improve. 

In \cite{liu2017calibrated} they consider  stochastic and dueling bandits, and the authors impose two specific fairness constraints: \emph{smooth fairness} and \emph{calibrated fairness}. Smooth fairness indicates that two arms with similar reward distributions should be selected with comparable probabilities. Technically, for all $t$ with probability $1-\delta$ we have $D_1(\pi_t(a),\pi_t(b)) \leq \varepsilon_1 D_2(\theta_a,\theta_b) + \varepsilon_2 $, where $D_1,D_2$ are suitable divergence functions with $\varepsilon_1,\varepsilon_2\geq 0$ are suitable constant. They develop a Thompson-Sampling method that achieves a fairness regret of $O((KT)^{2/3})$.
Calibrated fairness, on the other hand, requires that each arm be sampled with a probability proportional to the likelihood of its reward being the greatest.

In \cite{gillen2018online} the authors study fairness in a linear contextual bandit setting. They highlight the difficulty of defining a precise fairness metric over individuals. To avoid this issue, they assume the algorithm has access to an oracle that understands fairness but cannot define it explicitly. The algorithm learns about fairness through feedback on its decisions from the oracle, adjusting accordingly to meet the fairness constraint, and achieve a regret of $O(\sqrt{T})$.

\paragraph{$\alpha$-fairness criterion}
Another important body of work considers the $\alpha$-fairness criterion \cite{atkinson1970measurement,radunovic2007unified,si2022enabling} for fair resource allocation, which yields different fairness criteria based on the value of $\alpha$. Generally speaking, the aim is to find a policy maximizing the $\alpha$-criterion
\[
f_\alpha(\theta) = \begin{cases}
    \frac{\theta^{1-\alpha}}{1-\alpha} & \alpha \in [0,1) \cup (1,\infty),\\
    \log(\theta) & \alpha=1.
\end{cases}
\]
For $\alpha\to\infty$ we obtain the notion of \emph{max-min} fairness, which is used when we want to allocate as equal resources as possible to users/items However, sometimes it is unwise to allocate resources to users that are much more expensive than others. For $\alpha=0$ we obtain the classical greedy solution, while the case $\alpha=1$ is also known as \emph{proportional fair}. This latter case tries to allocate resources to users/items in a proportional manner. Therefore, we believe that the case $\alpha=1$ is part of the more general notion of \emph{individual fairness}, which is considered in \cite{Wang21fair, Wang21,talebi2018learning}.

The authors of \cite{Wang21fair} study a MAB setting in which the goal is to devise a fair allocation according to some merit function $p$ that is strictly positive. Their aim is to devise a policy $\pi\in \{\pi \in [0,1]^K : \sum_a \pi_a=1\}$ that ensures each arm has a selection rate proportional to its merit, that is
\[
\frac{\pi_a}{\pi_b} = \frac{p(\theta_a)}{p(\theta_b)}, \quad \forall a,b\in [K].
\]
This constraint yields an optimal fair policy of the type (Th. 3.1.1 in \cite{Wang21fair}) 
$$
\pi_a^\star = \frac{p(\theta_a)}{\sum_{b\in [K]}p(\theta_b)},\quad \forall a\in [K],
$$
and they measure the fairness of a policy based on the so-called \textit{fair regret} up to time $T$, which is defined as 
\[
R_F(T) = \sum_{t\in [T]} \sum_{a} |\pi^{\star}_a - \pi_{a}(t)|,
\]
where $\pi_{a}(t)$ is the policy selected by the agent at time $t$. The regret analysis relies on two conditions: (1) that the merit of each arm is positive, lower bounded by some known constant $\gamma$; (2) that the merit function is $L$-Lipschitz continuous. Without one of these conditions, they show that the minimax regret lower bound is of order $O(T)$. Their UCB-type algorithm satisfies a fairness regret of $\tilde{O}(L \sqrt{KT}/\gamma)$ and a classical reward regret of $\tilde{O}(\sqrt{KT})$ with probability $1-\delta$.

In \cite{Wang21} the authors aim to devise a purely proportional fair allocation (with no consideration for regret). In this setting, an allocation is defined as a vector over actions $\pi = (\pi_a)_{a\in[K]}$ such that $\pi_a \ge 0$, and $\sum_{a} \pi_a = T$. Given the arm utilities vector $\theta = (\theta_a)_{a\in\mathcal{A}}$, an allocation is \textit{proportionally fair} if it solves the following optimization problem:
$$
\max_{\pi} \sum_{a} \theta_a \log(\pi_a) \text{ s.t. } \sum_{a} \pi_a = T.
$$ 
from which one can find the optimal fair allocation as follows
$$
\pi^\star_a = \frac{T \theta_a}{\sum_{b\in[K]}\theta_b}.
$$
Therefore the optimal solution satisfies $\pi_a^\star/\pi_b^\star = \theta_a/\theta_b$, similarly to \cite{Wang21fair}. They formulate {\sc Proportional Catch-up}, an algorithm that tries to play arms as to guarantee that $N_a(t)/N_b(t) \approx \hat \theta_a(t)/ \hat\theta_b(t)$, where $\hat \theta(t)=(\hat\theta_a(t))_{a\in[K]}$ is the utility estimator at time $t$.

\paragraph{Best arm identification (BAI)} The primary objective in standard best arm identification problems is to identify an arm that yields the highest expected reward, To the best of our knowledge, the setting of BAI with fairness constraint has been studied only in \cite{wu2023best}, where the authors consider fairness constraints of subpopulations. This type of constraint requires that the chosen arm must be fair across various subpopulations (such as different ethnic groups, age brackets, etc.). This is achieved by ensuring that the expected reward for each subpopulation exceeds certain predefined thresholds. 
\else \fi

\end{document}